\documentclass{article}

\usepackage{arxiv}

\usepackage[utf8]{inputenc} 
\usepackage[T1]{fontenc}    

\usepackage{multirow}
\usepackage{multicol}

\usepackage[round]{natbib}

\usepackage[dvipsnames]{xcolor}

\usepackage{tcolorbox}

\usepackage{hyperref}       
\usepackage{url}            
\usepackage{booktabs}       
\usepackage{amsfonts}       
\usepackage{nicefrac}       
\usepackage{amsmath}
\usepackage{amssymb}
\usepackage{microtype}      
\usepackage{xcolor}         

\usepackage{amsthm}
\usepackage{thmtools, thm-restate}
\usepackage{mathtools}
\usepackage{bm}
\usepackage{booktabs}       
\usepackage{enumitem}

\usepackage{wrapfig}
\usepackage{colortbl}
\usepackage{bbm}
\usepackage{optidef}
\allowdisplaybreaks
\usepackage{algcompatible}
\usepackage{float}
\usepackage{microtype}
\usepackage{graphicx}
\usepackage{subfigure}

\hypersetup{
    colorlinks,
    linkcolor={cyan!80!black},
    citecolor={RoyalBlue!80!white},
    urlcolor={RoyalBlue}
}

\usepackage{algorithm,algpseudocode}

\def\E{\mathbb E}

\usepackage{tikz}

\newtheorem{theorem}{Theorem}

\newtheorem{lemma}{Lemma}

\newtheorem{definition}{Definition}

\newtheorem*{theorem*}{Theorem}
\def \E {\mathbb{E}}

\def \R {\mathbb{R}}

\def \P {\mathbb{P}}

\def \cD {\mathcal{D}}

\def \cF {\mathcal{F}}

\def \cO {\mathcal{O}}

\def \cS {\mathcal{S}}

\def \cX {\mathcal{X}}

\def \TPR {\text{TPR}}
\def \FPR {\text{FPR}}

\def \FPRhat {\widehat{\FPR}}

\newcommand{\ldef}{\vcentcolon=}

\def \sign {\text{sign}}

\algnewcommand{\LeftComment}[1]{\Statex \(\triangleright\) #1}
\renewcommand{\algorithmicrequire}{\textbf{Input:}}
\renewcommand{\algorithmicensure}{\textbf{Output:}}

\usepackage{hyperref}

\title{ \textbf{Taming False Positives in Out-of-Distribution Detection \\ with Human Feedback}}

\author{ Harit Vishwakarma  \\
	\small{University of Wisconsin-Madison, WI}\\
	\texttt{hvishwakarma@cs.wisc.edu} \\
\And
       {Heguang Lin}   \\
	\small{University of Pennsylvania, PA}\\
	\texttt{hglin@seas.upenn.edu} \\
 \And
	{Ramya Korlakai Vinayak} \\
	\small{University of Wisconsin-Madison, WI}\\
	\texttt{ramya@ece.wisc.edu} \\
}
\date{}

\begin{document}

\maketitle

\begin{abstract}
Robustness to out-of-distribution (OOD) samples is crucial for safely deploying machine learning models in the open world. Recent works have focused on designing scoring functions to quantify OOD uncertainty. Setting appropriate thresholds for these scoring functions for OOD detection is challenging as OOD samples are often unavailable up front. Typically, thresholds are set to achieve a desired true positive rate (TPR), e.g., $95\%$ TPR. However, this can lead to very high false positive rates (FPR), ranging from 60 to 96\%, as observed in the Open-OOD benchmark. In safety-critical real-life applications, e.g., medical diagnosis, controlling the FPR is essential when dealing with various OOD samples dynamically. To address these challenges, we propose a mathematically grounded OOD detection framework that leverages expert feedback to \emph{safely} update the threshold on the fly. We provide theoretical results showing that it is guaranteed to meet the FPR constraint at all times while minimizing the use of human feedback. Another key feature of our framework is that it can work with any scoring function for OOD uncertainty quantification. Empirical evaluation of our system on synthetic and benchmark OOD datasets shows that our method can maintain FPR at most $5\%$ while maximizing TPR \footnote{Appeared in the 27th International Conference on Artificial Intelligence and Statistics (AISTATS 2024).}. 
\end{abstract}

\section{Introduction}
Deploying machine learning (ML) models in the open world makes them subject to out-of-distribution (OOD) inputs: an ML model trained to classify on $K$ classes also encounters points that do not belong to any of the classes in the training data.
Modern ML models, in particular deep neural networks, can fail silently with high confidence on such OOD points \citep{nguyen2015Fooling, amodei2016Safety}. 
Such failures can have serious consequences in high-risk applications, e.g., medical diagnosis and autonomous driving.
Safe deployment of ML models in an open world setting needs mechanisms that ensure robustness to OOD inputs. 
The importance of this problem has led to the development of many methods for OOD detection \citep{liang2017enhancing, lee2018simple, liu2020energy, ming2022cider} 
which aim to produce a \emph{score} that can be used to decide on OOD vs in-distribution (ID) for a given point. 
For a detailed survey of literature in the area of generalized OOD detection, see ~\cite{yang2021generalized}.

ID data is usually plentiful, but we do not get to see different kinds of OOD samples before deployment. Consequently, many works in OOD detection are largely limited to \emph{static settings} where the ID data is used to set a threshold on the scores used for detection ~\citep{liang2017enhancing, liu2020energy,  ming2022cider}. 
In these scenarios, this is usually done by setting a threshold that achieves a certain level of true positive rate (TPR), such as $95\%$. However, this can lead to a very high false positive rate (FPR), e.g., ranging between $60\%$ to $96\%$ as observed in the Open-OOD benchmark  \citep{yang2022openood}. Furthermore, even if the ID data distribution remains the same after deployment, the OOD data could vary, resulting in highly fluctuating FPR. Thus, having a small, fixed amount of OOD data collected a priori to validate the FPR at a given threshold would not help in guaranteeing the desired FPR.

\begin{figure*}[t]
    \centering
    \vspace{-8pt}
    \includegraphics[width=0.85\textwidth]{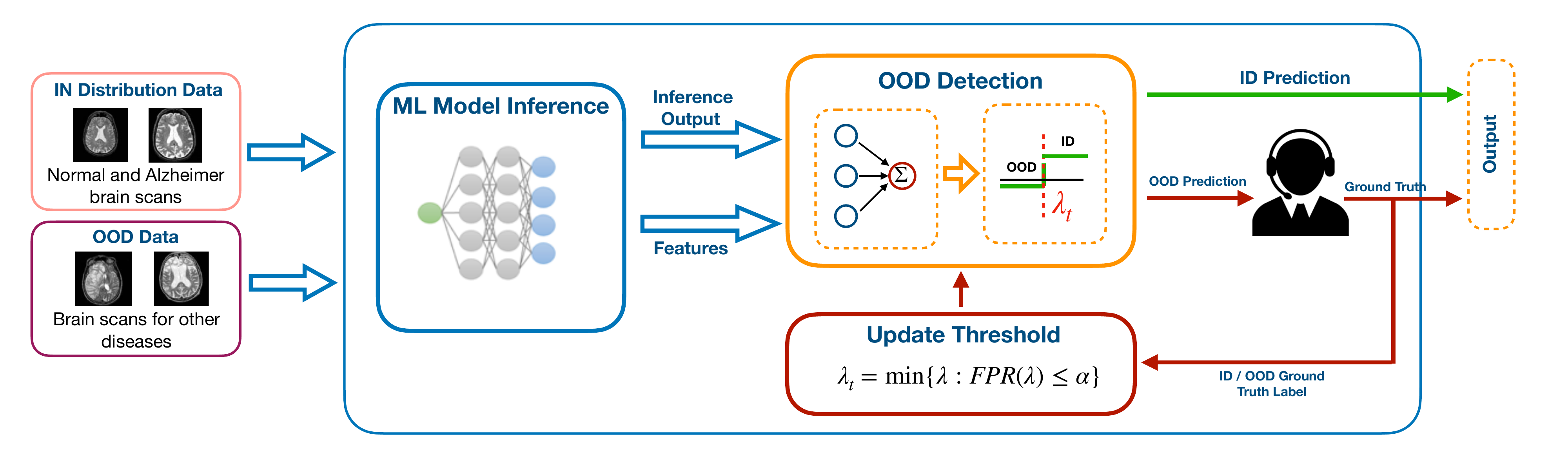}
    \label{fig:OODSystemWithHumanInLoop}
    \vspace{0pt}
    \caption{{Illustration of OOD detection with human-in-the-loop with FPR control. In this example, the ID data is of brain scans of normal people and those with Alzheimer's disease. The OOD data could be anything other than these, e.g. brain scans of patients with some other diseases. }}
    \vspace{-8pt}
\end{figure*}
In safety critical applications, the consequences of classifying an OOD point as ID (false positive) are more catastrophic than classifying an ID point as OOD (false negative). For example, in the medical diagnosis of brain scans, when the system is in doubt it is better to classify a scan as OOD and defer the decision to human experts rather than for the ML model to give it an ID  label i.e., predicting an in-distribution disease or classifying it as a normal scan.  Therefore, safely using ML models in such applications requires systems guaranteeing that the FPR is below a certain acceptable rate, e.g., FPR below $5\%$. 

Furthermore, it is difficult to anticipate or collect the exact type of OOD data that the system can encounter during deployment. Thus it is crucial that such systems adapt to the OOD data while controlling the FPR. Motivated by these challenges we pose the following goal for safe OOD detection.

\begin{tcolorbox}[top=3pt,bottom=3pt,left=3pt,right=3pt]
\textbf{Goal:} Develop a human-in-the-loop out-of-distribution detection system that has guaranteed false positive rate control while minimizing the amount of human intervention needed.
\end{tcolorbox}

\textbf{Our Contributions:} Toward this goal, we make the following contributions:
\begin{enumerate}[itemsep=2pt, topsep=0pt, partopsep=0pt, parsep=0pt,leftmargin=*]
\item \textbf{Human-in-the-loop OOD detection framework.} We propose a novel mathematically grounded framework that incorporates expert human feedback to safely update the OOD detection threshold, ensuring robustness to variations in OOD data encountered after deployment. Our framework can be used with any scoring function.
\item \textbf{Guaranteed FPR control.}  For stationary settings, we provide theoretical guarantees for our framework on controlling FPR at the desired level at all times and also provide a bound on the time taken to reach a given level of optimality. Using the insight from this analysis, we also propose an approach for settings with change points that reduces the duration of violation of FPR control. 
\item \textbf{Empirical validation on benchmark datasets.} We evaluate our framework through extensive simulations both in stationary and distribution shift settings. Through experiments on benchmark OOD datasets in image classification tasks with various scoring functions, we demonstrate the effectiveness of our proposed framework. 
\end{enumerate}
We emphasize that our aim is to develop a framework that can use any scoring function and safely adapt the threshold on the fly to enable the safe deployment of ML models. Therefore, our work is complementary to works that develop scoring functions for OOD detection. 

\section{Human-in-the-Loop OOD Detection}

\label{sec:framework}
We propose a human-in-the-loop OOD detection system (Figure~\ref{fig:OODSystemWithHumanInLoop}) that can work with any ML inference model and scoring function for OOD detection. We begin by describing the problem setting and then discuss each component of our proposed system in detail. See Algorithm \ref{alg:adaptive-ood} for step-by-step pseudocode.

\subsection{Problem Setting}

\textbf{Data stream.} Let $\mathcal{X} \subseteq \mathbb{R}^d$ denote the feature space and $\mathcal{Y}=\{-1,1\}$ denote the label space for OOD detection with ``$1$'' denoting ID and ``$-1$'' denoting OOD. 
Let the distribution of ID and OOD data be denoted by $\cD_{id}$ and $\cD_{ood}$ respectively. Let $x_t \in \mathcal{X}$ denote the sample received at the time $t$. Let $y_t \in \{ -1,1\}$ denote the true label for $x_t$ with respect to ID or OOD classification. We assume $x_t$ are independent and drawn according to the following mixture model,
$x_t \sim (1-\gamma)\ \cD_{id} + \gamma\ \cD_{ood},$
where $\gamma \in (0,1)$ is the fraction of OOD points in the mixture. Note that $\cD_{id}$, $\cD_{ood}$ and $\gamma$ are \emph{unknown}.


\textbf{Scoring function.} 
After receiving data point $x_t$, the system uses a given scoring function, $g:\cX \mapsto \cS \subseteq \R$, to compute a score quantifying the uncertainty of the point being ID or OOD. Our system is designed to work with any scoring function based OOD uncertainty quantification. 
Let $s_t=g(x_t)$ denote the score computed for point $x_t$.
To be consistent across various scoring functions, let a higher score indicate ID and a lower score indicate OOD points. 
After computing score $s_t$ the system needs to decide whether $x_t$ is OOD or ID, it is done using a threshold-based classifier $h_{\lambda}: \R \mapsto \{-1,+1\}$ parameterized with $\lambda \in \Lambda \subseteq \mathbb{R}$:
$h_{\lambda}(g(x)) = \sign(g(x)-\lambda)$. Here we assume  $\Lambda = (\Lambda_{\min}, \Lambda_{\max})$. 

\textbf{FPR and TPR.} The population level $\FPR$ and $\TPR$ for any $\lambda \in \Lambda $ are defined as follows,
\begin{align*}
\FPR(\lambda) &= \E_{x\sim \cD_{ood}}[\mathbf{1} \{ g(x) > \lambda \}] \quad \text{and}\\  \TPR(\lambda) &= \E_{x\sim \cD_{id}}[\mathbf{1} \{ g(x) > \lambda \}].
\end{align*}
Note that the cumulative distribution function (CDF) of $\cD_{ood}$, $\text{CDF}_{\cD_{ood}}(\lambda) =  \E_{x\sim \cD_{ood}}[\mathbf{1} \{ g(x) \leq \lambda \}]$. Therefore, $\FPR(\lambda) = 1 - \text{CDF}_{\cD_{ood}}(\lambda)$. Similarly, $\TPR(\lambda) = 1 - \text{CDF}_{\cD_{id}}(\lambda)$. Since the CDF of any distribution is a monotonic function, both the FPR and TPR are monotonic in $\lambda$.
\begin{wrapfigure}{r}{.46\textwidth}
 \centering
 \vspace{5pt}
\includegraphics[height=4cm, keepaspectratio]{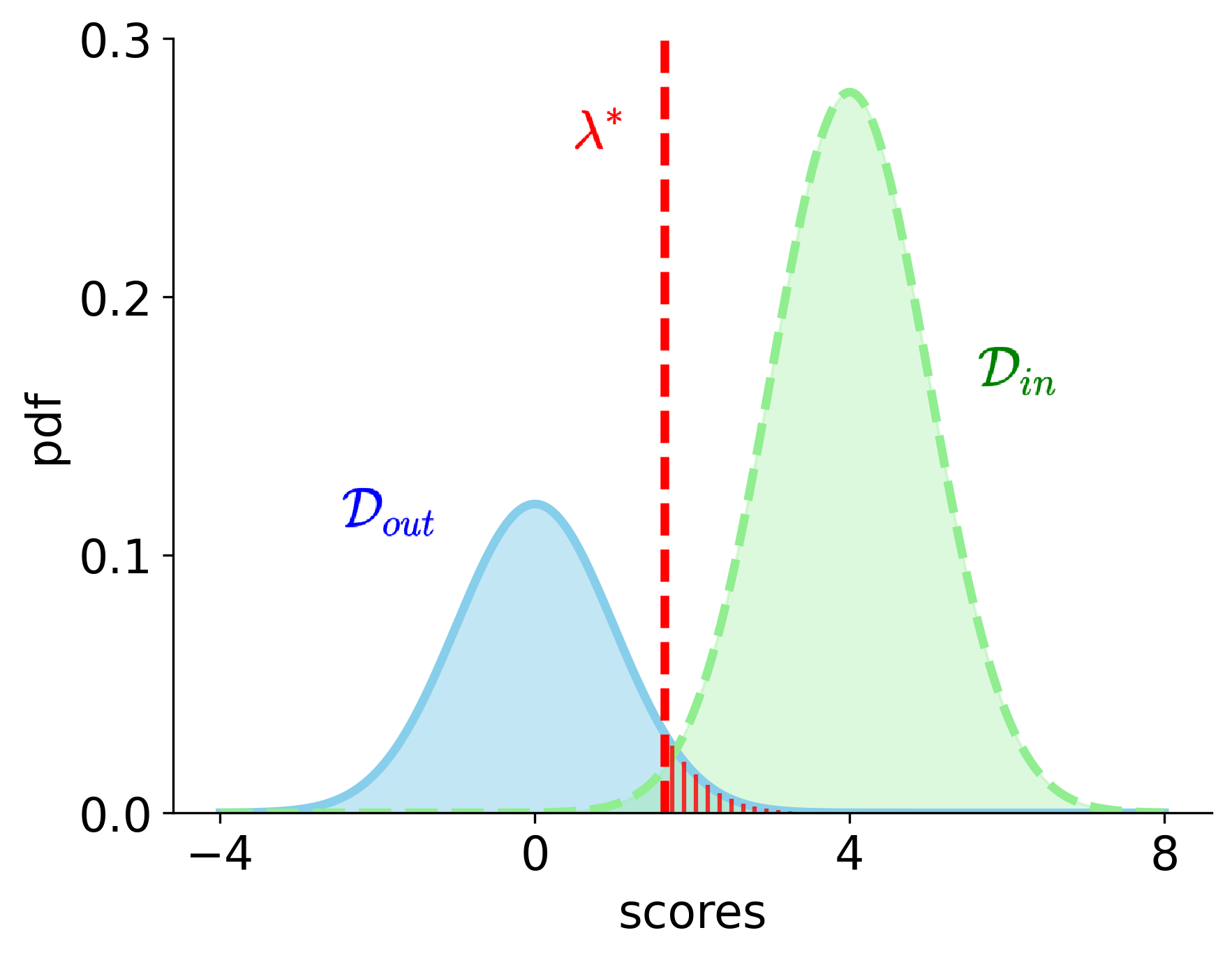}
  \caption{$\text{FPR}\left(\lambda^{\star}\right)=\alpha \longleftrightarrow \text{CDF}_{\mathcal{D}_{\text {out }}}\left(\lambda^{\star}\right)=1-\alpha$. {Optimal $\lambda^\star$} for the optimization problem~\eqref{eqn:ideal-opt} with $\alpha = 0.05$ and $x_t \overset{i.i.d}{\sim} 0.7\ \cD_{in} + 0.3\ \cD_{out}$, where  $\cD_{in}$ is $\mathcal{N}(4, 1)$ and $\cD_{out}$ is $\mathcal{N}(0, 1)$. 
  }
 \label{fig:FPRTPROptEg}
 \vspace{-8pt}

\end{wrapfigure}

\textbf{Expert human feedback.} Our goal is to tackle critical applications where a human expert examines the samples that are declared as OOD (instead of the ML model making automatic predictions on them). The feedback obtained from the human expert can be used to safely update the OOD detection threshold at each time step, $\lambda_t$, so that the FPR is maintained below the desired rate of $\alpha$.
One can trivially control FPR by setting $\lambda_t = \Lambda_{\max}$, i.e., always getting human feedback. This would of course be too expensive and defeat the purpose of using an ML model in the first place.  
Therefore, in addition to controlling the FPR, we aim to minimize the human feedback solicited by the system. In an ideal system, the points declared as ID are directly classified by the ML model, and only the points declared as OOD are examined by human expert.
Thus, minimizing human feedback is equivalent to maximizing the TPR. This can be done by setting the threshold as, $\lambda_t := \text{arg } \max_\lambda \TPR(\lambda) \text{ subject to } \FPR(\lambda) \leq \alpha$. Since the $\TPR$ is monotonic in $\lambda$, we can re-write this further as follows, 
\begin{tcolorbox}[top=-1pt,bottom=0pt,left=1pt,right=0pt, width=470pt]
\begin{equation}
\lambda_t^\star := \text{arg } \underset{\lambda  \in \Lambda}{\text{min }}
 \lambda, \quad \text{s.t. } \quad \FPR(\lambda) \le \alpha.
\tag{P1}\label{eqn:ideal-opt}
\end{equation}
\end{tcolorbox}
The optimal threshold, denoted by $\lambda^\star$, is the smallest $\lambda$ such that $\FPR(\lambda^\star)=\alpha = 1 - \text{CDF}_{\cD_{out}}(\lambda^\star)$ (see Figure~\ref{fig:FPRTPROptEg}).
When the distribution of the OOD points, $\cD_{ood}$, is not changing, then setting $\lambda_t^\star = \lambda^\star$ for all $t$ would be the optimal solution. Note that, changing the mixture ratio $\gamma$, or the distribution of the ID points $\cD_{id}$ does not affect the value of the optimal threshold.
As we do not have access to the true FPR and TPR values, we cannot solve the optimization 
problem~\eqref{eqn:ideal-opt}. Instead, we have to estimate the threshold at time $t$, denoted by $\hat{\lambda}_t$, using the observations until time $t$. 

\subsection{Adaptive Threshold Estimation}

Ideally, we want to avoid human feedback for points that are determined as ID by the system, i.e., with a score greater than $\hat{\lambda}_t$. However, in order to have an unbiased estimate of the FPR and to detect potential changes in the distribution of OOD samples and therefore change in true FPR, we obtain human feedback with a small probability $p$ for points predicted as ID by the system. We refer to this as \emph{importance sampling}. 



\textbf{FPR estimation and adapting the threshold.} At each time $t$, we observe $x_t \overset{i.i.d}{\sim}(1-\gamma) \cD_{id} + \gamma \cD_{ood}$, and $s_t = g(x_t)$ is the corresponding score. If 
$s_t \le \hat{\lambda}_{t-1}$, where $\hat{\lambda}_{t-1}$ is the threshold determined in at time $t-1$,  then it is considered an OOD point and hence gets a human label for it and we get to know whether it is in fact OOD or ID. If $s_t > \hat{\lambda}_{t-1}$, then $x_t$ is considered an ID point and hence gets a human label only with probability $p$. So, we get to know whether it is truly ID or not with probability $p$. Now we have to update the threshold, $\hat{\lambda}_t$, such that the $\FPR(\hat{\lambda}_t) \leq \alpha$, by finding the minimum $\lambda$ that satisfies this constraint in order to maximize the TPR. 

Our approach is based on using the feedback on the samples that are examined by human experts till time $t$ to construct an unbiased estimator of $\FPR(\lambda)$ (see Equation~\ref{eqn:EmpFPR}). We also construct an upper confidence interval for the estimated $\FPR(\lambda)$ that is valid at all thresholds $\lambda \in \Lambda$ and for all times simultaneously with high probability (see Equation~\ref{eqn:psi-theory}). 
This enables us to optimize for $\lambda$ such that the upper bound on the true $\FPR(\lambda)$ is at most $\alpha$ at each time $t$ and thus safely update the threshold $\lambda_t$. 
Let $S^{(o)}_t = \{s^{(o)}_1,\ldots s^{(o)}_{{N}_t^{(o)}}\}$ denote the scores of the points that have been truly identified as OOD from human feedback so far, and $I^{(o)}_t$ be the corresponding time points. We estimate the FPR as follows,
\begin{equation}
\FPRhat(\lambda,t) \ldef \frac{1}{N_t^{(o)}} \sum_{u \in I^{(o)}_t} Z_u(\lambda),
\label{eqn:EmpFPR}
\end{equation}
\begin{equation*}
 Z_{u}(\lambda) \ldef  \begin{cases} 
     \mathbf{1} (  s^{(o)}_{u} > \lambda), &   \text{ if } s^{(o)}_u \le \hat{\lambda}_{u-1} \\  \frac{1}{p}\mathbf{1}(s^{(o)}_{u}>\lambda), & \text{ w.p. } p \text{ if } s^{(o)}_u > \hat{\lambda}_{u-1} \\
    0, & \text{ w.p. } 1-p \text{ if } s^{(o)}_u > \hat{\lambda}_{u-1}
 \end{cases}.
 \end{equation*}
We show that our estimator for the FPR is unbiased. The proof is deferred to the Appendix \ref{subsec:proofs}. 
\begin{lemma}
    Let $p>0$, $\FPRhat(\lambda,t)$ as defined in eq.~\eqref{eqn:EmpFPR} is an unbiased estimate of the true $\FPR(\lambda)$, i.e.,  $\E[\FPRhat(\lambda,t)] = \FPR(\lambda)$.
\end{lemma}
\begin{wrapfigure}{r}{0.5\textwidth}
\begin{minipage}{0.5\textwidth}
\vspace{-10pt}
\begin{algorithm}[H]
\begin{algorithmic}[1]
\caption{ Human in the Loop OOD Detection} \label{alg:adaptive-ood}
\Require{FPR threshold $\alpha$ , sampling probability $p \in (0, 1)$,
Scoring function $g:\mathcal{X}\mapsto \mathbb{R}$},  
\State {$S_0 = \Phi $, $\hat{\lambda}_0 = \infty$}

\For{  $t = 1,2,\ldots$ }
   \State{Receive data point $x_t$ ;  $s_t = g(x_t)$} 
   \If{ $s_t \le \hat{\lambda}_{t-1}$}
  { $l_t = 1$}
   \Else { $l_t \sim \texttt{Bernoulli}(p)$}
   \EndIf
    \If{$l_t= 1$}
        \State {$y_t =$ 
 \texttt{GetExpertLabel$(x_t)$}}
     \State{$S_t = S_{t-1} \cup \{(s_t,y_t)\} $}
   \EndIf 
  \State{
$\hat{\lambda}_t := \text{arg } \underset{\lambda  \in \Lambda}{\text{min}}\ \lambda\ \text{ s.t. }\ \FPRhat(\lambda,t) + \psi(t,\delta)\le \alpha $
}

   \If{$l_t=1$} { Output $y_t$}
    \Else { Output $\hat{y}_t = \sign(s_t -\hat{\lambda}_t)$}
    \EndIf
\EndFor
\end{algorithmic}
\end{algorithm}
\vspace{-25pt}
\end{minipage}
\end{wrapfigure}

 

\textbf{Finding threshold using a UCB on FPR.} We propose using our estimated FPR with an upper confidence bound (UCB), which we will describe soon,
to obtain the following optimization problem \eqref{opt-P2},
\begin{tcolorbox}[top=-1pt,bottom=0pt,left=1pt,right=0pt]
\begin{equation}
\hat{\lambda}_t := \text{arg } \underset{\lambda  \in \Lambda}{\text{min}}\ \lambda\ \text{ s.t. }\ \FPRhat(\lambda,t) + \psi(t,\delta)\le \alpha,
\tag{P2}\label{opt-P2}
\end{equation}
\end{tcolorbox}
where the term $\psi(t,\delta)$ is a time-varying upper confidence which is simultaneously valid for all $\lambda$ for all time with probability at least $1-\delta$ for any given $\delta \in (0,1)$. The minimization problem can be solved in many ways. We use a binary search procedure where we search over a grid on $(\Lambda_{\min}, \Lambda_{\max})$ with grid-size $\nu$. 
The procedure searches for a smallest $\lambda$ such that $\FPRhat(\lambda,t) + \psi(t,\delta)\le \alpha$. It uses eq. \eqref{eqn:EmpFPR} to compute the empirical FPR at various thresholds and the confidence interval $\psi(t,\delta)$ given in eq. \eqref{eqn:psi-theory}. Details of the binary search procedure are in the Appendix.


\textbf{Upper confidence bound (UCB).} Our algorithm hinges on having confidence intervals on the FPR that are valid for all thresholds and for all times simultaneously. To construct such bounds,
we use the confidence bounds based on Law of Iterated Logarithm(LIL) of \citep{khinchine1924LIL}. We note that at each time step $t$, whether the sample $x_t$ gets human feedback or not depends on the previous threshold $\hat{\lambda}_{t-1}$ which is a function of data up to time $t-1$ and the importance sampling. Therefore, \emph{the samples used to estimate the FPR are dependent} which prevents direct application of known results that are developed for i.i.d. samples~\citep{howard2022AnytimeValid}. We build upon the LIL bounds for martingales \citep{balsubramani2015LIL} and derive a confidence interval bound that is valid in our setting, which is given by the following equation,
\begin{equation}
  \psi(t,\delta) := \sqrt{\frac{ 3 c_t}{N_t^{(o)}} \bigg [ 2\log \log \Big ( \frac{3 c_t N_t^{(o)}}{2}\Big )  + \log \Big ( \frac{2 L}{\delta} \Big)\bigg ]}, \label{eqn:psi-theory}
\end{equation}
where $c_t = 1-\beta_t + \frac{\beta_t}{p^2}$, $\beta_t = \frac{N_t^{(o,p)}}{N_t^{(o)}}$ and $N_t^{(o,p)}$ is the number of points sampled using importance sampling until time $t$ and $\nu \in (0, 1)$ is a discretization parameter set by the user, $L = (\Lambda_{\max} - \Lambda_{\min})/\nu.$



\section{Theoretical Guarantees}
\label{sec:theory}
\begin{figure*}[t]
    \centering
\includegraphics[width=0.95\textwidth]{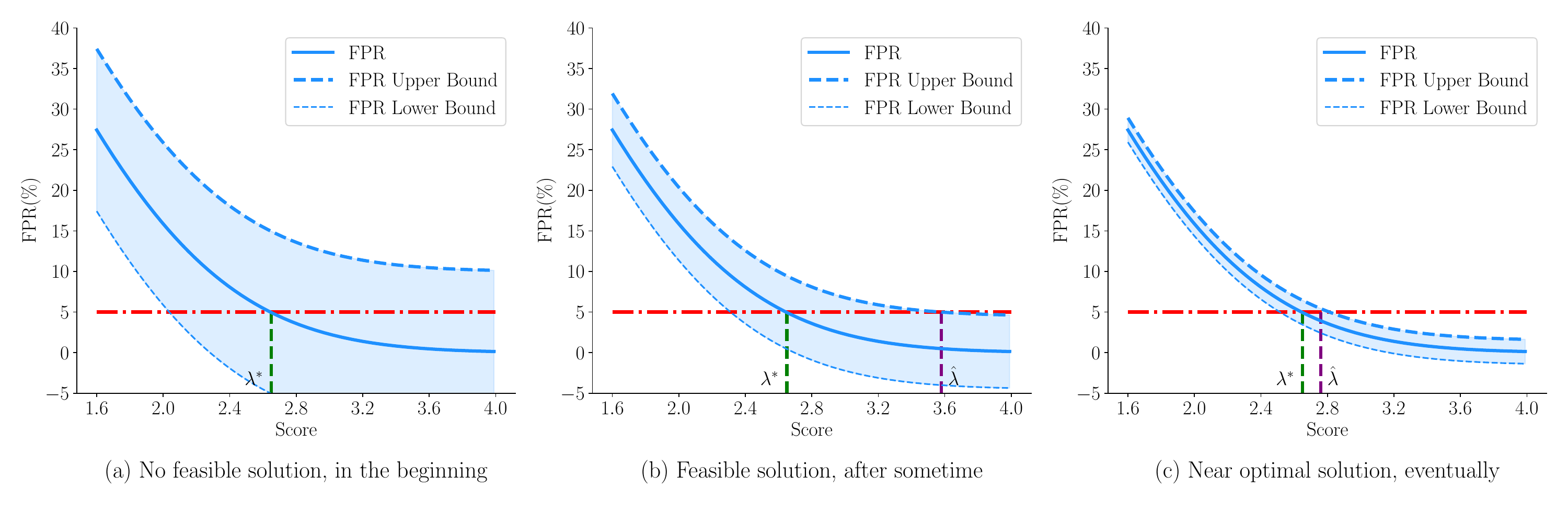}
    \caption{{Illustration of the confidence interval defined in eq. \eqref{eqn:psi-theory} on FPR and their effect on threshold estimation. As the system receives more OOD samples the confidence intervals will shrink and lead to better thresholds safely ($\hat{\lambda}_t \geq \lambda^\star$).
    }}\label{fig:OptP2Explain}
\end{figure*}
We want three provable properties for our approach. First, it must have guaranteed FPR control---the safety property we set out to ensure. In addition, we would like to show a bound on the number of streamed observations (i.e., the time) taken to reach a point where every point does not need human feedback to ensure safety. Finally, we wish to have some notion of optimality, and a bound on the number of observations before it is reached. In this section, we provide a result that provides all of these properties under the following assumptions: (i) 
we are in the stationary setting, i.e., the distributions do not change over time, and (ii) the score distributions for the ID and OOD samples have sub-Gaussian tails. 

To quantify how close to the optimal operating point the system is at any given time, we define the following notion of $\eta$-optimality.
\begin{definition}($\eta$-optimality)
\label{defn:optimality}
   For any $\eta > \FPR(\lambda^*) - \lim_{\epsilon \to 0^{+}} \FPR(\lambda^* + \epsilon)$, the system is said to be operating in the $\eta$-optimal regime after some time point $T_\eta$, if  $\FPR(\lambda^*) - \FPR(\hat{\lambda}_t) \le \eta$ for all $t\ge T_{\eta}$. 
\end{definition}
Using the estimated FPR in eq.~\eqref{eqn:EmpFPR} and the anytime valid confidence intervals on the FPR at all thresholds we obtained in eq.~\eqref{eqn:psi-theory}, we provide the following guarantees for Algorithm~\ref{alg:adaptive-ood}. 
\begin{theorem}
Let $\alpha, \delta, p, \gamma \in (0, 1)$. Let $x_t \overset{i.i.d}{\sim}(1-\gamma) \cD_{id} + \gamma \cD_{ood}$ and let $c_t = 1-\beta_t + \frac{\beta_t}{p^2}$, $\beta_t = \frac{N_t^{(o,p)}}{N_t^{(o)}}$ where $N_t^{(o,p)}$ is the number of OOD points sampled using importance sampling until time $t$ and $N_t^{(o)}$ is the total number of OOD points observed till time $t$. Let $n_0 = \min\{u: c_u N^{(o)}_u \ge 173\log(\frac{8}{\delta}) \}$ and $t_0$ be such that $N_{t_{0}}^{(o)}\ge n_0$. If Algorithm~\ref{alg:adaptive-ood} uses the optimization problem~\eqref{opt-P2} to find the thresholds with the upper confidence term $\psi(N_t^{(o)},\delta/2)$ given by eq.~\eqref{eqn:psi-theory}, then there exist constants $C_1, C_2, C_3 >0$ such that with probability at least $1-\delta$, 
\begin{enumerate}[leftmargin=15pt,itemsep=0.25pt,topsep=0pt]
    \item \textbf{Controlled FPR.} For all $t \geq t_0$, $\FPR(\hat{\lambda}_t) \leq \alpha$.
    \item \textbf{Time to reach feasibility.} The algorithm  will find a feasible threshold, $\hat{\lambda}_t$ such that $\FPRhat(\hat{\lambda}_t) + \psi(N_t^{(o)}) \le \alpha$, for all $t \ge \max(t_0,T_f)$ , where, 
    $ T_f = \frac{2C_1}{\gamma \alpha^2} \log \Big( \frac{4C_2}{\delta}\log(\frac{C_3}{\alpha}) \Big) + \frac{1}{\gamma^2}\log(\frac{4}{\delta}) $.
    \item \textbf{Time to reach $\eta-$optimality.}  For all $t \ge \max(t_0,T_{\eta\text{-opt}})$, $\hat{\lambda}_t$ satisfy the $\eta$-optimality condition in definition \ref{defn:optimality}, when
    $\FPRhat(\hat{\lambda}_{T_{\eta\text{-opt}}}) \in [\alpha -\eta/2, \alpha]$ and $T_{\eta\text{-opt}} = \frac{8C_1}{\gamma \eta^2} \log \Big( \frac{4C_2}{\delta}\log(\frac{2C_3}{\eta}) \Big) + \frac{1}{\gamma^2}\log(\frac{4}{\delta})$.
  
\end{enumerate}
\label{thm:main_theorem}
\end{theorem}

%

We discuss the results below and defer the proofs to the Appendix \ref{subsec:proofs}.

\textbf{Controlled false positive rate.} We design our framework with the goal of safely updating the threshold. Our method guarantees that $ \hat{\lambda}_{t} \geq \lambda^\star$ at all times, i.e., we approach the optimal $\lambda^\star$ from above and therefore we never violate the FPR constraint.
This property is crucial in applications where accurately controlling the FPR is essential. \\ 

\textbf{Time to reach feasibility.} Our algorithm begins with setting $\lambda_0 = \Lambda_{\max}$, and therefore obtaining human feedback on all the points until the time point when we can find a threshold that enables us to safely declare scores above it as ID (Fig~\ref{fig:OptP2Explain}b). 
Our analysis provides an upper bound on the time taken by the algorithm to find such a safe, $\hat{\lambda}_t < \Lambda_{\max}$ such that, $\text{FPR}(\hat{\lambda}_t) \leq \alpha$. We call this time as \emph{time to reach feasibility}, $T_f$, and it is the time step at which a sufficient number of observations $N_{T_f}^{(o)}$ is obtained so that the confidence interval $\psi(T_f,\delta/2) \le \alpha$. It is inversely proportional to the level $\alpha$ and the fraction of OOD samples $\gamma$.

\textbf{Time to reach $\eta$-optimality.} As the time proceeds and the confidence intervals around the estimated FPR at different thresholds start to get smaller, the estimated safe threshold $\hat{\lambda}_t$ starts to approach $\lambda^\star$ (Fig~\ref{fig:OptP2Explain}). If the estimated FPR at time step $T_{\eta\text{-opt}}$, denoted as $\FPRhat(\hat{\lambda}_{T_{\eta\text{-opt}}})$, is within the range $[\alpha - \eta/2, \alpha]$ and the confidence interval $\psi(T_{\eta\text{-opt}},\delta/2) \le \eta/2$, then for all time points after $T_{\eta\text{-opt}}$ the algorithm will find a $\hat{\lambda}_t$ that satisfies the $\eta$-Optimality condition. In this regime, the algorithm operates in a state where the difference between the FPR at the true optimal threshold, $\FPR(\lambda^*) = \alpha$, and the FPR at the estimated threshold $\FPR(\hat{\lambda}_t)$, is bounded by $\eta$. Our analysis provides a bound on the time $T_{\eta\text{-opt}}$ which is the time point when the number of acquired OOD samples $N_{T_{\eta\text{-opt}}}^{(o)}$ becomes at least $\frac{4C_1}{\eta^2}\log \Big( \frac{2C_2}{\delta}\log(\frac{2C_3}{\eta}) \Big)$. It is inversely proportional to the closeness to optimality $\eta$ and the fraction of OOD samples $\gamma$. 



The details of the proof are available in the appendix.

\begin{figure*}[t]
  \centering
\includegraphics[width=0.90\textwidth]
  {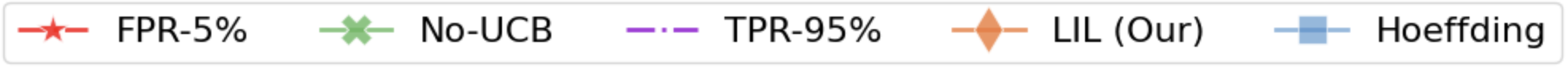}
  \mbox{
    \subfigure[ No window, no distribution shift. \label{fig:sim_no_shift_no_window}]{\includegraphics[scale=0.26]{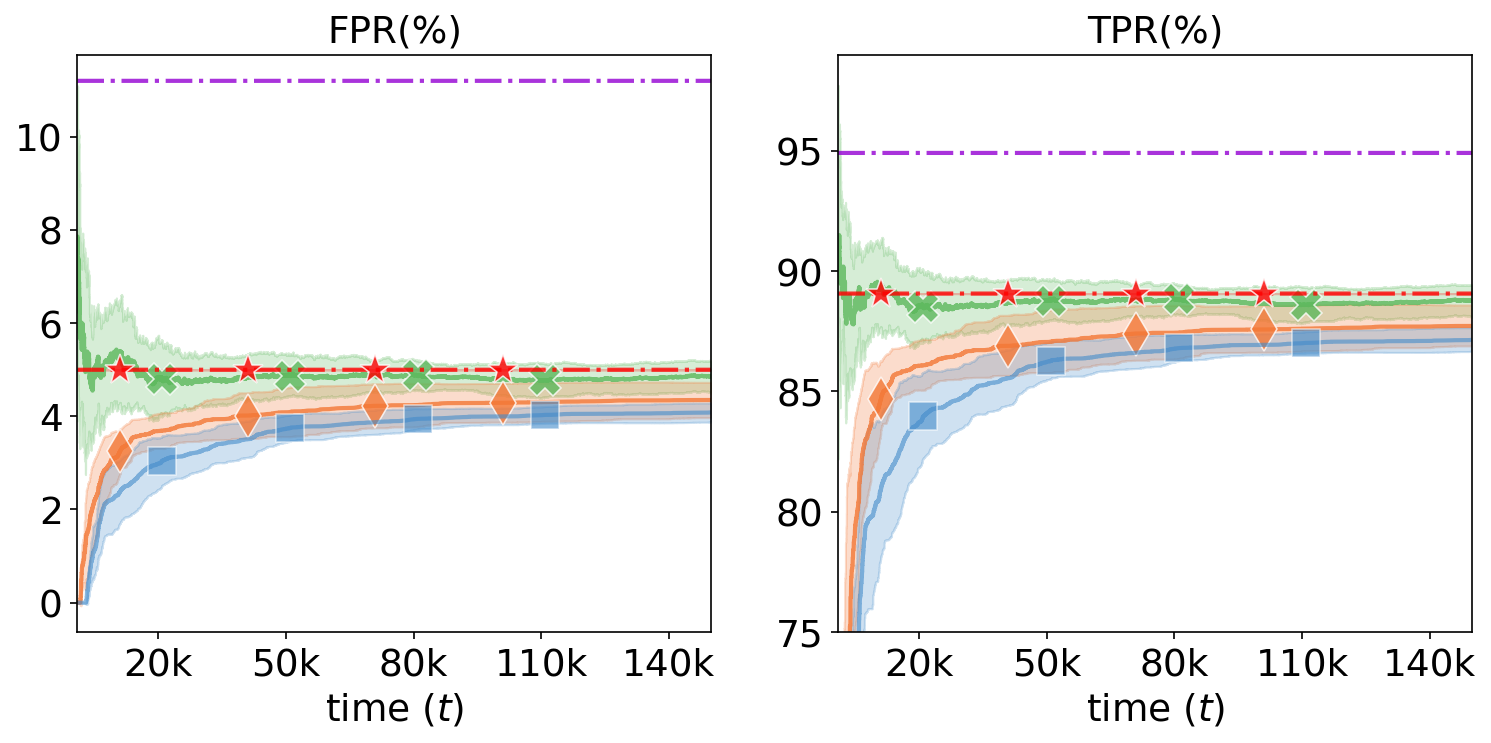}}
    \hspace{50pt}
    \subfigure[ 5K window, no distribution shift\label{fig:sim_no_shift_win_5k}]{\includegraphics[scale=0.26]{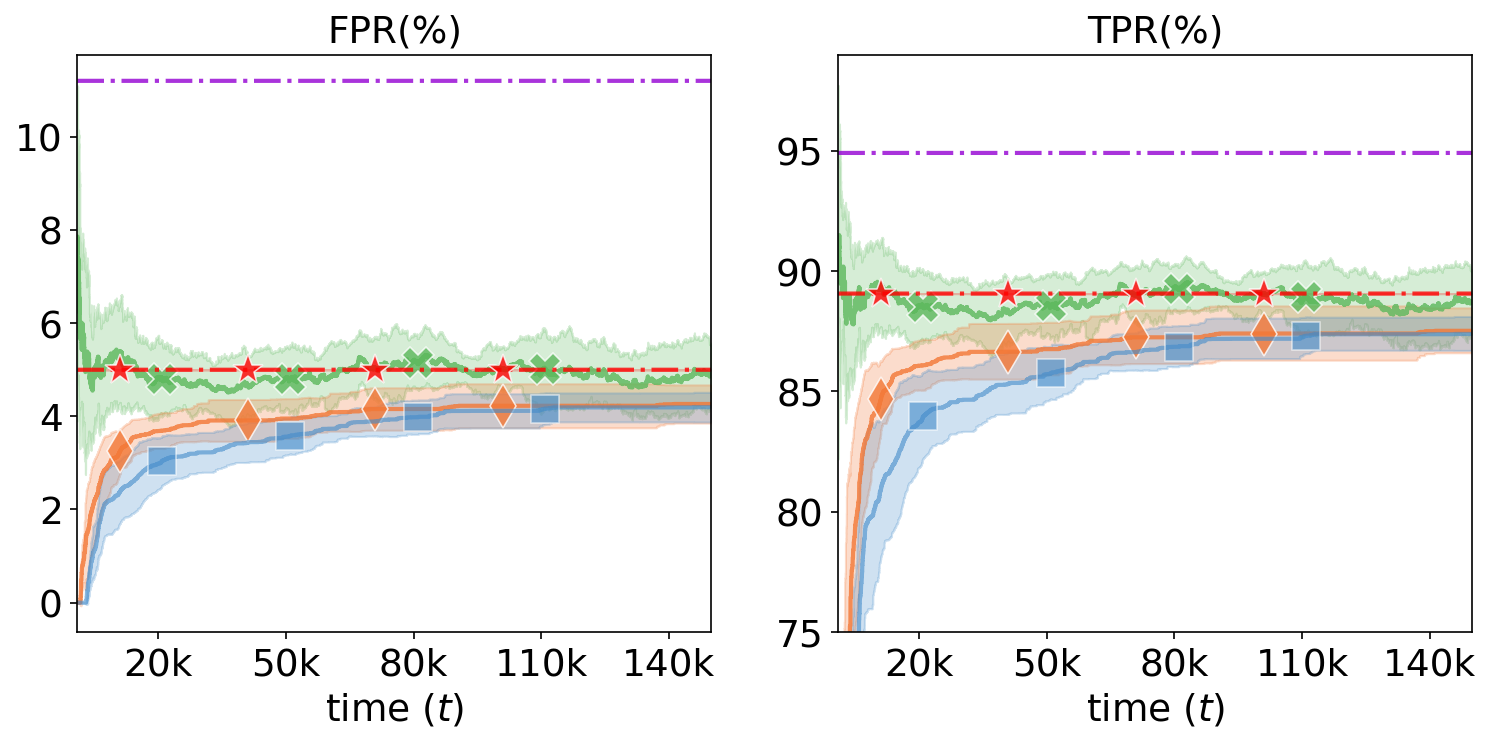}}
  }
  \caption{ Results on the synthetic data with stationary distributions, $\gamma=0.2$, and using no window. Each method is repeated 10 times. The mean and standard deviation are shown.}
    
    \label{fig:sim-no-change-41-main}
    \vspace{-5pt}
\end{figure*}

\begin{figure*}[t]
  
  \centering

  \mbox{
  \hspace{-10pt}
    \subfigure[Distribution shift, no window.  \label{fig:sim_shift_win_no}]{\includegraphics[scale=0.215]{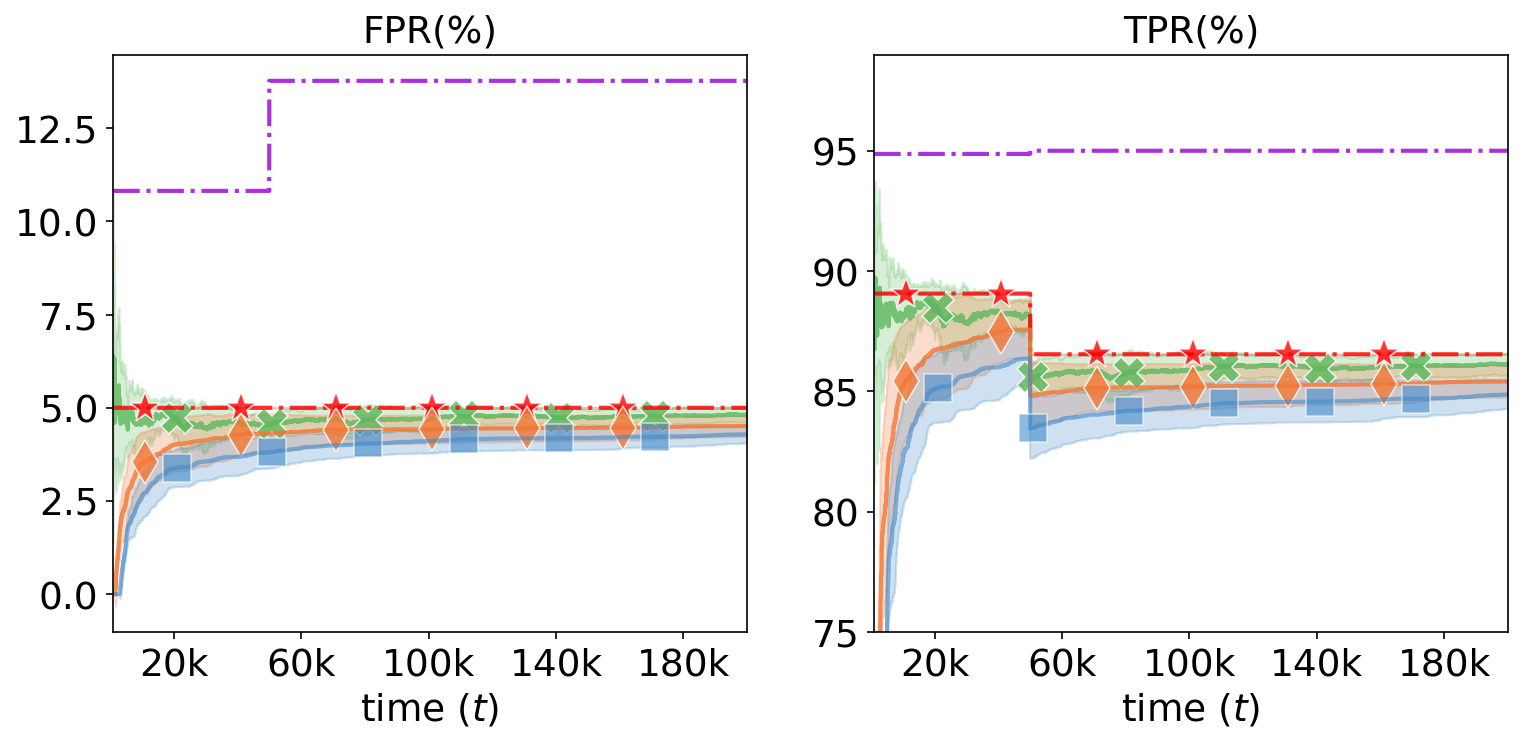}}
    
    \subfigure[Distribution shift, 5k window. \label{fig:sim_shift_win_5k}]{\includegraphics[scale=0.215]{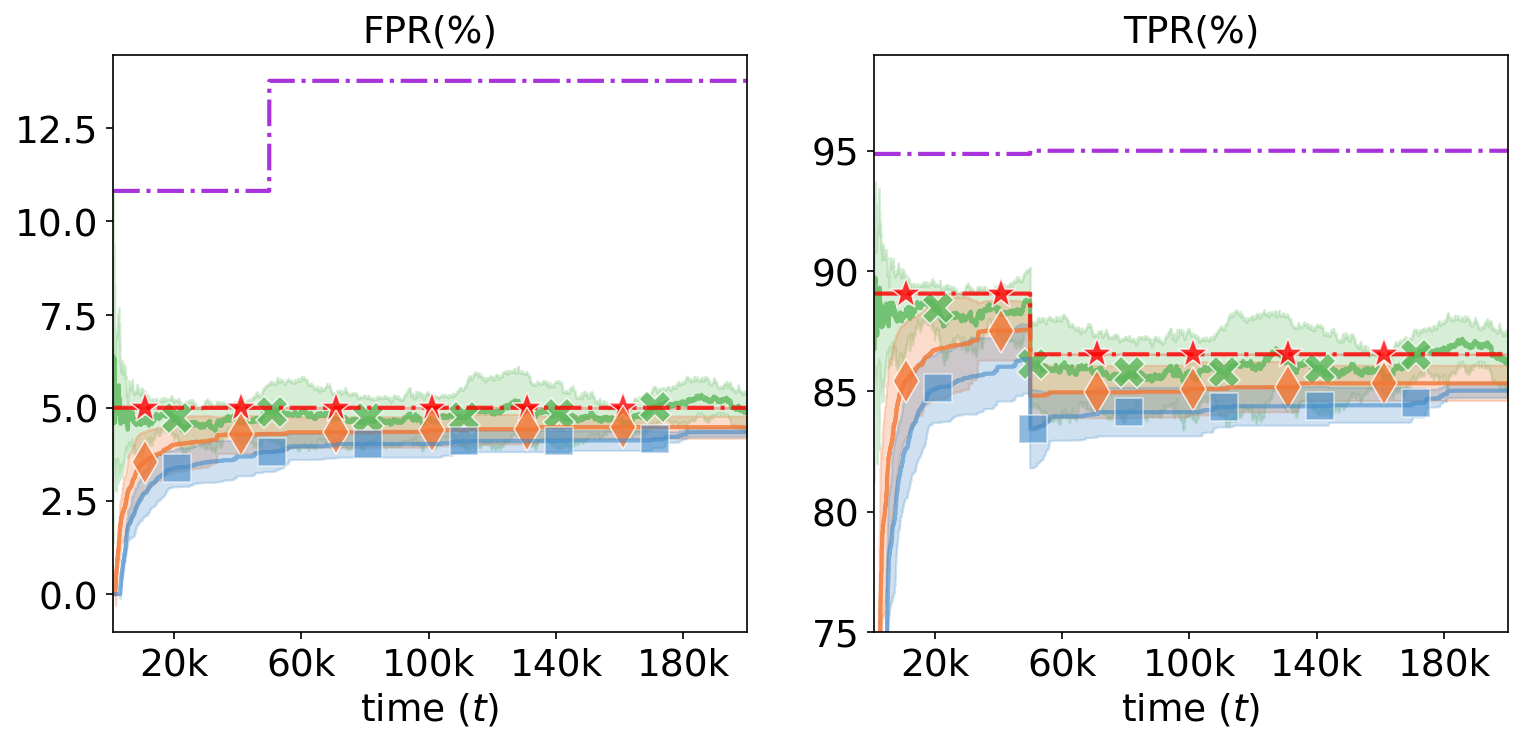}}

    \subfigure[Distribution shift, 10k window. \label{fig:sim_shift_win_10k}]{\includegraphics[scale=0.215]{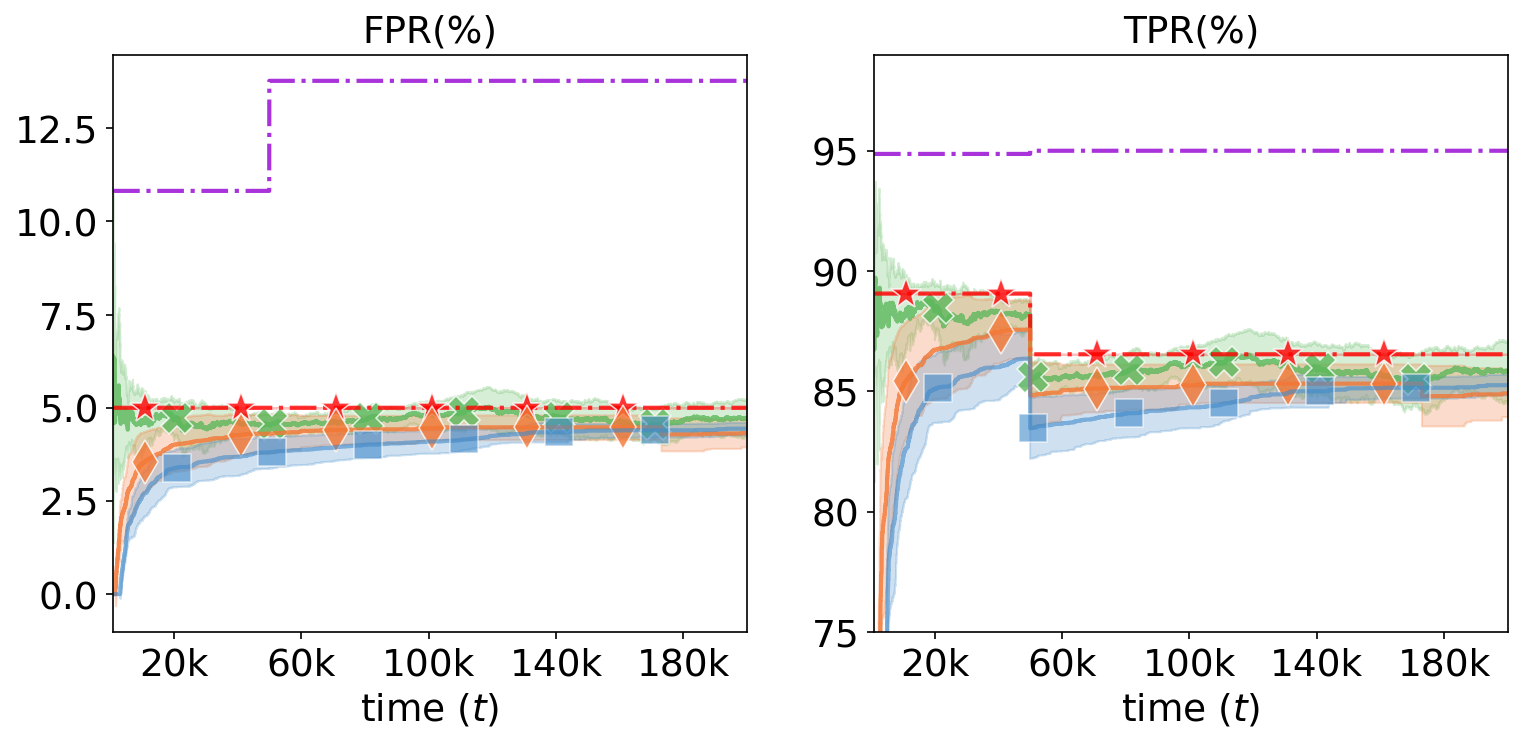}}
  }
  \caption{ 
    Effect of using various window sizes in synthetic data experiments. The distribution shift starts at $t=50k$. The arrow indicates the time at which the mean FPR + std. deviation over 10 runs goes below 5\% for the LIL method.
    }
  \label{fig:sim-shift-all} 
  
\end{figure*}

\begin{figure*}[t]
  \centering
 \includegraphics[width=0.98\textwidth]
  {figs-v2/legend_hfdng-3.png}
  \mbox{
  \hspace{-10pt}
    \subfigure[No distribution shift, no window.   \label{fig:knn_no_shift_no_win_cifar10_main}]{\includegraphics[scale=0.215]{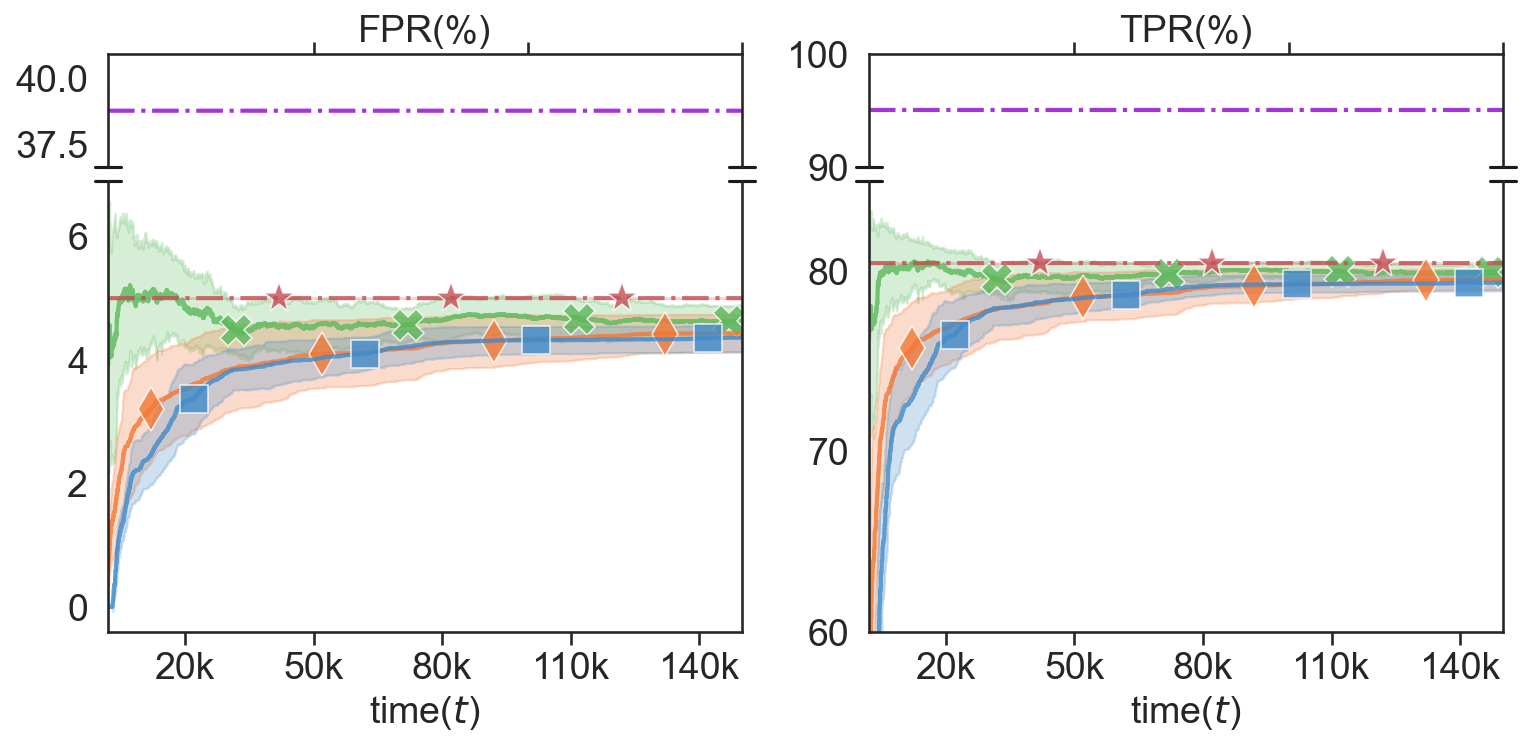}}
    \hspace{1pt}
    
    \subfigure[Distribution shift, 5k window. \label{fig:knn_shift_win_5k_cifar10_main}]{\includegraphics[scale=0.215]{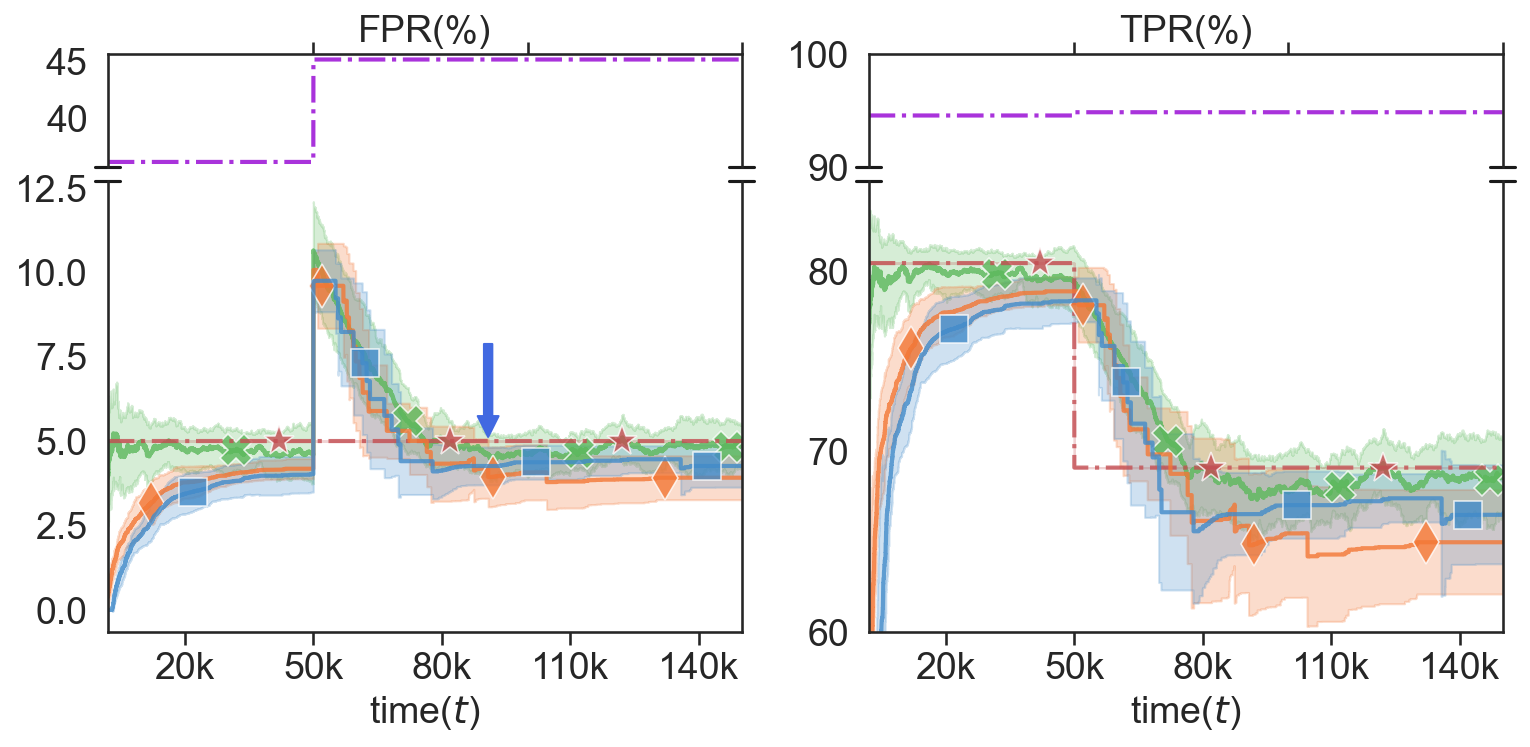}}

    \subfigure[Distribution shift, 10k window. \label{fig:knn_shift_win_10k_cifar10_main}]{\includegraphics[scale=0.215]{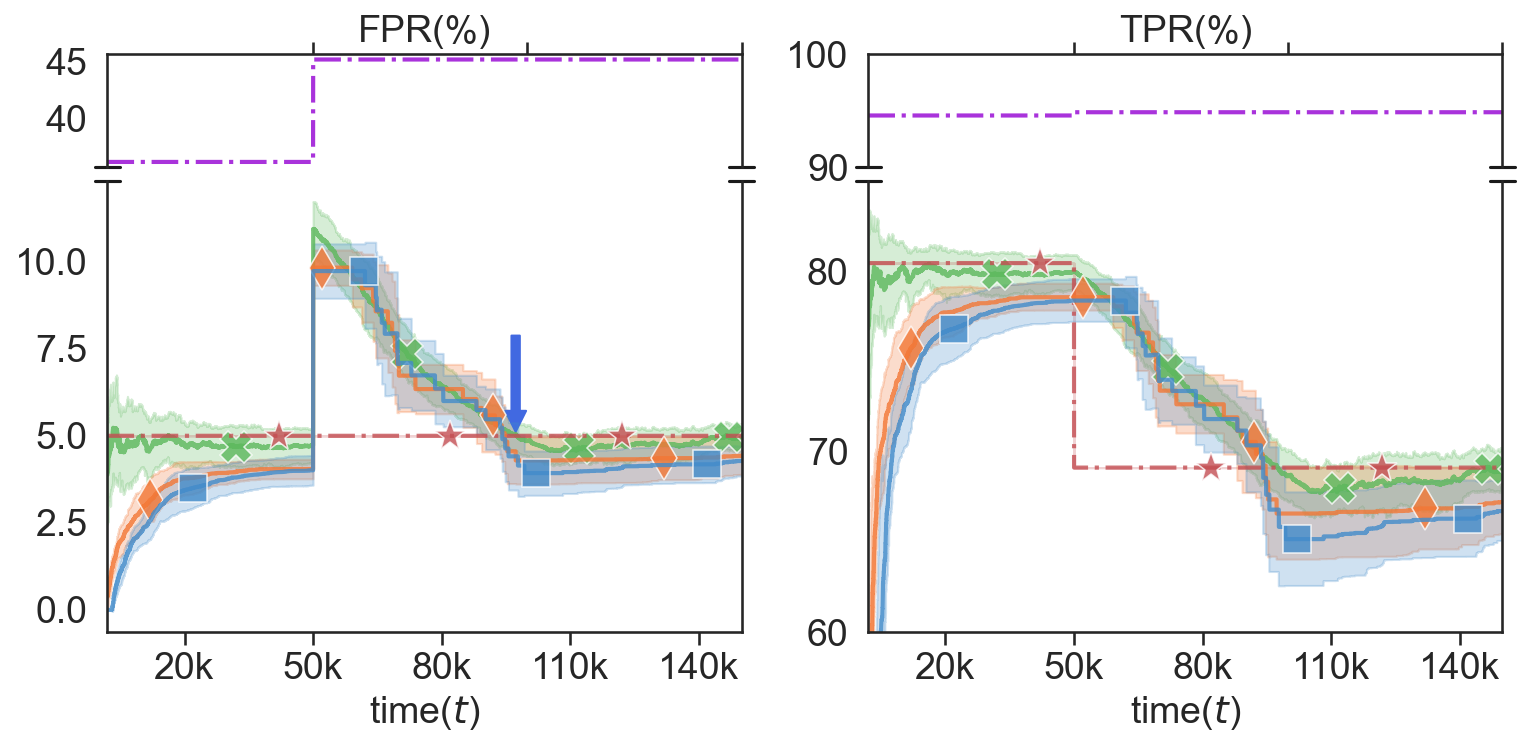}}
  }
  
  \caption{  {
   Results with the KNN scores on Cifar-10 as ID dataset. For (b) and (c) the distribution shifts at $t=50k$. The arrow indicates the time at which the mean FPR + std. deviation over 10 runs goes below 5\% for the LIL method.
    }}
  \label{fig:knn-cifar10_main}
  \vspace{-5pt}
\end{figure*}

\section{Empirical Evaluation}
\label{sec:experiment}

We evaluate our method to verify the following claims:
\begin{enumerate}[leftmargin=0pt,topsep=0pt, label={}]
\item {\bf C1.} Compared to non-adaptive baselines, our approach achieves lower FPR while maximizing the TPR. 
\item {\bf C2.} In the stationary setting, our adaptive method based on the LIL upper confidence bound satisfies the FPR constraint at all times and produces high TPR.
\item {\bf C3.} The proposed framework is compatible with any OOD scoring functions.
\item {\bf C4.} Our method continues to work even in distribution shift settings with a simple adaption using the windowed approach described in Section \ref{sec:dist_shit_exp}.
\end{enumerate}


\textbf{Baselines.}  We compare our method against the non-adaptive baseline popularly used for OOD detection. This non-adaptive method (\textbf{TPR-95}) finds a threshold achieving 95\% TPR using the ID data and uses it at all times. For our adaptive method, we consider three choices of confidence intervals  \textbf{ i) No-UCB:}  does not use any confidence intervals, \textbf{  ii) LIL:} Uses confidence interval from eq.~\eqref{eqn:lil-heuristic},  and  \textbf{iii) Hoeffding:}  Uses the confidence intervals from Hoeffding's inequality \citep{hoeffding1963}. The confidence intervals from Hoeffding inequality are not valid simultaneously for all times but are a reasonable choice for a practitioner. 
\begin{equation}
 \tilde{\psi}(t,\delta) =  C_1\sqrt{\frac{  c_t}{N_t^{(o)}} \bigg ( \log \log \big (C_2 c_t N_t^{(o)} \big )  + \log \Big ( \frac{C_3}{\delta} \Big )\bigg )} . 
  \tag{LIL-Heuristic}
  \label{eqn:lil-heuristic}
\end{equation}

\begin{wrapfigure}{r}{0.45\textwidth}
  \vspace{-15pt}
   \centering
    {\includegraphics[scale=0.3]{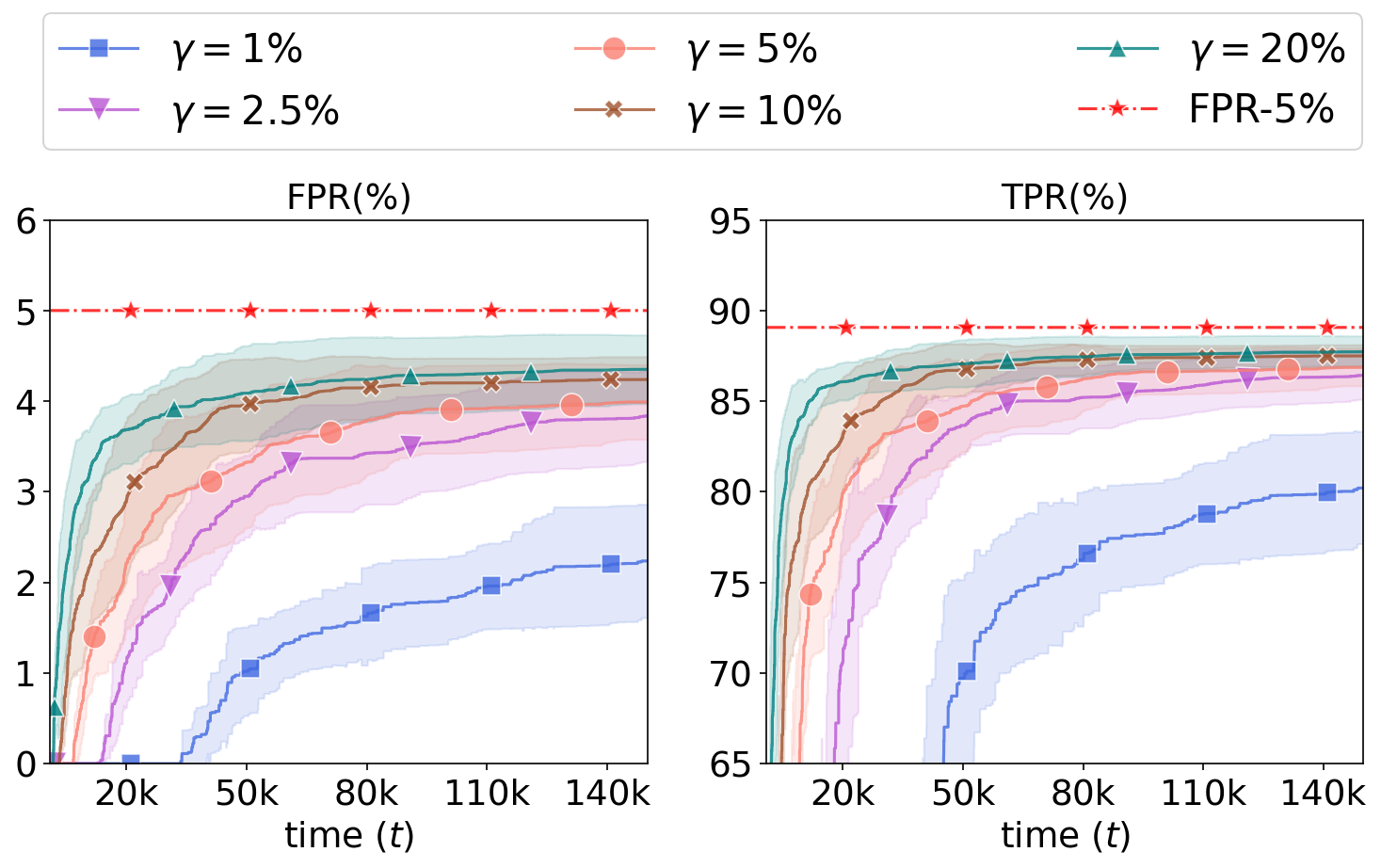}}

\vspace{-5pt}

  \caption{
  \small{
      Results on the synthetic data with stationary distributions and different mixing ratios $\gamma$.
     }
    }
    \label{fig:sim-gamma-change}
    \vspace{-10pt}
\end{wrapfigure}

The theoretical LIL bound in eq.~\eqref{eqn:psi-theory} has constants that can be pessimistic in practice. We get around this by using a LIL-Heuristic bound which has the same form as in eq.~\eqref{eqn:psi-theory} but with different constants. We consider the form in eq.~\eqref{eqn:lil-heuristic}. We find the constants $C_1, C_2, C_3$ using a simulation on estimating the bias of a coin with different constants and picking the ones so that the observed failure probability is below 5\%. We use $C_1 = 0.5$ and $C_2 = 0.75$, $C_3=1$. We use $\alpha = 0.05$, $\delta=0.2$, and importance sampling probability $p=0.2$ through all the empirical evaluations.  More details are available in the Appendix \ref{subsubsec:constant-search} and the code \footnote{{\url{https://github.com/2454511550Lin/TameFalsePositives-OOD}}}.  
\begin{wraptable}{r}{0.45\textwidth}
 \scalebox{0.75}{
\small{
\begin{tabular}{ccrrrr} 
 \toprule 
\multirow{2}{*}{$\bm{\gamma}$} & \multirow{2}{*}{$\bm{T_f}$} & \multicolumn{4}{c}{$\bm{T_{\eta-\text{opt}}}$} \\ 
\cmidrule{3-6} 
& 	 & 	$\bm{\eta=1.0\%}$	 & 	$\bm{\eta=1.5\%}$	 & 	 $\bm{\eta=2.0\%}$	 & 	 $\bm{\eta=2.5\%}$ 	 \\ 
\toprule 
\multirow{2}{*}{\textbf{2.5\%}} & 14,167	 & 93,011 	 & 71,089 	 & 70,559 	 & 37,534 \\ 
	 	    & \tiny{$ \pm 602 $ }	 & \tiny{ $ \pm 27,387 $ }  	 & \tiny{ $ \pm 25,654 $ }  	 & \tiny{ $ \pm 35,056 $ }  	 & \tiny{ $ \pm 9302 $ }  	 \\ 

 \midrule 
\multirow{2}{*}{\textbf{5\%}} & 7,054	 & 53,971 	 & 47,143 	 & 39,864 	 & 32,473 \\ 
	 	    & \tiny{$ \pm 301 $ }	 & \tiny{ $ \pm 20,816 $ }  	 & \tiny{ $ \pm 24,004 $ }  	 & \tiny{ $ \pm 22,328 $ }  	 & \tiny{ $ \pm 20,262 $ }  	 \\ 

 \midrule 
\multirow{2}{*}{\textbf{10\%}} & 3,549	 & 50,748 	 & 35,517 	 & 26,435 	 & 17,312 \\ 
	 	    & \tiny{$ \pm 200 $ }	 & \tiny{ $ \pm 33,947 $ }  	 & \tiny{ $ \pm 22,131 $ }  	 & \tiny{ $ \pm 14,361 $ }  	 & \tiny{ $ \pm 7,757 $ }  	 \\ 

 \midrule 
\multirow{2}{*}{\textbf{20\%}} & 1,770	 & 40,240 	 & 28,943 	 & 9,004 	 & 6,500 \\ 
	 	    & \tiny{$ \pm 72 $ }	 & \tiny{ $ \pm 37,751 $ }  	 & \tiny{ $ \pm 31,138 $ }  	 & \tiny{ $ \pm 3,383 $ }  	 & \tiny{ $ \pm 2,495 $ }  	 \\ 
\bottomrule  
\end{tabular}
}
}
    \caption{\small{Time to reach feasibility $T_f$ and optimality $T_{\eta-\text{opt}}$ in the stationary setting for different $\eta$ and mixing ratios $\gamma$.}}
    \label{tab:time-table-gamma}
    \vspace{-8pt}
\end{wraptable}

\textbf{Synthetic data setup.}
We simulate the OOD and ID scores using a mixture of two Gaussians $\mathcal{N}_{id}(\mu=5.5,\sigma=4)$ and  $\mathcal{N}_{ood}(\mu=-6,\sigma=4)$. We randomly draw 100k samples with $\gamma=0.2$ (see Figure~\ref{fig:sim-no-change-41-main}). 

\textbf{Real data setup.}
   We use ID and OOD datasets and scoring functions from the OpenOOD benchmark  \citep{yang2022openood}. Here we show the results on CIFAR-10 \citep{krizhevsky2009learning} as an ID dataset and show the results on CIFAR-100, and Imagenet-1K  \citep{deng2009imagenet} in the Appendix \ref{sec:additional_exp}. To verify C3, we use various scoring functions: ODIN  \citep{liang2017enhancing}, Mahalanobis Distance  \citep{lee2018simple}, Energy Score  \citep{liu2020energy}, SSD  \citep{sehwag2021ssd}, VIM  \citep{wang2022vim}, and KNN  \citep{sun2022out} scores for the evaluation. Due to space limitation, we present results for the KNN  \citep{sun2022out} score here. For more details on the datasets, scores, and results on the rest of the scores please see the Appendix \ref{sec:additional_exp}.  







\subsection{Stationary Distributions Setting}
In the stationary setting the data distributions do not change over time. We use this setting to verify our theoretical claims as they are valid in such settings. We perform the experiments to verify claims C1 and C2 on synthetic and real data. See Figures  \ref{fig:sim_no_shift_no_window} and \ref{fig:knn_no_shift_no_win_cifar10_main} for the results. We make the following observations: (i)
We see that the non-adaptive method (TPR-95) with the fixed threshold has a high FPR at all times and violates the FPR constraint by a big margin. On the other hand, the adaptive methods improve with time. (ii) We see that not using a UCB leads to violation of FPR constraints and the methods with LIL-Heuristic, Hoeffding based intervals are able to maintain the FPR below the user given threshold $5\%$. Moreover, all the methods improve as they acquire more samples with time and eventually reach very close to the optimal solution. We note that our method (LIL-heuristic) is faster in this regard than while maintaining safety.

\textbf{Time to reach feasibility and optimality.}
In our theoretical results, we derived bounds on the time to reach feasibility ($T_f$) and the time to reach $\eta$-optimality denoted by $T_{\eta-\text{opt}}$. These times are inversely proportional to the mixing ratio $\gamma$ and the optimality level $\eta$. To verify this we run the LIL method on the synthetic data setup with different values of $\gamma$ and observe $T_f$ (corresponding to $\alpha=5\%$) and $T_{\eta-\text{opt}}$. We report the mean and std. deviation of $T_f$ and  $T_{\eta-\text{opt}}$ over 10 runs with different random seeds (see Table \ref{tab:time-table-gamma}). We see both $T_f$ and  $T_{\eta-\text{opt}}$ decrease as $\gamma$ increases and  $T_{\eta-\text{opt}}$ is also inversely proportional to the optimality level. The corresponding FPR and TPR trends for each $\gamma$ are shown in Figure \ref{fig:sim-gamma-change}. These trends also corroborate our understanding of the effect of $\gamma$ on the time for feasibility and optimality.


\subsection{Distribution Shift Setting}
\label{sec:dist_shit_exp}

We now proceed to investigate the case where the distributions change at a specific time point. One of the motivations for the proposed system is to be able to adapt to the variations of the OOD data. As long as $\cD_{ood}$ does not change, any changes in the $\cD_{id}$ or the mixing ratio $\gamma$ do not affect the true FPR and therefore the optimal $\lambda^\star$. However, the true FPR does get affected when $\cD_{ood}$ changes. When there is a change in $\cD_{ood}$, estimating the FPR using all the acquired samples so far heavily biases the estimate towards scores that are far behind in time from the previous $\cD_{ood}$. This leads to a long delay before our unbiased estimate of FPR can catch up to the change. 

\begin{wrapfigure}{r}{0.45\linewidth}
  
  \centering
    {\includegraphics[scale=0.3]{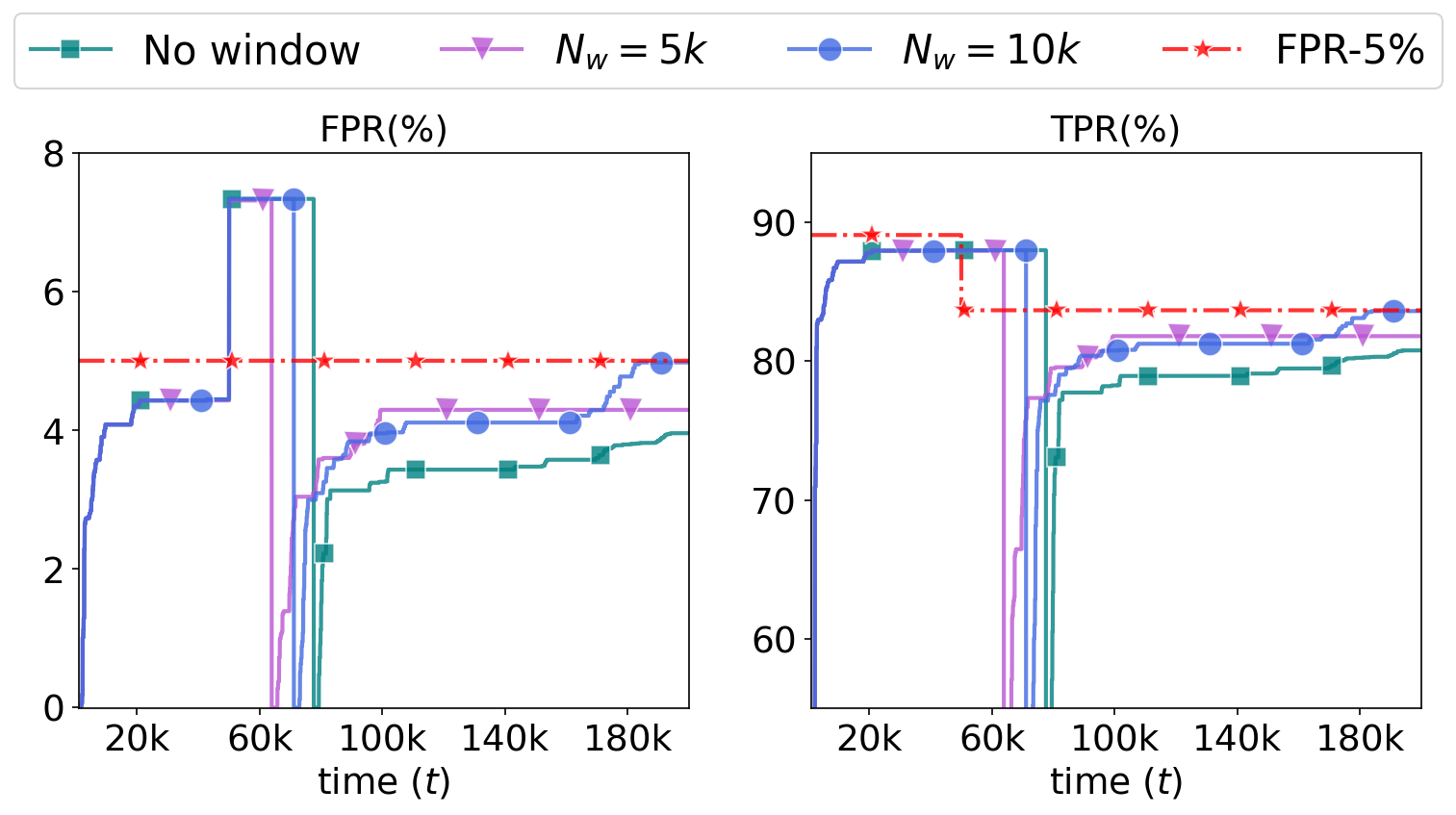}}
  
  \caption{
  {
    Effect of change detection and restart on the performance with different window sizes.
    }}
  \label{fig:sim-change_restart} 
  \vspace{-10pt}
\end{wrapfigure}

\textbf{Windowed approach. } To overcome this challenge, we propose a sliding window-based approach with the adaptive methods. The user can set a window size $N_w>0$ and the system will estimate the FPR and the confidence intervals using only the most recent $N_w$ samples that are determined as OOD by human feedback. This allows the system to more quickly adapt the threshold that is well aligned with the new distribution(s) of OOD samples. 
\begin{wrapfigure}{r}{0.45\textwidth}
  \centering
  
   {\includegraphics[scale=0.36]{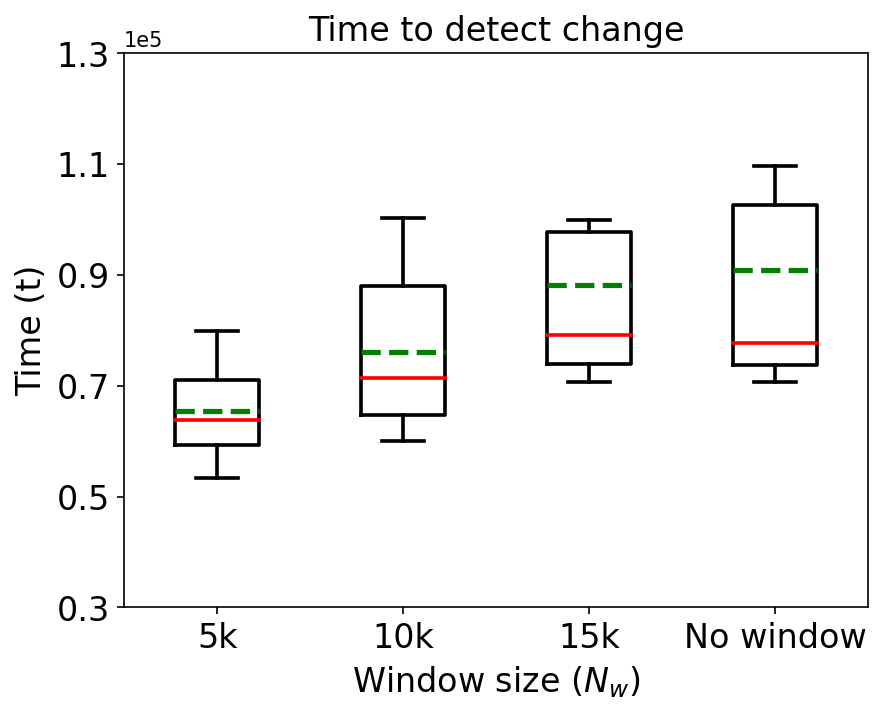}}

    \vspace{-5pt}

  \caption{
  {
    Change detection times with different window sizes on synthetic data. The median and mean are shown using solid and dashed lines respectively.
    }}
    \label{fig:sim-box-time-change} 
  \vspace{-15pt}
\end{wrapfigure}

\textbf{Change detection. } We use the following criteria to detect change, if $\FPRhat(\hat{\lambda}_{t-1},t) -\psi(t,\delta) > \alpha $ then it means the OOD distribution has changed. Here the $\FPRhat$ and $\psi$ are computed using the samples in the window. We use change detection only for the methods with confidence intervals. See the appendix for more details on this criteria.


The window size $N_w$ has trade-offs, i.e., using a smaller window will enable faster change detection and adapting to the new distribution but imposes limitations on the optimality as the smallest width of the confidence interval possible is inversely proportional to the window size. 
We verify claim C4 and study these trade-offs using experiments on synthetic and real data.

 For the synthetic data, we use the same ID and OOD distribution as above till $t=50k$ and change the OOD distribution to $\mathcal{N}_{ood}(\mu=-5,\sigma=4)$ at time $t=50k$. In the real data setting, we use the CIFAR-10 as ID and a mixture of MNIST, SVHN, and Texture datasets as OOD till $t=50k$ and a mixture of TinyImageNet, Places365, CIFAR-100 as OOD after $t=50k$.
We run TPR-95, and adaptive methods LIL, and Hoeffding in these settings with different choices of window sizes with 10 repeated runs using different random seeds and show the mean and std. deviations of the FPR and TPR in Figures \ref{fig:sim-shift-all} and \ref{fig:knn-cifar10_main}. 

We find that the windowed approach adapts more quickly compared to the method without a window (see Figures \ref{fig:sim_shift_win_no},\ref{fig:sim_shift_win_10k}). 


\textbf{Effect of window size.}
The results with various window sizes are shown in Figures \ref{fig:sim_shift_win_5k}, \ref{fig:sim_shift_win_10k} on synthetic data and in Figures \ref{fig:knn_shift_win_5k_cifar10_main}, \ref{fig:knn_shift_win_10k_cifar10_main} on real data. We also show the box plots of change detection times with LIL in Figure \ref{fig:sim-box-time-change}. As expected, we see that with smaller window size the change is detected earlier and the method is able to adapt faster. 
We also see that while smaller window helps in faster adaption but limits how close to the optimal TPR is achieved. 

To showcase the effect of \textbf{window-based approach in stationary setting}, 
we run the methods with a window size of 5k in the fixed distribution setting (see Figure \ref{fig:sim_no_shift_win_5k}). We observe similar behavior to the case without a window but with higher variance as the confidence intervals are limited by the window size.

\textbf{Restart after change detection.} In the previous experiments the algorithm kept using all the samples in its window even after detecting the change. The window can contain samples from the previous distribution till some time which leads to prolonged violation of FPR constraint. In safety critical applications one might apply a conservative approach i.e., restart the algorithm after detecting the change by emptying the window and resetting the threshold to $\Lambda_{\text{max}}$. We run the LIL based method with this variation using different window sizes and show the FPR and TPR of a \emph{median trend} in Figure \ref{fig:sim-change_restart}. To get the median trend we run the algorithm with 10 random seeds and pick the FPR, and TPR trends corresponding to the run having the median change detection time. We see that the FPR and TPR drop to 0 immediately after the change is detected and then the method recovers gradually. The variation without a window took a longer time to detect the change and hence lags behind the window-based methods in approaching optimality after restart. With a 5k window, it detected the change earlier but due to the small window, it is not able to reach optimality. The one with a window size of 10k is a good trade-off as it is neither too late in detecting the change nor lags too far in approaching the optimality.


\vspace{-5pt}
\section{Related Works}
\vspace{-5pt}

Recent years have seen significant advancements in OOD detection. Surveys by \citet{salehi2021unified} and \citet{yang2022generalized} compare outlier, anomaly, and novelty detection, framing generalized OOD detection. We explore related research in these domains.

\textbf{Out-of-Distribution detection.}
 Many recent works proposed methods to quantify a score (uncertainty) that can separate OOD and ID data points. These include confidence-based scores \citep{Bendale2015TowardsOS, hendrycks2017a}, Temperature scaling as used in ODIN \citep{liang2018enhancing},  Energy-based scores \citep{liu2020energy, Wang2021CanMC, wu2023energybased}. With-in the vein of post-hoc methods ReACT \citep{sun2021react} and ASH \citep{djurisic2023ASH} employ light-weight perturbations to activation functions at inference time. VIM \citep{wang2022vim}  generates a virtual OOD class logit and matches it with original logits using constant scaling. GradNorm \citep{huang2021importance}  GradOrth \citep{behpour2023gradorth} use gradient information.
 Distance-based methods detect samples as OOD if they are relatively far away from ID data. These include minimum Mahalanobis distance from class centroids \citep{lee2018simple, ren2023ood},  distances between representations learned using self-supervised learning \citep{tack2020csi, sehwag2021ssd}, non-parametric KNN distance  \citep{sun2022out} and distances on Hyperspherical embeddings \citep{ming2022cider}. Density-based approaches employ probabilistic models to characterize the density of ID data, positing that OOD points should appear in regions of low density. Some of the notable works in this vein are \citep{bishop1994novelty}, \cite{nalisnick2019hybrid} likelihood ratios \citep{ren2019LRT}, normalizing flows \citep{kirichenko2020normalizing}, invertible networks \citep{schirrmeister2020understanding}, input complexity \citep{Serrà2020Input}, confusion log probabilities 
 \citep{Winkens2020ContrastiveTF}, likelihood regret scores for variational autoencoders \citep{xiao2020Regret}.

\textbf{Training-time regularization} addresses OOD detection by reducing the confidence or increasing the energy on the OOD points using auxiliary OOD data for model training \citep{bevandić2018discriminative, geifman2019SelectiveNet, Mohseni2020SelfSupervisedLF, jeong2020maml, wei2022logit, Yang2021SemanticallyCO, lee2018OOD-calib, hendrycks2018deep, KatzSamuels2022TrainingOOD}. OOD data can be highly diverse and it is hard to collect or anticipate the kind of OOD data the model will see after deployment. Moreover, if there is some change in OOD data these methods will require expensive retraining of the model. 

\textbf{Controlled false OOD detection rate (FDR).}  PAC-style guarantees on the OOD detection aiming to minimize false detections of OOD points are provided in the supervised settings \citep{liu2018PAC-OOD}. There is an emerging line of  works on OOD detection using Conformal Prediction (CP)
\citep{vovk2005ConformalBook}. Various techniques use non-conformity measures (NCMs) to assess alignment with the training dataset. Examples of previously proposed NCMs include those based on random forests, ridge regression, support vector machines \citep{vovk2005ConformalBook}, k-nearest neighbors \citep{vovk2005ConformalBook, papernot2018deepkn}, and VAE with SVDD \citep{cai2020Realtime-OOD}. Conformal Anomaly Detection (ICAD) \citep{lax2011CAD,laxhammar2015inductive} combines statistical hypothesis testing, NCM scores and inductive conformal anomaly detection (ICAD) \citep{laxhammar2015inductive}for OOD detection. Building upon ICAD, the iDECODe method \citep{kaur2022IDECODE} introduces a novel NCM measure for OOD detection. \citet{bates2023ConformalOutlier} demonstrate issues with conformal p-values under the ICP framework, proposing a technique based on high-probability bounds to compute calibration-conditional conformal p-values.

We note the following key differences: (i) The definition of \emph{inliers (ID) as positives} and \emph{outliers (OOD) as negatives} in our work is the opposite of these works. Thus, controlling FPR in these works translates to guaranteeing a lower bound on TPR in our setting rather than controlling FPR (rate of predicting true OOD as ID).
These works are for offline settings -- fixed sets of inliers and calibration set are used to learn the inference rule which is used to make predictions on a test set such that the fraction of true ID predicted as OOD remains below a desired rate.  Our guarantees on FPR are valid for all time points ($\forall t>0$), while the paper provides guarantees for detection in a given test set of points. 

\textbf{OOD detection with test-time optimization.}  MEMO \citep{zhang2022memo}  uses multi-head models such that the trained model can be adapted with test time distribution shift. ETLT \citep{fan2022simple} trains a separate linear regression model during test time to calibrate the OOD scores as OOD scores. Other works such as \citep{wang2020tent} and \citep{iwasawa2021test} address the issue in a post-hoc manner without altering the trained model. We consider these methods complementary to our work as our framework can adopt the calibrated OOD scores and adapt the threshold safely with FPR control.

Generalized OOD detection 
 is a classical problem that has drawn several promising solutions from researchers across diverse fields including databases, networks, etc. For a comprehensive treatment of the topic, we refer the reader to the book on outlier analysis \citep{Aggarwal2017Book} and recent surveys on ood detection \citep{yang2021generalized, salehi2021unified}.

\textbf{Time uniform confidence sequences} also known as any time valid confidence sequences are confidence intervals designed for
streaming data settings, providing time-uniform and non-asymptotic coverage guarantees  \citep{darling1967CS, lai1976CS}.  Using these sequences gets rid of the requirement of selecting sample size (or stopping time) ahead of time. Due to these nice properties, they have been used in various applications including A/B Testing \citep{Johari2015CanIT, maharaj2023ABTesting}, Multi-armed bandits \citep{jamieson2014lil, szorenyi15bandits}. Iterated-logarithm confidence sequences for best-arm identification using mean estimation are utilized in \citep{jamieson2014lil, kaufmann2016Best-arm, zhao2016Adaptive}. Similar to \citep{darling1968Seq}, \citep{szorenyi15bandits} include a confidence sequence valid over quantiles and time, derived via a union bound applied to the DKW inequality \citep{DKW1956, massart1990DKW}. \citep{howard2022AnytimeValid} improve the upper bounds of \citep{szorenyi15bandits} by replacing the logarithmic
factor with an iterated-logarithm one and improving the constants. Similar to \citep{howard2022AnytimeValid}, in our setting, we require confidence sequences valid uniformly over all times and quantiles. However, their results are valid when samples are drawn i.i.d. and this assumption breaks in our setting due to importance sampling. Interestingly, despite this issue, we obtain a martingale sequence in our setting and build upon the work of \citep{balsubramani2015LIL} to obtain confidence sequences valid uniformly over time and thresholds.

\vspace{-4pt}
\section{Conclusion}
\vspace{-4pt}
We presented a mathematically grounded framework for human-in-the-loop OOD detection. By incorporating expert feedback and utilizing confidence intervals based on the Law of Iterated Logarithm (LIL), our approach maintains control over FPR while maximizing the TPR. The empirical evaluations on synthetic data and image classification tasks demonstrate the effectiveness of our method in maintaining FPR at or below 5\% while achieving high TPR. Our theoretical guarantees are valid for stationary settings. We leave the extension to non-stationary settings as future work.
\section{Acknowledgments}
 We thank Frederic Sala, Changho Shin, Dyah Adila, Jitian Zhao, Tzu-Heng Huang, Sonia Cromp, Albert Ge, Yi Chen,  Kendall Park, and Daiwei Chen for their valuable inputs. We thank the anonymous reviewers for their valuable comments and constructive feedback on our work.

\bibliographystyle{abbrvnat}

\bibliography{references}

\newpage 

\appendix

\section{Appendix}
\label{sec:appendix}

The appendix is organized as follows. We summarize the notation in Table \ref{table:glossary}. Then we give the proof of the main theorem (Theorem \ref{thm:main_theorem}) and the proofs of supporting lemmas. Further, we provide additional experiments and insights from them.

\subsection{Glossary} \label{appendix:glossary}
\label{sec:gloss}
The notation is summarized in Table~\ref{table:glossary} below. 
\begin{table*}[h]
\centering
\begin{tabular}{l l}
\toprule
Symbol & Definition \\
\midrule
$\mathcal{X}$ & feature space. \\
$\mathcal{Y}$ & label space, $\{+1,-1\}$, +1 for ID and -1 for OOD .\\
$\cD_{id},\cD_{ood}$ & distributions of ID and OOD points.\\
$\gamma$ & mixing ratio of OOD and ID distributions.\\
$\lambda$ & threshold for OOD classification.\\
$\FPR(\lambda)$ & population level false positive rate with threshold $\lambda$.\\
$\TPR(\lambda)$ & population level true positive rate with threshold $\lambda$.\\
$\FPRhat(\lambda,t)$ & empirical FPR at time $t$, adjusted to account for importance sampling (see eq. \eqref{eqn:EmpFPR}).\\
$\lambda^*$ & the optimal threshold for OOD classification s.t. $\FPR(\lambda)\le \alpha$ and $\TPR(\lambda)$ is maximized.\\
$\hat{\lambda}_t$ & the estimated threshold at round $t$.\\
$x_t,y_t$ & sample and the true label at time $t$ .\\
$g$&  OOD uncertainty quantification (score) function.\\
$s_u^{(o)}$ & score of $u^{th}$ OOD sample.\\
$i_u^{(o)}$ & indicator variable denoting whether $s_u^{(o)}$ was importance sampled or not.\\
$N_t^{(o)}$ & number of OOD points till time $t$.\\
$N_t^{(o,p)}$ & number of OOD points sampled using importance sampling until time $t$.\\
$\beta_t$ & it is equal to $N_t^{(o,p)}/N_t^{(o)}$.\\
$p$ & probability for importance sampling.\\
$\delta$& failure probability. \\
$\alpha$ & user given upper bound on $\FPR$ that the algorithm needs to maintain. \\
$\eta$ & the algorithm is in $\eta-$optimality if close $\FPR(\lambda^*)-\FPR(\hat{\lambda}_t)\le \eta$. \\
$\Lambda_{\min},\Lambda_{\max}$ & the minimum and maximum scores(thresholds) considered by the algorithm.\\
$\nu$ & discretization parameter for the interval $[\Lambda_{\min},\Lambda_{\max}]$ set by the user.\\
$\psi(t,\delta)$ &  LIL based confidence interval at time $t$.\\
\toprule
\end{tabular}
\caption{
	Glossary of variables and symbols used in this paper.
}
\label{table:glossary}
\end{table*}

\subsection{Proofs}
\label{subsec:proofs}

\textbf{Summary of the setting.} At each time $t$, we observe $x_t \overset{i.i.d}{\sim}(1-\gamma) \cD_{id} + \gamma \cD_{ood}$, and $s_t = g(x_t)$ is the corresponding score. If $s_t \le \hat{\lambda}_{t-1}$, then it is considered an OOD point and hence gets a human label for it and we get to know whether it is in fact OOD or ID. If $s_t > \hat{\lambda}_{t-1}$, then it is considered an ID point and hence gets a human label only with probability $p$. So, we get to know whether it is truly ID or not with probability $p$. Now we have to update the threshold, $\hat{\lambda}_t$, such that the $\FPR(\hat{\lambda}_t) \leq \alpha$ for all $t$, while trying to maximize $\TPR(\hat{\lambda}_t)$. Our approach is based on constructing an unbiased estimator of $\FPR(\lambda)$ using the OOD samples received till time $t$ and in conjunction with confidence intervals for $\FPR(\lambda)$ for at all thresholds $\lambda \in \Lambda$ that is valid for all times simultaneously. Together, at each time $t$, these give us a reliable upper bound on the true $\FPR(\lambda)$ for all $\lambda$ enabling us to find the smallest $\lambda$ such that the upper bound on $\FPR(\lambda)$ is at most $\alpha$. 
Let $S^{(o)}_t = \{s^{(o)}_1,\ldots s^{(o)}_{{N}_t^{(o)}}\}$ denote the scores of the points that have been truly identified as OOD points from human feedback and $I^{(o)}_t$ be the corresponding time points. We estimate the FPR as follows,
\begin{align}
\hspace{-8pt}
\FPRhat(\lambda,t) = \frac{1}{N_t^{(o)}} \sum_{u \in I^{(o)}_t} Z_u(\lambda),\ \text{ where }
 Z_{u}(\lambda) \ldef & \begin{cases} 
     \mathbf{1} (  s^{(o)}_{u} > \lambda), &   \text{ if } s^{(o)}_u \le \hat{\lambda}_{u-1} \\  \frac{1}{p}\mathbf{1}(s^{(o)}_{u}>\lambda), & \text{ w.p. } p \text{ if } s^{(o)}_u > \hat{\lambda}_{u-1} \\
    0, & \text{ w.p. } 1-p \text{ if } s^{(o)}_u > \hat{\lambda}_{u-1}
 \end{cases}.\label{eqn:EmpFPR}
 \end{align}

\textbf{Proof outline.}
To obtain the guarantees in Theorem~\ref{thm:main_theorem} we need confidence intervals $\psi(t,\delta)$ that are simultaneously valid with high probability for the FPR estimates at all time points and all thresholds. There is a rich line of work that provides tight confidence intervals valid for all times based on the Law of Iterated Logarithm (LIL)  ~\citep{khinchine1924LIL, kolmogorov1929LIL, smirnov1944LIL}. Non-asymptotic versions of LIL have been proved in various settings e.g. multi-armed bandits  \citep{jamieson2014lil} and for quantile estimation  \citep{howard2022AnytimeValid}. Roughly speaking, these bounds provide confidence intervals that are $\cO(\sqrt{\log \log (t)/t}$ and are known to be tight. However, most of them assume the samples to be i.i.d. In our setting,  our treatment of observing the human feedback is dependent on whether the score is above or below $\hat{\lambda}_{t-1}$ which itself is estimated using all the past data which creates dependence and 
prevents us from utilizing results developed for i.i.d. samples  ~\citep{howard2022AnytimeValid}. The main technical challenge is to show that the upper confidence bound in eq.~\eqref{eqn:psi-theory} holds for all time and all thresholds with dependent samples used in estimating the FPR in eq.~\eqref{eqn:EmpFPR}. We handle this by first showing that there is a martingale structure in our estimated FPR (eq.~\eqref{eqn:EmpFPR}). We then exploit this structure by using LIL results for martingales  ~\citep{balsubramani2015LIL}. A limitation of \citep{balsubramani2015LIL} is that it can only provide us confidence intervals valid for FPR estimate for a given threshold $\lambda$. However, we need intervals that are simultaneously valid for all $\lambda$ as well. Building upon the work  ~\citep{balsubramani2015LIL} we derive confidence intervals that are simultaneously valid for all $t$ and finitely many thresholds. Equation~\eqref{eqn:psi-theory} shows the $\psi(t,\delta)$ we obtain. 


 Next, we show that the above estimator $\FPRhat(\lambda,t)$ is indeed an unbiased of false positive rate $\FPR(\lambda)$.

\begin{lemma}
      Let $p>0$, $\FPRhat(\lambda,t)$ as defined in eq.~\eqref{eqn:EmpFPR} is an unbiased estimate of the true $\FPR(\lambda)$, i.e.,  $\E[\FPRhat(\lambda,t)] = \FPR(\lambda)$.
\end{lemma}

\begin{proof}
Let $i_t^{(o)}$ be the indicator variable denoting whether $s_t^{(o)}$ was sampled using importance sampling (i.e. $i_t^{(o)}=1$) or not (i.e. $i_t^{(o)}=0$). Denote the pair as  $r_t^{(o)} = (s_t^{(o)},i_t^{(o)})$ for brevity. The proof is by induction. Since the first sample is drawn without any importance sampling so, we have $\E_{r_1^{(o)}}[\FPRhat(\lambda,1)] = \FPR(\lambda)$. Now assume that $\E_{r_{t-1}^{(o)},\ldots, r_1^{(o)}}[\FPRhat(\lambda,t-1)] = \FPR(\lambda)$.

\begin{align*}
\E_{r_t^{(o)},r_{t-1}^{(o)},\ldots,r_{1}^{(o)}}[\FPRhat(\lambda,t)] &= \E \Big[ \frac{1}{N_t^{(o)}}\sum_{u\in I^{(o)}_t}Z_u(\lambda) \Big] \\
&= \E \Big[ \frac{Z_t(\lambda)}{N_t^{(o)}} + \frac{N_t^{(o)} -1}{N_t^{(o)}} \frac{1}{N_t^{(o)}-1}\sum_{u\in I^{(o)}_{t-1}}Z_u(\lambda) \Big] \\
&= \E \Big[ \frac{Z_t(\lambda)}{N_t^{(o)}} + \frac{N_t^{(o)} -1}{N_t^{(o)}} \FPRhat(\lambda, t-1) \Big] \\
&= \frac{1}{N_t^{(o)}} \Big[ \E [Z_t(\lambda)]  + (N_t^{(o)} -1) \cdot \E[\FPRhat(\lambda, t-1)] \Big] \\ 
&= \frac{1}{N_t^{(o)}} \Big [ \E[Z_t(\lambda)]  + (N_t^{(o)} -1) \cdot \FPR(\lambda) \Big] \\
&=\frac{1}{N_t^{(o)}}  \Big [ \E_{r_t^{(o)}|\hat{\lambda}_{t-1}}[Z_t(\lambda)] + (N_t^{(o)} -1) \cdot \FPR(\lambda) \Big ] \\
&= \frac{1}{N_t^{(o)}}  \Big [ \FPR(\lambda) + (N_t^{(o)} -1) \cdot \FPR(\lambda) \Big ] \\ 
&= \FPR(\lambda)
\end{align*}

\end{proof}

Having an unbiased estimator solves one part of the problem. In addition, we need confidence intervals on this estimate that are valid for anytime and for the choices of $\lambda$ considered. Due to the dependence between the samples, we cannot directly apply similar results developed for quantile estimation in the i.i.d. setting \citep{howard2022AnytimeValid}. Fortunately, part of this problem has been addressed in \citep{balsubramani2015LIL}, where they provide anytime valid confidence intervals when the estimators form a martingale sequence. We restate this result in the following lemma \ref{lem:lil-martingale} and then building upon this result, in the next lemma \ref{lem:fpr_conf_interval} we derive such confidence intervals for our setting. 
\begin{lemma}\label{lem:lil-martingale} (\citep{balsubramani2015LIL}) Let $\overline{M}_t$ be a martingale and suppose $|\overline{M}_t - \overline{M}_{t-1}|\le \rho_t$ for constants $\{\rho_t\}_{t>1}$, let $m_0 = \min_{t \ge 1} |\overline{M}_t|$. Fix any $\delta \in (0,1)$, and let $t_0 = \min\{u: \sum_{t=1}^u \rho_t^2 \ge 173\log(\frac{4}{\delta}) \}$ then,

\begin{equation}
    \P \bigg( \exists t\ge t_0 : |\overline{M}_t| \ge \sqrt{3 \Big ( \sum_{i=1}^t\rho_i^2 \Big ) \Big(2 \log \log \big(\frac{3\sum_{i=1}^t \rho_i^2}{m_0} \big) + \log \big(\frac{2}{\delta} \big) \Big) }\bigg) \le \delta
\end{equation}
\end{lemma}
\begin{proof}
This lemma is a restatement of theorem 4 in ~\citep{balsubramani2015LIL}. For proof details please see \citep{balsubramani2015LIL}.
\end{proof}

In the next lemma, we show that the sums of $Z_u(\lambda)$ form a martingale sequence, allowing us to apply the results from the above lemma (\ref{lem:lil-martingale}), and then we generalize it to all $\lambda$ in some finite set.
\begin{lemma}
\label{lem:fpr_conf_interval}
(Anytime valid confidence intervals on FPR) Let $X^{(o)}_t = \{x^{(o)}_1,\ldots x^{(o)}_{{N_t}^{(o)}}\}$ be the samples drawn from $D_{ood}$ till round $t$ and let $S^{(o)}_t = \{s^{(o)}_1,\ldots s^{(o)}_{{N_t}^{(o)}}\}$ be the scores of these points, let $c_t = 1-\beta_t + \frac{\beta_t}{p^2}$, $\beta_t = \frac{N_t^{(o,p)}}{N_t^{(o)}}$ and $N_t^{(o,p)}$ is the number of points sampled using importance sampling until time $t$ and $\nu \in (0, 1)$ is a discretization parameter set by the user. Let $\Lambda = \{\Lambda_{\min},\Lambda_{\min}+\nu,\ldots,\Lambda_{\max}\}$. Let $n_0 = \min\{u: c_u N^{(o)}_u \ge 173\log(\frac{4}{\delta}) \}$ and $t_0$ be such that $N_{t_{0}}^{(o)}\ge n_0$.
, then for any $\delta \in (0,1)$, 
 
\begin{equation}
    \P \bigg( \exists t\ge t_0 : \sup_{\lambda \in \Lambda} \FPRhat(\lambda,t) - \FPR(\lambda) \ge \psi(t,\delta) \bigg) \le \delta
\end{equation}
for,
\begin{equation}
  \psi(t,\delta) = \sqrt{\frac{ 3 c_t}{N_t^{(o)}} \bigg [ 2\log \log \Big ( \frac{3 c_t N_t^{(o)}}{2}\Big )  + \log \bigg ( \frac{2 |\Lambda|}{\delta} \bigg )\bigg ]} 
\end{equation}
\end{lemma}

\begin{proof}
First, we show that we have a martingale sequence as follows,

Let $M_t(\lambda) =  \sum_{u=1}^{N_t^{(o)}} Z_u(\lambda)$, and consider the centered random variables,
$$\overline{M}_t(\lambda) = M_t(\lambda) - \E[M_t(\lambda)] \quad \text{and}\quad \overline{Z}_t(\lambda) = Z_t(\lambda) - \FPR(\lambda)$$

Let $\cF_{t}$ be the $\sigma-$algebra of events till time $t$ i.e. $(s_1^{(o)},i_1^{(o)}),\ldots,(s_{t-1}^{(o)},i_{t-1}^{(o)}), (s_t^{(o)},i_t^{(o)})$. 

It is easy to see that $\E[\overline{M}_t]\le \frac{1}{p} < \infty$ and $\overline{M}_t$ is $\cF_{t}$-measurable for all $t>1$.
Further, we can see, 
\[ \E \big[ {\overline{M}}_t(\lambda) | \cF_{t-1} \big] = \E[\overline{Z}_{t}(\lambda) + \overline{M}_{t-1}(\lambda)|\cF_{t-1}] = \E[\overline{Z}_{t}(\lambda)|\cF_{t-1}]  +\E[ \overline{M}_{t-1}(\lambda) |\cF_{t-1}] = \overline{M}_{t-1}(\lambda) \]
Since, $\E[\overline{Z}_{t}(\lambda)|\cF_{t-1}] = 0$ and $\E[ \overline{M}_{t-1}(\lambda) |\cF_{t-1}] = \overline{M}_{t-1}(\lambda) $. Thus we have that  $\overline{M}_{t}$ is a martingale sequence. Further, we also have the following,

\[ |\overline{M}_t(\lambda) - \overline{M}_{t-1}(\lambda)| \le \begin{cases} 1 \quad \text{if } i_t^{(o)}=0 \\
\frac{1}{p}  \quad \text{if } i_t^{(o)}=1 
\end{cases}  \]

Let $\beta_t \in (0,1)$ be the fraction of OOD points sampled using probability $p$ till round $t$. Let $N_t^{(o)}$ be the total number of points OOD points sampled till round $t$ and $N_t^{(o,p)}$ be the points sampled from importance sampling, then $\beta_t = \frac{N_t^{(o,p)}}{N_t^{(o)}}$.

Let $c_t = 1-\beta_t + \frac{\beta_t}{p^2}$. We know $p$ and the number of points sampled with importance sampling, without importance sampling we know$\beta_t, c_t$ are at time $t$. Applying lemma \ref{lem:lil-martingale} we get the following result for a given $\lambda$,
\begin{equation}
    \P \bigg( \exists t\ge t_0 : \overline{M}_t(\lambda) \ge \sqrt{3 \Big ( c_t N_t^{(o)} \Big ) \Big(2 \log \log \big(3c_t N_t^{(o)} \big) + \log \big(\frac{2}{\delta} \big) \Big) }\bigg) \le \delta
\end{equation}

\begin{equation}
    \P \bigg( \exists t\ge t_0 : \FPRhat(\lambda,t) -\FPR(\lambda,t) \ge \sqrt{\frac{3  c_t  }{N_t^{(o)}}\Big(2 \log \log \big(3c_t N_t^{(o)} \big) + \log \big(\frac{2}{\delta} \big) \Big) }\bigg) \le \delta
\end{equation}
Doing the union bound for the failure probability over all $\lambda \in \Lambda$, (where $|\Lambda|<\infty$) gives us the result.

\end{proof}
Our performance guarantees in the main theorem \ref{thm:main_theorem} are based on $\psi(t,\delta)$ becoming smaller than certain values. In the next lemma we derive bound on $N_t^{(o)}$ such that $\psi(t,\delta)$ is at most $\mu$ and use it in the proof of the main theorem \ref{thm:main_theorem}.

\begin{lemma}
\label{lem:psi_invert}
  Let $\psi(t,\delta) = \sqrt{\frac{3 c_t}{N_t^{(o)}} \bigg ( 2\log \log \Big ( 3 c_t N_t^{(o)}\Big )  + \log \Big (  \frac{2|\Lambda|}{\delta}\Big ) \bigg )}, $ and let there be a constant $C_0$ and time $T_0$, such that $\beta_t \le C_0p^2$ for all $t\ge T_0$ (worst case $T_0=1$ and $C_0=1/p^2$). Then $\psi(t,\delta)\le \mu$ for any $t>T_{\mu}>T_0$ such that $N_{T_\mu}^{(o)} = \frac{10(C_0 +1)}{\mu^2}\log \Big( \frac{|\Lambda|}{\delta}\log(\frac{5(C_0 +1)}{\mu}) \Big) .$
\end{lemma}

\begin{proof}
First we write a simplified form of $\psi(t,\delta)$ for all $t>T_0$ as follows,
\[ \psi(t,\delta) = \sqrt{\frac{3 (C_0+1)}{N_t^{(o)}} \bigg ( 2\log \log \Big ( 3 (C_0+1) N_t^{(o)}\Big )  + \log \Big (  \frac{2|\Lambda|}{\delta}\Big ) \bigg )}\]
In the above equation we used the bound on $\beta_t \le C_0 p^2$ in the equation $c_t = 1- \beta_t + \beta_t/p^2$ leading to $c_t \le C_0 +1$,
Now, for brevity let $a_1 = 3(C_0 +1)$ and $a_2 = 2|\Lambda|$ and rewrite $\psi(t,\delta)$ as follows,
\[ \psi^2(t,\delta) = \frac{ a_1}{N_t^{(o)}} \bigg ( 2\log \log \Big ( a_1 N_t^{(o)}\Big )  + \log \Big (  \frac{a_2}{\delta}\Big ) \bigg )
 \le  \frac{ 2a_1}{N_t^{(o)}} \bigg ( \log \Big ( \frac{a_2}{\delta} \log \Big ( a_1 N_t^{(o)}\Big ) \Big ) \bigg ) \]
We want to find $N_t^{(o)}$ such that $\psi^2(t,\delta) \le \mu^2$. It is difficult to directly invert this function. To get a bound on $N_t^{(o)}$ we first assume the following form for it with unknown constants $b_1,b_2,b_3>0$ and then figure out the constants by simplifying $\psi^2(N_t^{(o)})$ and constraining it to be at most $\mu^2$. 
\[ \text{ Let } N^{(o)}_{T_{\mu}} = \frac{b_1}{\mu^2} \log \bigg ( \frac{a_2}{b_3\delta} \log \big ( \frac{b_2}{\mu} \big )\bigg )\]

\begin{align*}
   \psi^2(T_\mu,\delta)  &\le \frac{2a_1}{N^{(o)}_{T_{\mu}}} \log \Big [\frac{a_2}{\delta} \log (a_1N^{(o)}_{T_{\mu}}) \Big ] \\
     &= \frac{2 a_1\mu^2}{ b_1 \log \Big ( \frac{a_2}{b_3\delta} \log \big ( \frac{b_2}{\mu} \big )\Big )} \log \bigg [ \frac{a_2}{\delta}  \log \bigg \{  \frac{a_1 b_1}{\mu^2}  \log \bigg ( \frac{a_2}{b_3\delta}  \log \Big ( \frac{b_2}{\mu} \Big ) \bigg ) \bigg \} \bigg ] \\
   &\overset{(i)}{\le} \frac{2 a_1\mu^2}{ b_1 \log \Big ( \frac{a_2}{b_3\delta} \log \big ( \frac{b_2}{\mu} \big )\Big )}  \log \bigg [ \frac{a_2}{\delta} \log \bigg \{  \frac{a_1 b_1}{\mu^2}  \log \bigg ( \frac{a_2}{b_3\delta}  \frac{b_2}{\mu}  \bigg ) \bigg \} \bigg ] \\
   &= \frac{2 a_1\mu^2}{ b_1 \log \Big ( \frac{a_2}{b_3\delta} \log \big ( \frac{b_2}{\mu} \big )\Big )}  \log \bigg [ \frac{a_2}{\delta}  \log \bigg \{  \frac{a_1 b_1}{\mu^2}  \log \bigg ( \frac{a_2 b_2}{b_3\delta \mu} \bigg ) \bigg \} \bigg ] \\
   & \overset{(ii)}{\le} \frac{2 a_1\mu^2}{ b_1 \log \Big ( \frac{a_2}{b_3\delta} \log \big ( \frac{b_2}{\mu} \big )\Big )}  \log \bigg [ \frac{a_2}{\delta}  \log \bigg \{  \frac{a_1 b_1}{\mu^2}   \frac{a_2 b_2}{b_3\delta \mu}  \bigg \} \bigg ] \\
   &=\frac{2 a_1\mu^2}{ b_1 \log \Big ( \frac{a_2}{b_3\delta} \log \big ( \frac{b_2}{\mu} \big )\Big )}  \log \bigg [ \frac{a_2}{\delta} \log \bigg \{  \frac{a_1 b_1}{\mu^2}    \frac{a_2 b_2}{b_3\delta \mu}  \bigg \} \bigg ] \\
   &=\frac{2 a_1\mu^2}{ b_1 \log \Big ( \frac{a_2}{b_3\delta} \log \big ( \frac{b_2}{\mu} \big )\Big )} \log \bigg [ \frac{a_2}{\delta}  \log \bigg \{  \frac{a_1 b_1 a_2 }{b_3 b_2^2 \delta}    \Big ( \frac{b_2}{\mu} \Big )^3  \bigg \} \bigg ] \\
   &\overset{(iii)}{\le}\frac{2 a_1\mu^2}{ b_1 \log \Big ( \frac{a_2}{b_3\delta} \log \big ( \frac{b_2}{\mu} \big )\Big )} \log \bigg [ \frac{a_2}{\delta}  \frac{a_1 b_1 a_2 }{b_3  b_2^2 \delta}  \log \bigg \{   \Big ( \frac{b_2}{\mu} \Big )^3 \bigg \}    \bigg ] \\
   &= \frac{2 a_1\mu^2}{ b_1 \log \Big ( \frac{a_2}{b_3\delta} \log \big ( \frac{b_2}{\mu} \big )\Big )}  \log \bigg [ \frac{a_2}{\delta}  \frac{ 3 a_1 b_1 a_2 }{b_3  b_2^2 \delta} \log    \Big ( \frac{b_2}{\mu} \Big )   \bigg ] \\
   &=\frac{2 a_1\mu^2}{ b_1 \log \Big ( \frac{a_2}{b_3\delta} \log \big ( \frac{b_2}{\mu} \big )\Big )}  {\log \bigg [  \Big ( \frac{a_2}{b_3\delta} \Big)^2} { \frac{3 a_1 b_1 b_3 }{b_2^2} \log    \Big ( \frac{b_2}{\mu} \Big ) } {  \bigg ] }\\
   &{\overset{(iv)}{\le}} \frac{2 a_1\mu^2 {\frac{3 a_1 b_1 b_3 }{b_2^2}}}{ b_1 \log \Big ( \frac{a_2}{b_3\delta} \log \big ( \frac{b_2}{\mu} \big )\Big )}  {\log \bigg [  \Big ( \frac{a_2}{b_3\delta} \Big)^2 } { \log    \Big ( \frac{b_2}{\mu} \Big ) } {  \bigg ] }\\
   & {\overset{(v)}{\le}} \frac{ 2 a_1 \mu^2 \cdot {2}  {\frac{3 a_1 b_1 b_3 }{b_2^2}}}{ b_1 \log \Big ( \frac{a_2}{b_3\delta} \log \big ( \frac{b_2}{\mu} \big )\Big )}  {\log \bigg [   \frac{a_2}{b_3\delta} } {  \log    \Big ( \frac{b_2}{\mu} \Big ) } { \bigg ]} \\
   &=\frac{12\mu^2 a_1^2 b_3}{b_2^2}  .
\end{align*}

The inequalities ${(i)},{(ii)}$ follow from $\log(x)\le x $ for any $x>0$. 

The inequality ${(iii)}$ comes from $\log(ax) \le a \log(x)$ for any $a>2,x>2$. We use $a=\frac{a_1b_1 a_2}{b_3 b_2^2 \delta}$ and $x=\Big (\frac{b_2}{\mu} \Big )^3$,
this enforces the following constraints,
 \begin{equation}
 \label{eq:T_up_c1}
     \frac{b_2}{\mu} > 2^{1/3}
 \end{equation}
 \begin{equation}
 \label{eq:T_up_c2}
     \frac{a_1b_1 a_2}{b_3 b_2^2 \delta} > 2
 \end{equation}
 For ${(iv)}$ we again use $\log(ax) \le a \log(x)$ with $a=\frac{3a_1b_1b_3}{b_2^2}$ and $x=  \big ( \frac{a_2}{b_3\delta} \big)^2  \log    \big ( \frac{b_2}{\mu} \big ) $, this enforces the following constraints,
  \begin{equation}
  \label{eq:T_up_c3}
     \frac{3a_1b_1b_3}{b_2^2} > 2
 \end{equation}
 \begin{equation}
 \label{eq:T_up_c4}
    \big ( \frac{a_2}{b_3\delta} \big)^2  \log    \big ( \frac{b_2}{\mu} \big ) > 2
 \end{equation}

 Lastly, ${(v)}$ follows by using $\log(x^a y) \le a \log(xy)$ for any $x>0,a>1,y>1$. For this we use $x=\frac{a_2}{b_3\delta}$ and $y=\log(\frac{b_2}{\mu})$, leading the following constraints,

 \begin{equation}
 \label{eq:T_up_c5}
   \log(\frac{b_2}{\mu}) > 1
 \end{equation}

For $\psi^2(T_{\mu}) \le \mu^2$, we need 
\begin{equation}
\label{eq:T_up_c6}
  12 a_1^2b_3 \le b_2^2
\end{equation}

Let $b_3 =2 , b_1= 10 a_1, b_2 = 5 a_1$ then the constraints \ref{eq:T_up_c1},\ref{eq:T_up_c2},\ref{eq:T_up_c3},\ref{eq:T_up_c4},\ref{eq:T_up_c5} and \ref{eq:T_up_c6} are satisfied ( when $|\Lambda|\ge 10$ ) for any $\mu \in (0,1), \delta \in (0,1)$. Thus we have,

\begin{equation}
  \psi(T_\mu,\delta) \le \mu \, \text{ for } N_{T_\mu} = \frac{10(C_0 +1)}{\mu^2}\log \Big( \frac{|\Lambda|}{\delta}\log(\frac{5(C_0 +1)}{\mu}) \Big)
\end{equation}

\end{proof}

\begin{lemma}
\label{lem:num_tosses}
 Let the  data points $\{x_t\}_{t\ge 1}$ be independent draws from the mixture distribution $(1-\gamma)\cD_{id} + \gamma \cD_{ood}$,  and $N_t^{(o)}$ be the number of OOD points received till time $t$ from this distribution, then for any $\delta \in (0,1)$ for any $t \ge T_k$ we have $N_t^{(o)}\ge k$ w.p. $1-\delta$, where $T_{k}$ is given as follows, 
 \begin{equation}
     T_k = \frac{2k}{\gamma} + \frac{1}{\gamma^2}\log(\frac{1}{\delta}).
 \end{equation}
\end{lemma}

\begin{proof} 
We want to find $t$ such that $N_t^{(o)} \ge k $
w.p. $1-\delta$. This is the same as finding the number of coin tosses of a coin with bias $\gamma$ so that the number of heads observed is at least $k$. Applying Hoeffding's inequality gives us the following w.p. $1-\delta$,
\[N_t^{(o)} \ge \gamma t - \sqrt{\frac{t}{2}\log(\frac{1}{\delta})} .\]
Equating the r.h.s. above with $k$ and solving for $t$ will give us the desired bound on $t$. Note that it is enough to have an upper bound on $t$ that satisfies the following and then use that upper bound as $T_k$.
\[ \gamma t - \sqrt{\frac{t}{2}\log(\frac{1}{\delta})} = k.\]

To simplify the calculations, let $c = \sqrt{\frac{1}{2}\log\frac{1}{\delta}}$ and let $t=u^2$ then we have the following quadratic equation,
\[ \gamma u^2 -c u -k = 0.\]
Considering the larger of the two solutions,
\[ u = \frac{c + \sqrt{c^2 + 4k\gamma}}{2\gamma}.\]
Using the fact that for any $a,b \ge 0$, $\sqrt{a+b} \le \sqrt{a} + \sqrt{b}$,
\[ u \le \frac{c+ \sqrt{c^2} +  \sqrt{4k\gamma}}{2\gamma} = \frac{2c + 2\sqrt{k\gamma}}{2\gamma} =\frac{c}{\gamma} + \sqrt{\frac{k}{\gamma}}.\]
Lastly, using $(a+b)^2 \le 2a^2 + 2b^2$ for any $a,b \in \R$ we get the following upper bound on $t$,
\[ t = u^2 \le \frac{2c^2}{\gamma^2} + \frac{2k}{\gamma} = \frac{2k}{\gamma} + \frac{1}{\gamma^2}\log(\frac{1}{\delta}).\]


\end{proof}

\begin{theorem}
Let $\alpha, \delta, p, \gamma \in (0, 1)$. Let $x_t \overset{i.i.d}{\sim}(1-\gamma) \cD_{id} + \gamma \cD_{ood}$ and let $c_t = 1-\beta_t + \frac{\beta_t}{p^2}$, $\beta_t = \frac{N_t^{(o,p)}}{N_t^{(o)}}$ where $N_t^{(o,p)}$ is the number of OOD points sampled using importance sampling until time $t$ and $N_t^{(o)}$ is the total number of OOD points observed till time $t$. Let $n_0 = \min\{u: c_u N^{(o)}_u \ge 173\log(\frac{8}{\delta}) \}$ and $t_0$ be such that $N_{t_{0}}^{(o)}\ge n_0$. If Algorithm~\ref{alg:adaptive-ood} uses the optimization problem~\eqref{opt-P2} to find the thresholds with the upper confidence term $\psi(N_t^{(o)},\delta/2)$ given by eq.~\eqref{eqn:psi-theory}, then there exist constants $C_1, C_2, C_3 >0$ such that with probability at least $1-\delta$, 
\begin{enumerate}[leftmargin=15pt,itemsep=0.25pt,topsep=0pt]
    \item \textbf{Controlled FPR.} For all $t \geq t_0$, $\FPR(\hat{\lambda}_t) \leq \alpha$.
    \item \textbf{Time to reach feasibility.} The algorithm  will find a feasible threshold, $\hat{\lambda}_t$ such that $\FPRhat(\hat{\lambda}_t) + \psi(N_t^{(o)}) \le \alpha$, for all $t \ge \max(t_0,T_f)$ , where, $ T_f = \frac{2C_1}{\gamma \alpha^2} \log \Big( \frac{4C_2}{\delta}\log(\frac{C_3}{\alpha}) \Big) + \frac{1}{\gamma^2}\log(\frac{4}{\delta}) $.
    \item \textbf{Time to reach $\eta-$optimality.}  For all $t \ge \max(t_0,T_{\eta\text{-opt}})$, $\hat{\lambda}_t$ satisfy the $\eta$-optimality condition in definition \ref{defn:optimality}, when
    $\FPRhat(\hat{\lambda}_{T_{\eta\text{-opt}}}) \in [\alpha -\eta/2, \alpha]$ and $T_{\eta\text{-opt}} = \frac{8C_1}{\gamma \eta^2} \log \Big( \frac{4C_2}{\delta}\log(\frac{2C_3}{\eta}) \Big) + \frac{1}{\gamma^2}\log(\frac{4}{\delta})$.
\end{enumerate}
\label{thm:main_theorem}
\end{theorem}

\begin{proof}

To prove this we first obtain confidence intervals on FPR valid w.p. $1-\delta/2$ using Lemma \ref{lem:fpr_conf_interval}. Then applying Lemma \ref{lem:psi_invert} on these confidence intervals gives us the number of OOD samples that are sufficient to guarantee a certain width of the confidence intervals and lastly we use Lemma \ref{lem:num_tosses} to bound the time point such that we observe a certain number of OOD points till that time. We do this for the second and third points separately each time invoking Lemma \ref{lem:num_tosses} with failure probability $\delta/4$ and then union bound over them.

\textit{Controlled FPR.}  This follows from Lemma \ref{lem:fpr_conf_interval} (with probability $1-\frac{\delta}{2}$) and the fact the algorithm uses confidence intervals on FPR estimate that are valid for all $t \ge t_0$ for the choices of $\lambda$ it considers. 
    
\textit{Time to reach feasibility.} Applying Lemma \ref{lem:psi_invert}
with $\mu=\alpha$ gives bound on $N_{T_f}$ with $C_1 = 10(C_0 +1), C_2 = |\Lambda|, C_3 = 5(C_0+1)$. Then using Lemma \ref{lem:num_tosses} with $k=N_{T_f}$ gives us the desired $T_f$.

\textit{Time to reach $\eta$-optimality.} We know,
$\FPR(\lambda^*) = \alpha$, and it is given that $\FPRhat(\hat{\lambda}_t) \in [\FPR(\lambda^*) -\eta/2, \alpha ]$
\[ \FPR(\hat{\lambda}_t) \in [\FPR(\lambda^*) - \eta/2 -\psi(t,\delta) , \alpha ]\]
 this means $\FPR(\hat{\lambda}_t) \ge \FPR(\lambda^*) - \eta/2 -\psi(t,\delta)$
\[ \FPR(\lambda^*) - \FPR(\hat{\lambda}_t) \le  \eta/2 +\psi(t,\delta)\]

If $\psi(t,\delta) \le \eta/2$ we have, $ \FPR(\lambda^*) - \FPR(\hat{\lambda}_t)\le \eta$. Thus applying we want to find  $t$ for which  $\psi(t,\delta) = \eta/2$. Applying lemma \ref{lem:psi_invert}
with $\mu=\eta/2$ gives bound on $N_{T_f}$ with $C_1 = 40(C_0 +1), C_2 = |\Lambda|, C_3 = 10(C_0+1)$.  Then using Lemma \ref{lem:num_tosses} with $k=N_{T_{opt}}$ gives us the desired $T_{opt}$.
\end{proof}

This concludes the proofs of the main results. Next, we present details of the procedure used to solve the optimization problem \ref{opt-P2} and additional experiments on synthetic and real datasets. 


\subsection{Additional Details of the Algorithm}
We use the following algorithm (Algorithm \ref{alg:solve-lambda-binary}) based on binary search to solve the optimization problem \ref{opt-P2} i.e., find $\hat{\lambda}_t$. In addition to the best estimate of current threshold $\hat{\lambda}_t$ it also returns a flag $feasible$ that indicates whether the procedure found a threshold satisfying the constraint in \ref{opt-P2} or not.

\begin{algorithm}[H]
\begin{algorithmic}[1]
\caption{ SolveOptForLambda} \label{alg:solve-lambda-binary}
\Require{FPR threshold $\alpha$ , $S_t$}
        

\State{$low=1,high= \frac{\Lambda_{\max} - \Lambda_{\min}}{\nu}, feasible=False$}
\While{ $low< high$}
    \State{$mid = \lceil (low+high)/2\rceil$}
    \State{$\lambda_{\text{mid}} = \Lambda_{\min} + k \nu$}
    \If{ $\FPRhat(\lambda_{\text{mid}},t) + \psi(t,\delta) \le \alpha$ }
        \State{$feasible = True$}
        \State{$high=mid$}
    \Else 
        \State{ $low = mid$}
    \EndIf

\EndWhile
\State{Output $feasible, \lambda_{\text{mid}}$}

\end{algorithmic}
\end{algorithm}

\subsection{Additional Experiments and Details}
\label{sec:additional_exp}
In the simulations we study the effect of changing $\gamma$, using different window sizes and the case when the In-Distribution shifts. For the real data experiments, we study the performance of the methods under different settings with different scoring functions on CIFAR-10 and CIFAR-100 as In-Distribution datasets.

\subsubsection{Searching for Constants in LIL-Heuristic}
\label{subsubsec:constant-search}
The theoretical LIL bound in eq.~\eqref{eqn:psi-theory} has constants that can be pessimistic in practice. We get around this by using a LIL-Heuristic bound which has the same form as in eq.~\eqref{eqn:psi-theory} but with different constants in particular we consider the form in eq.~\eqref{eqn:lil-heuristic}. We find the constants $C_1, C_2, C_3$ using a simulation on estimating the bias of a coin with different constants and picking the ones so that the observed failure probability is below 5\%. 
\begin{equation}
 \tilde{\psi}(t,\delta) =  C_1\sqrt{\frac{  c_t}{N_t^{(o)}} \bigg ( \log \log \big (C_2 c_t N_t^{(o)} \big )  + \log \Big ( \frac{C_3}{\delta} \Big )\bigg )} . 
  \tag{LIL-Heuristic}
  \label{eqn:lil-heuristic}
\end{equation}

Specifically, we keep $C_3=1$, and run for $\delta \in \{ 0.01, 0.05, 0.1, 0.2, 0.3, 0.4 \}$ with varying $C_1$ from $0.1$ to $0.9$ and $C_2$ from $1.5$ to $4.75$. For each choice of  $\delta, C_1, C_2$, we toss an unbiased coin (mean $p=0.5$) for $T = 10k$ times. For each choice of $t=1,2,\cdots, T$, we compute the empirical mean $\hat{p}$ of the coin and define it as a failure if $p \notin [\hat{p} - \tilde{\psi}(t,\delta), \hat{p}+ \tilde{\psi}(t,\delta)]$. We run this process for 100 times and compute the average failure probability for each choice of $t=1,2,\cdots, T$. We then pick the constant so that the observed average probability is below 5\%. Throughout the paper, we use $C_1 = 0.5$ and $C_2 = 0.75$.



\subsubsection{Additional Experiments on Synthetic Data}

\textbf{In-Distribution shift.} We study the scenario when the ID distribution changes and the OOD distribution remains fixed. In this setting, the FPR for any threshold does not change since the OOD does not change but the TPR changes due to the change in ID distribution. Thus, we expect that with this type of change our method will not violate the FPR constraint and it will gradually adapt to the threshold achieving the new TPR. 

We simulate the OOD and ID scores using a mixture of two Gaussians $\mathcal{N}_{id}(\mu=5.5,\sigma=4)$ and  $\mathcal{N}_{ood}(\mu=-5,\sigma=4)$  with $\gamma = 0.2$. To simulate distribution change we change the ID distribution to $\mathcal{N}_{id}(\mu=5,\sigma=4)$ at time $t=50k$. We ran the methods 10 times with different random seeds. The results are shown in Figure \ref{fig:sim-id-shift}. We can clearly see that changing ID distribution(ID scores getting closer to the OOD scores) leads to a decrease in the TPR at the threshold with 5\% FPR. Since the estimation of threshold only depends on the FPR estimates and hence only on OOD samples, changing ID distribution does not affect this estimation so the methods perform the same as in the setting of no-distribution shift but get a reduction in the TPR at FPR $5\%$.

\begin{figure*}[t]
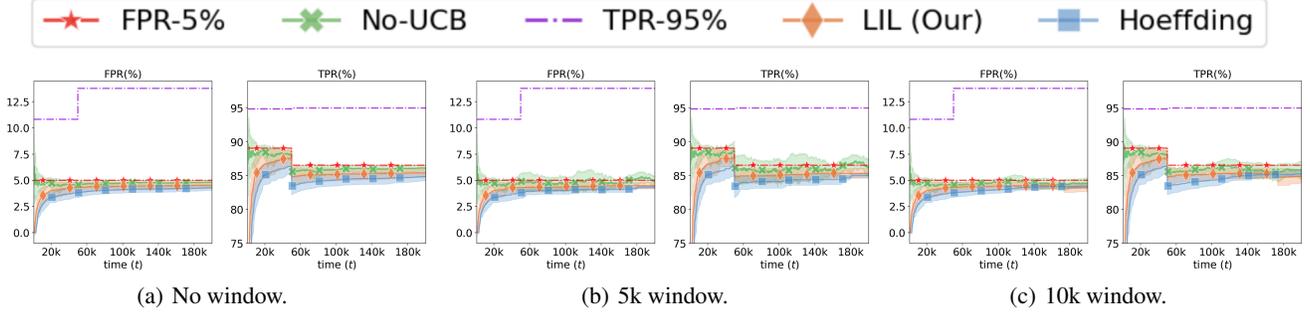

  \vspace{-10pt}
  \centering
  \includegraphics[width=0.999\textwidth]{figs-v2/legend_hfdng-3.png}
  \mbox{
  \hspace{-15pt}
    \subfigure[No window.  \label{fig:sim_id_shift_win_no}]{\includegraphics[scale=0.22]{figs-v2/sim-id-change/sim_id_change_plot_window_None_gamma_0.2_hfding.png}}
    \hspace{1pt}
    
    \subfigure[5k window. \label{fig:sim_id_shift_win_5k}]{\includegraphics[scale=0.22]{figs-v2/sim-id-change/sim_id_change_plot_window_5000_gamma_0.2_hfding.png}}

    \subfigure[10k window. \label{fig:sim_id_shift_win_10k}]{\includegraphics[scale=0.22]{figs-v2/sim-id-change/sim_id_change_plot_window_10000_gamma_0.2_hfding.png}}
    
  }
  \caption{{
    Changing ID distribution in synthetic data.} 
    }
  \label{fig:sim-id-shift}
  \vspace{-10pt}
\end{figure*}

\subsubsection{Additional Real OOD Datasets Experiments}
We run our proposed system on real ID and OOD datasets with various OOD scoring methods.  We use $\alpha = 0.05$, $\delta=0.2$, and importance sampling probability $p=0.2$ through all the experiments. We used an Nvidia RTX 6000 graphics processing unit (GPU) to facilitate the inference process and retrieve confidence scores for various datasets. The primary experiments conducted were executed without the GPU.

\textbf{ID and OOD datasets}. We use CIFAR-10 or CIFAR-100 as ID datasets. In the distribution shift setting, if not specified, we use MNIST, SVHN, and Texture as the first mixture of OOD datasets, and TinyImageNet, Places365, and CIFAR-10/100 as the second mixture of OOD datasets by default. We use a pre-trained Resnet-50 model for SSD method,  and Resnet-18 for the rest of the methods.

\textbf{Scoring functions.} We use the following scoring functions,
 
\begin{enumerate}[leftmargin=*, topsep=-2pt] 
    \item \textbf{ODIN.} ODIN \citep{liang2017enhancing} takes the soft-max score from DNNs and scales the score with temperature. A gradient-based input perturbation is also used for better performance. We choose temperature $1000$ and input perturbation noise $0.0014$, as discussed in \citep{liang2017enhancing}. The results with this scoring function on CIFAR-10 and CIFAR-100 ID data settings are shown in Figures \ref{fig:odin-cifar10} and \ref{fig:odin-cifar100} respectively.
    
    \item \textbf{Mahalanobis Distance (MDS).}  
    For a given point $x$, the Mahalanobis Distance (MDS) based score is its MD from the closest class conditional distribution. We use the MD-based score as given in \citep{lee2018simple} for detecting OOD and adversarial samples. They compute the scores using representations from various layers of DNNs and combine them to get a better scoring function. We choose input perturbation noise to be $0.0014$. The results with this scoring function on CIFAR-10 and CIFAR-100 ID data settings are shown in Figures \ref{fig:mds-cifar10} and \ref{fig:mds-cifar100} respectively.
    \item \textbf{Energy Score (EBO).} 
    This score was proposed in \citep{liu2020energy} and it is well aligned with the probability density of the samples, with low energy implying ID and high energy implying OOD. We choose the temperature parameter to be $1$. The results with the EBO scoring function on CIFAR-10 and CIFAR-100 ID data settings are shown in Figures \ref{fig:ebo-cifar10} and \ref{fig:ebo-cifar100} respectively.

    \item \textbf{SSD.} It is based on computing the Mahalanobis distance in the feature space of the model trained on the unlabeled in-distribution data using self-supervised learning. We use the official implementation of \citep{sehwag2021ssd}. For CIFAR-10, we use the pre-train model they released. For CIFAR-100, We train a Resnet-50 using a contrastive self-supervised learning loss, SimCLR \citep{chen2020simple}. When calculating the distance-based OOD scores, we use one unsupervised clustering center as the only center for ID distribution for both CIFAR-10 and CIFAR-100.  The results with this scoring function on CIFAR-10 and CIFAR-100 ID data setting are shown in Figures \ref{fig:ssd-cifar10} and \ref{fig:ssd-cifar100} respectively.
    \item  \textbf{Virtual-logit Match}.  Virtual-logit Match (VIM) \citep{wang2022vim} combines the class-agnostic score from feature space and ID class-dependent logits. Specifically, an additional logit representing the virtual OOD class is generated from the residual of the feature against the principal space and then matched with the original logits by a constant scaling. We set the dimension of the principal space $D=100$. The results with VIM scoring function on CIFAR-10 and CIFAR-100 ID data settings are shown in Figures \ref{fig:vim-cifar10} and \ref{fig:vim-cifar100} respectively.
    
    \item \textbf{K-Nearest-Neighborhood.} Unlike other methods that impose a strong distributional assumption of the underlying feature space, the KNN-based method \citep{sun2022out} explores the efficacy of non-parametric nearest-neighbor distance for OOD detection. The distance between the test sample and its k-nearest training IN sample will be used as the score for threshold-based OOD detection. We choose neighbor number $k = 50$. The results with this scoring function on CIFAR-10 and CIFAR-100 ID data settings are shown in Figures \ref{fig:knn-cifar10-supp} and \ref{fig:knn-cifar100} respectively.
\end{enumerate}




\begin{figure*}[h]
  
  \centering
  \includegraphics[width=0.999\textwidth]{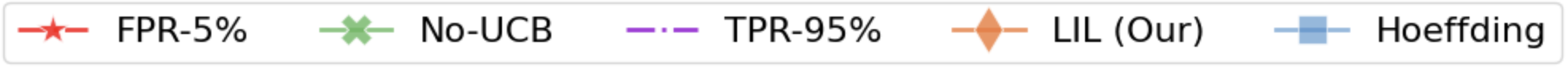}
  \mbox{
  \hspace{-10pt}
    \subfigure[No distribution shift, no window.   \label{fig:knn_no_shift_no_win_cifar10}]{\includegraphics[scale=0.215]{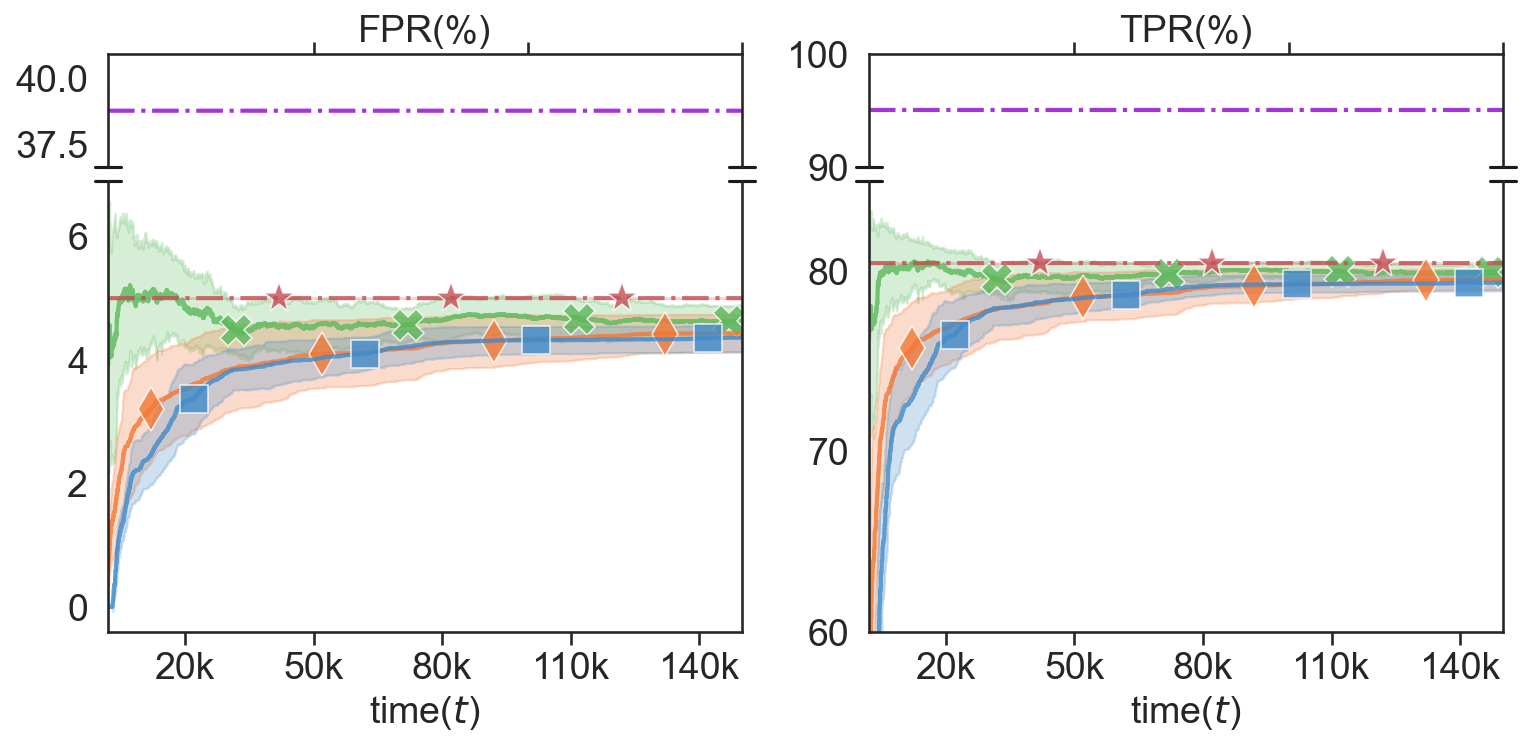}}
    \hspace{1pt}
    
    \subfigure[Distribution shift, 5k window. \label{fig:knn_shift_win_5k_cifar10}]{\includegraphics[scale=0.215]{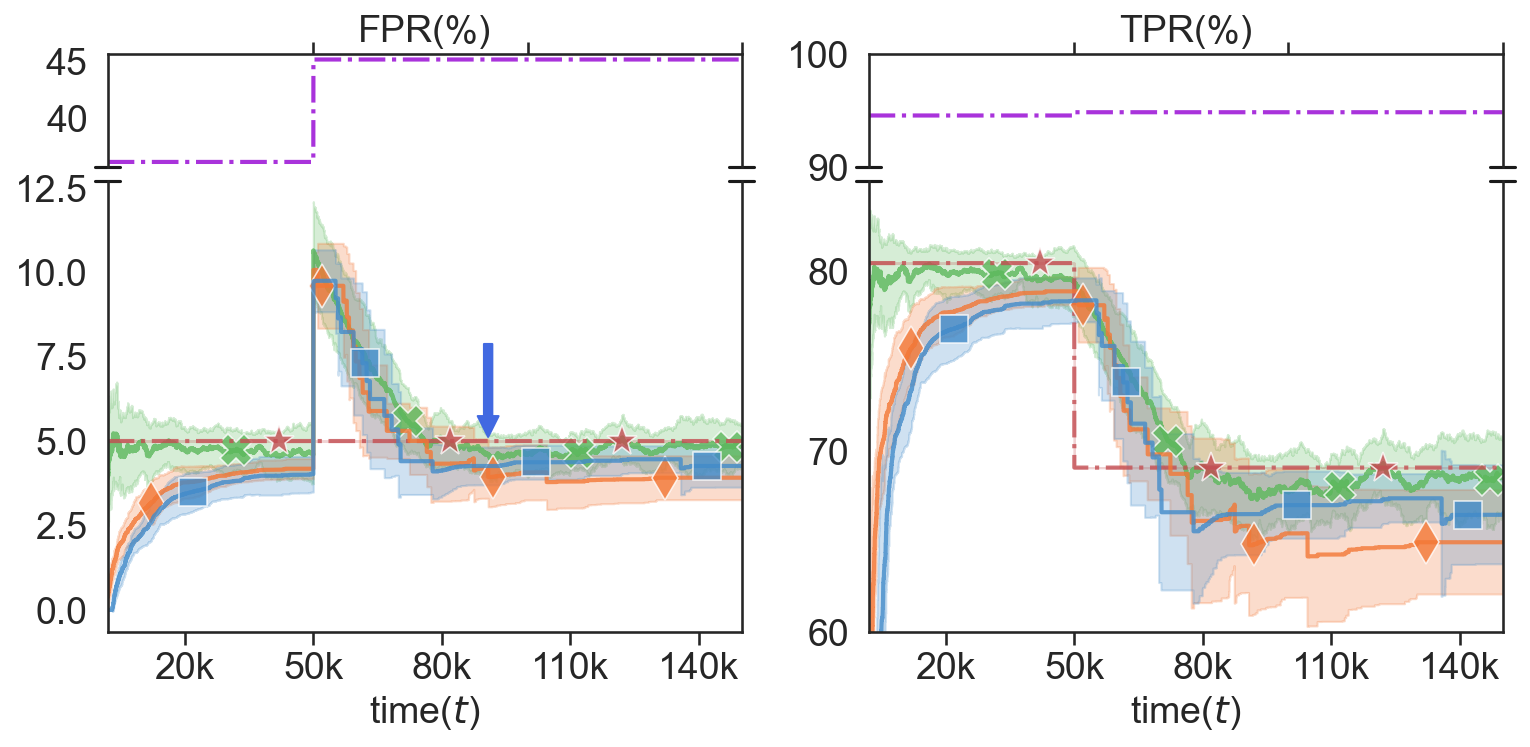}}

    \subfigure[Distribution shift, 10k window. \label{fig:knn_shift_win_10k_cifar10}]{\includegraphics[scale=0.215]{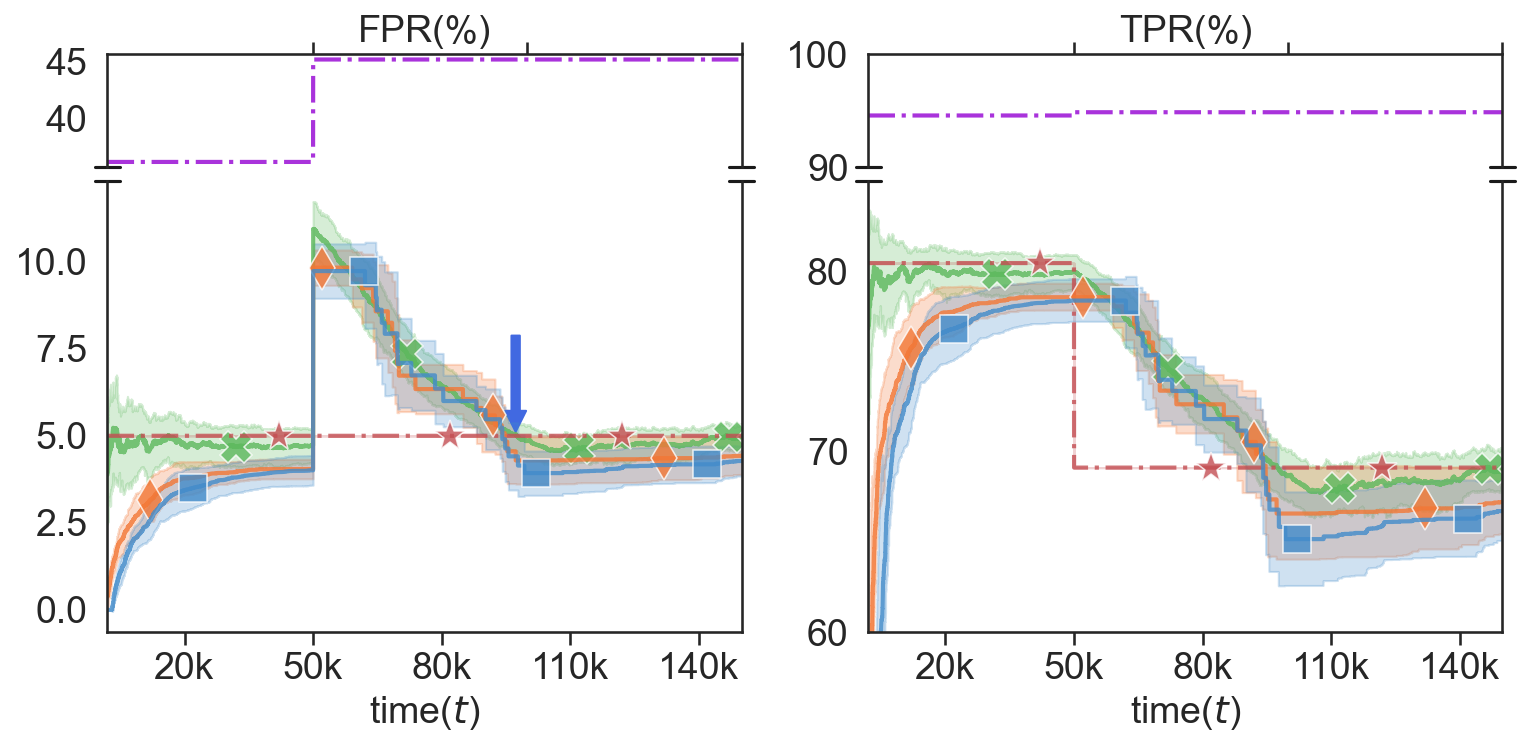}}
  }
  
  \vspace{-5pt}
  \caption{  {
   Results with the KNN scores on Cifar-10 as ID dataset. For (b) and (c) the distribution shifts at $t=50k$. The arrow indicates the time at which the mean FPR + std. deviation over 10 runs goes below 5\% for the LIL method.
    }}
  \label{fig:knn-cifar10-supp}
\end{figure*}

\begin{figure*}[t]
  \centering
\includegraphics[width=0.999\textwidth]{figs-v2/legend_hfdng-3.png}
  \mbox{
  \hspace{-10pt}
    \subfigure[No distribution shift, no window.   \label{fig:ebo_no_shift_no_win_cifar10}]{\includegraphics[scale=0.215]{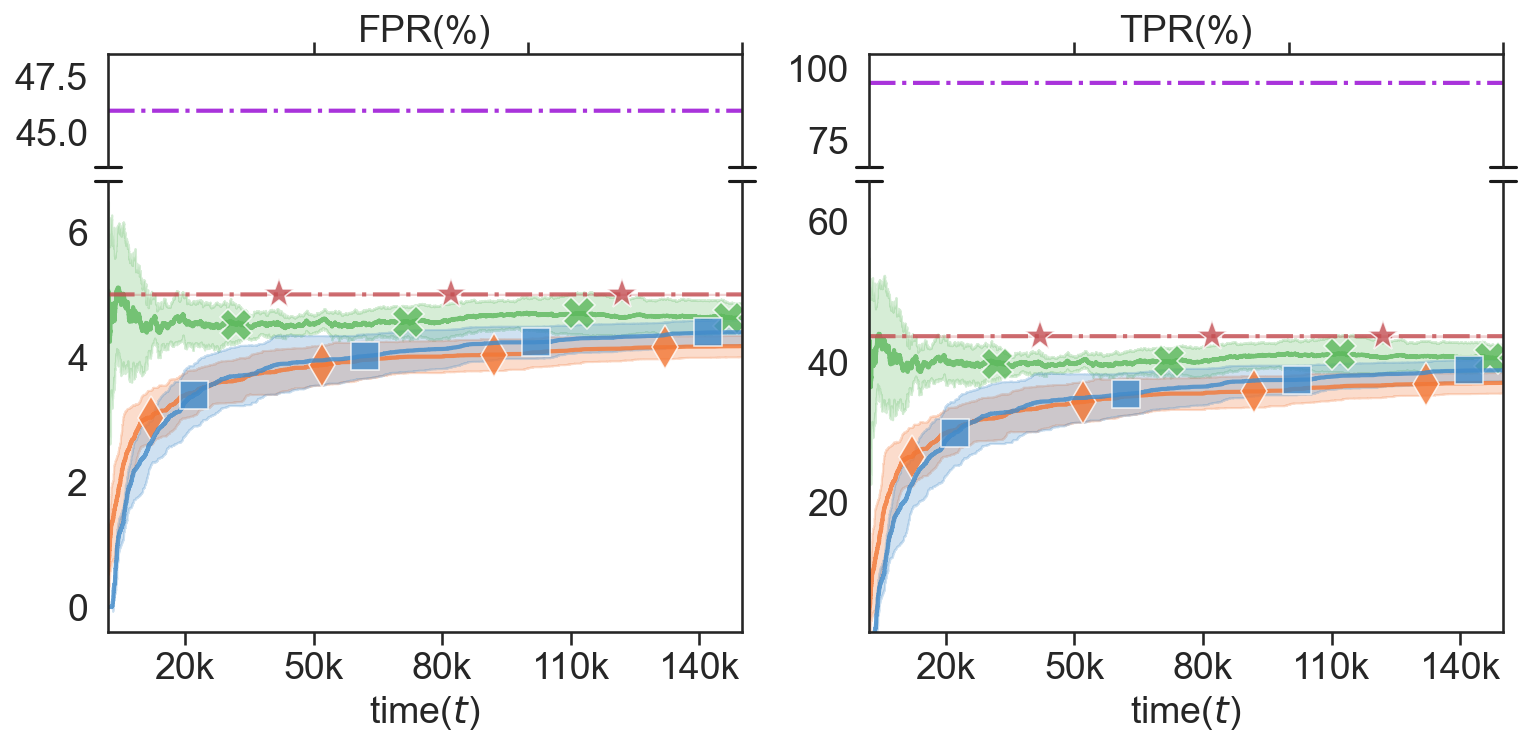}}
    \hspace{1pt}
    
    \subfigure[Distribution shift, 5k window. \label{fig:ebo_shift_win_5k_cifar10}]{\includegraphics[scale=0.215]{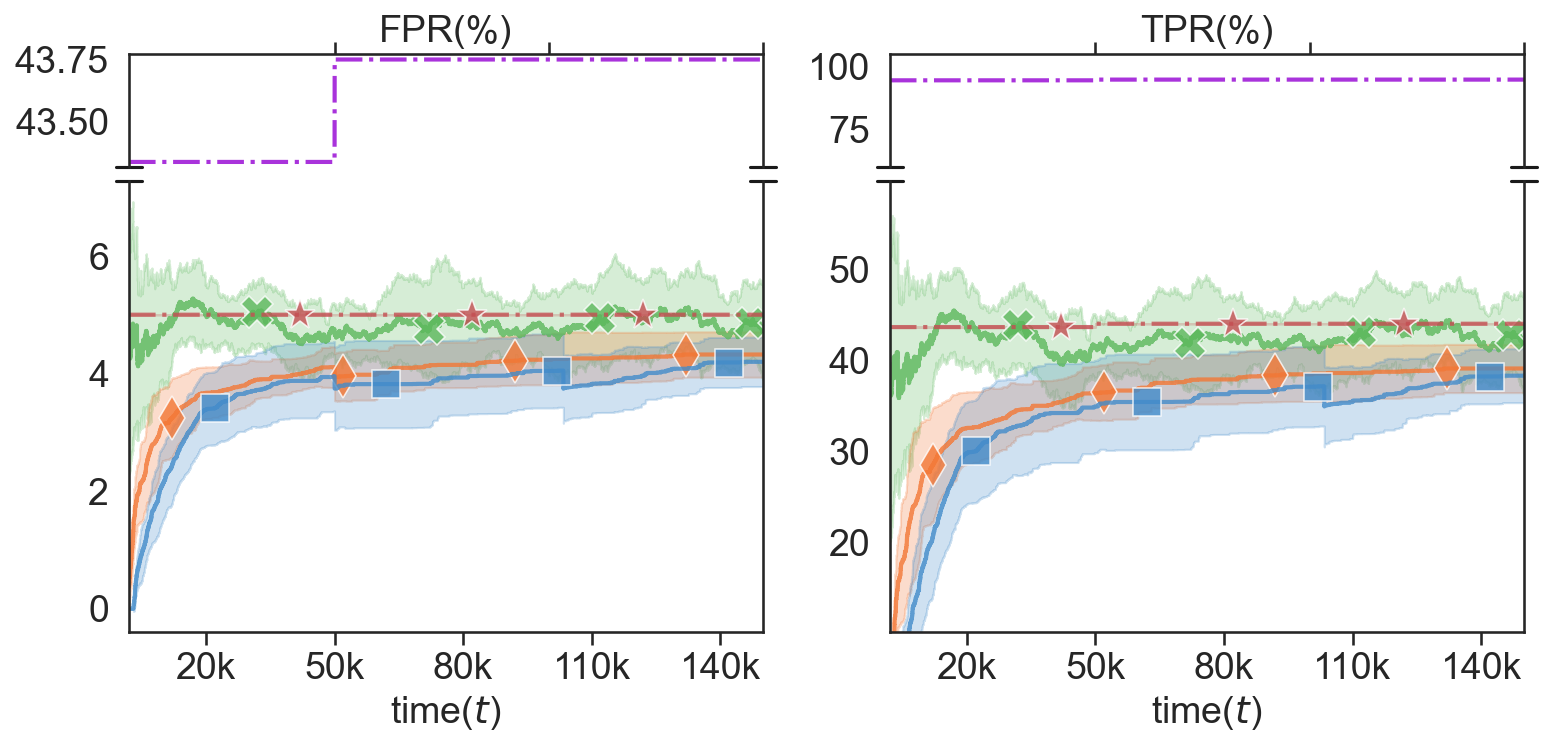}}

    \subfigure[Distribution shift, 10k window. \label{fig:ebo_shift_win_10k_cifar10}]{\includegraphics[scale=0.215]{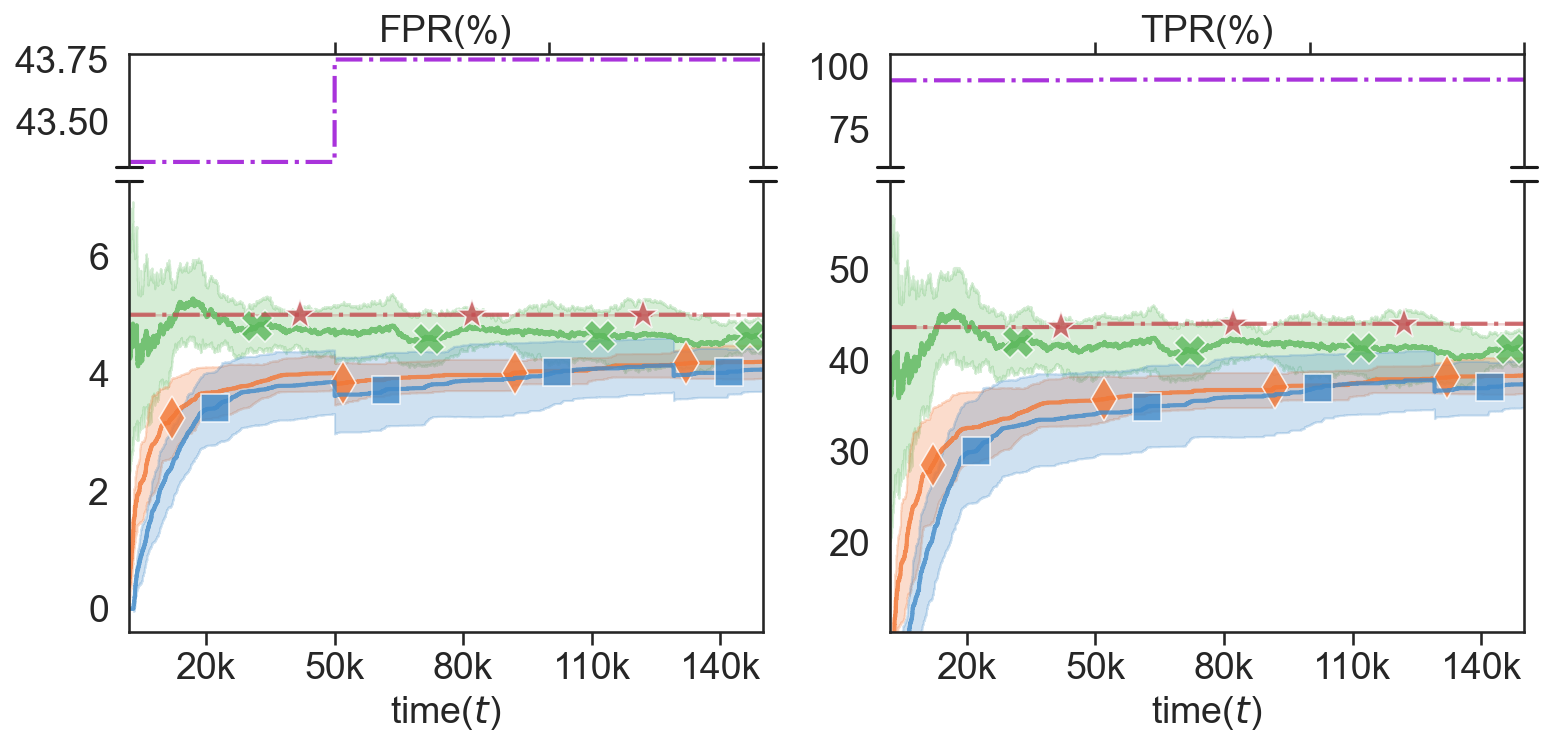}}
  }
  
  \vspace{-5pt}
  \caption{ \small {
   Results with the EBO scores on Cifar-10 as ID dataset. For (b) and (c) the distribution shifts at $t=50k$. The arrow indicates the time at which the mean FPR + std. deviation over 10 runs goes below 5\% for the LIL method.
    }}
  \label{fig:ebo-cifar10}
  \vspace{-5pt}
\end{figure*}

\begin{figure*}[h]
  \centering
  \mbox{
  \hspace{0pt}
    \subfigure[No distribution shift, no window.   \label{fig:mds_no_shift_no_win_cifar10}]{\includegraphics[scale=0.215]{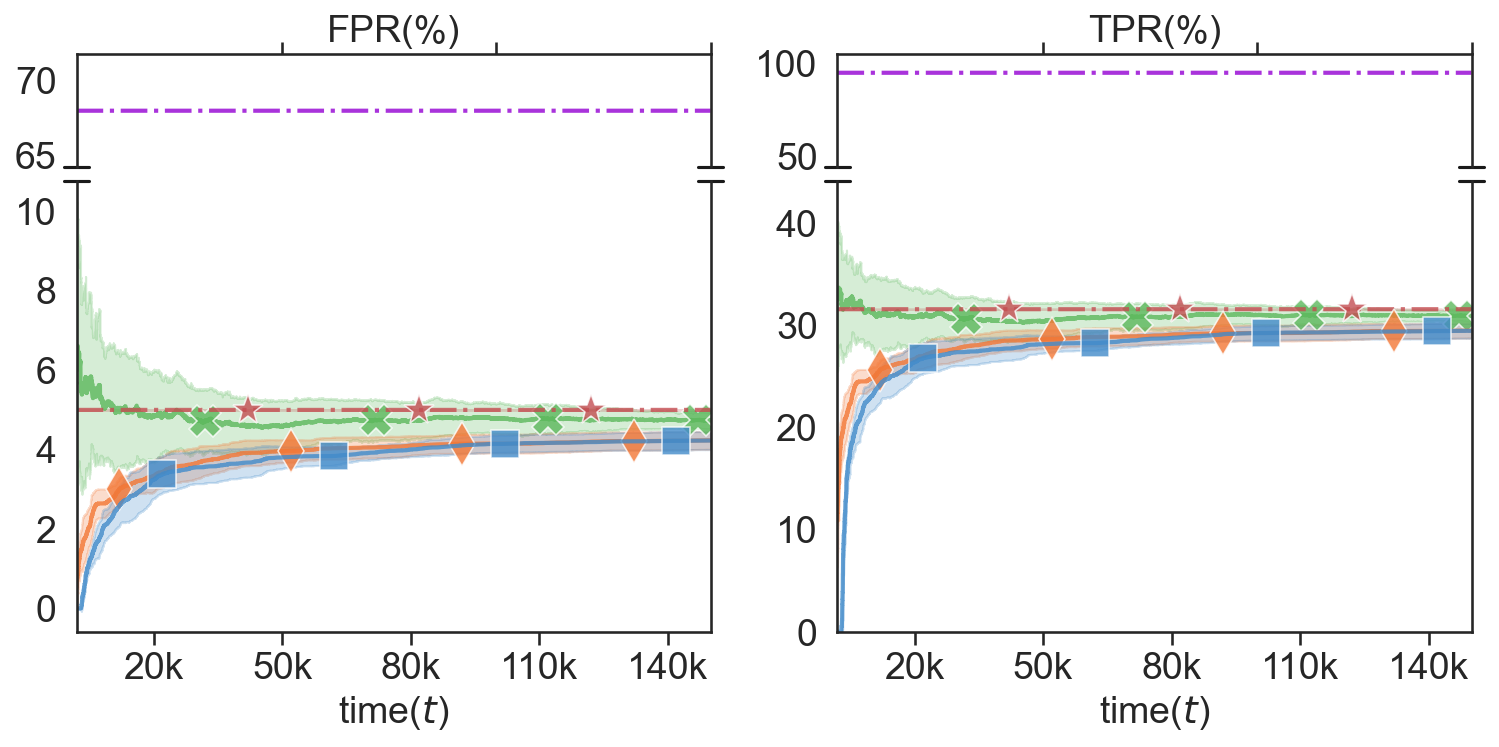}}
    \hspace{1pt}
    
    \subfigure[Distribution shift, 5k window. \label{fig:mds_shift_win_5k_cifar10}]{\includegraphics[scale=0.215]{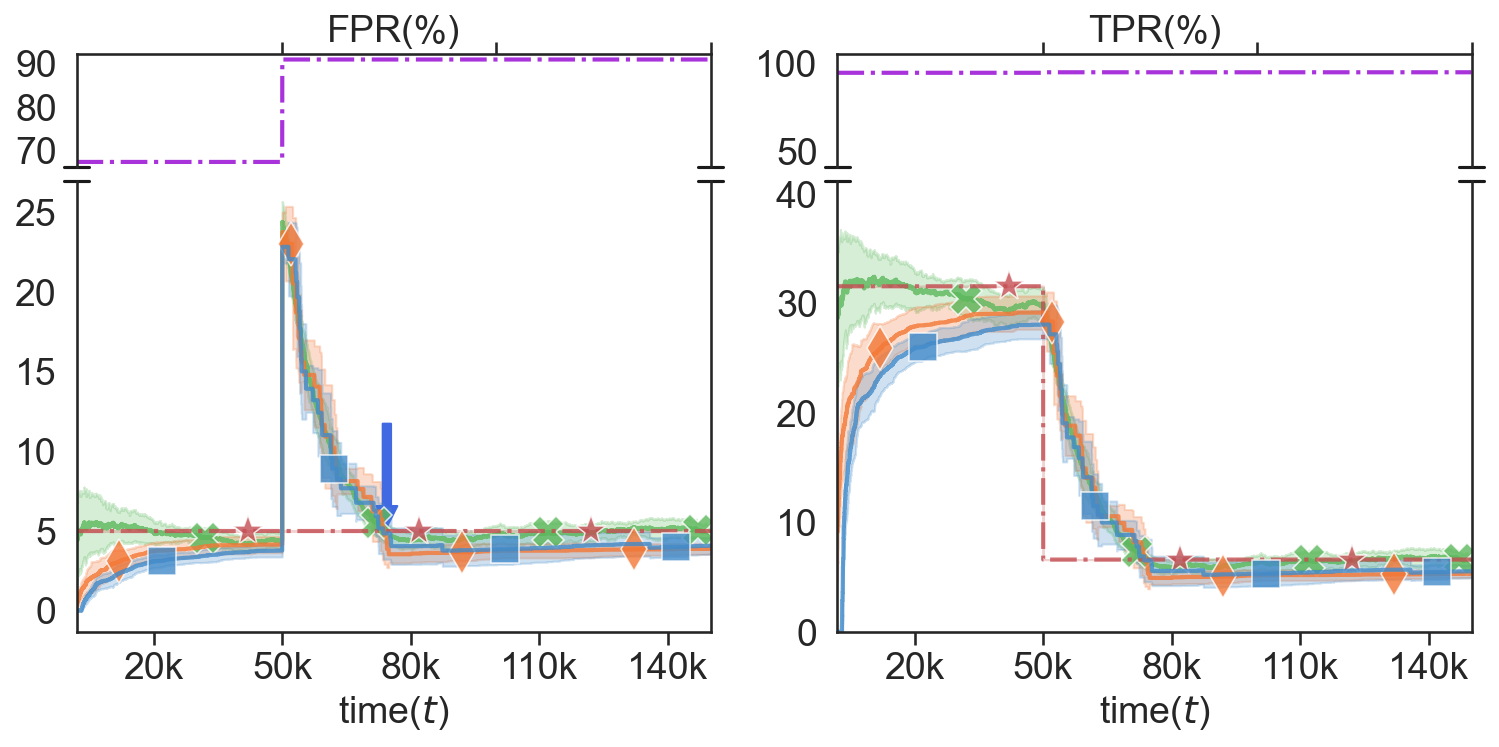}}

    \subfigure[Distribution shift, 10k window. \label{fig:mds_shift_win_10k_cifar10}]{\includegraphics[scale=0.215]{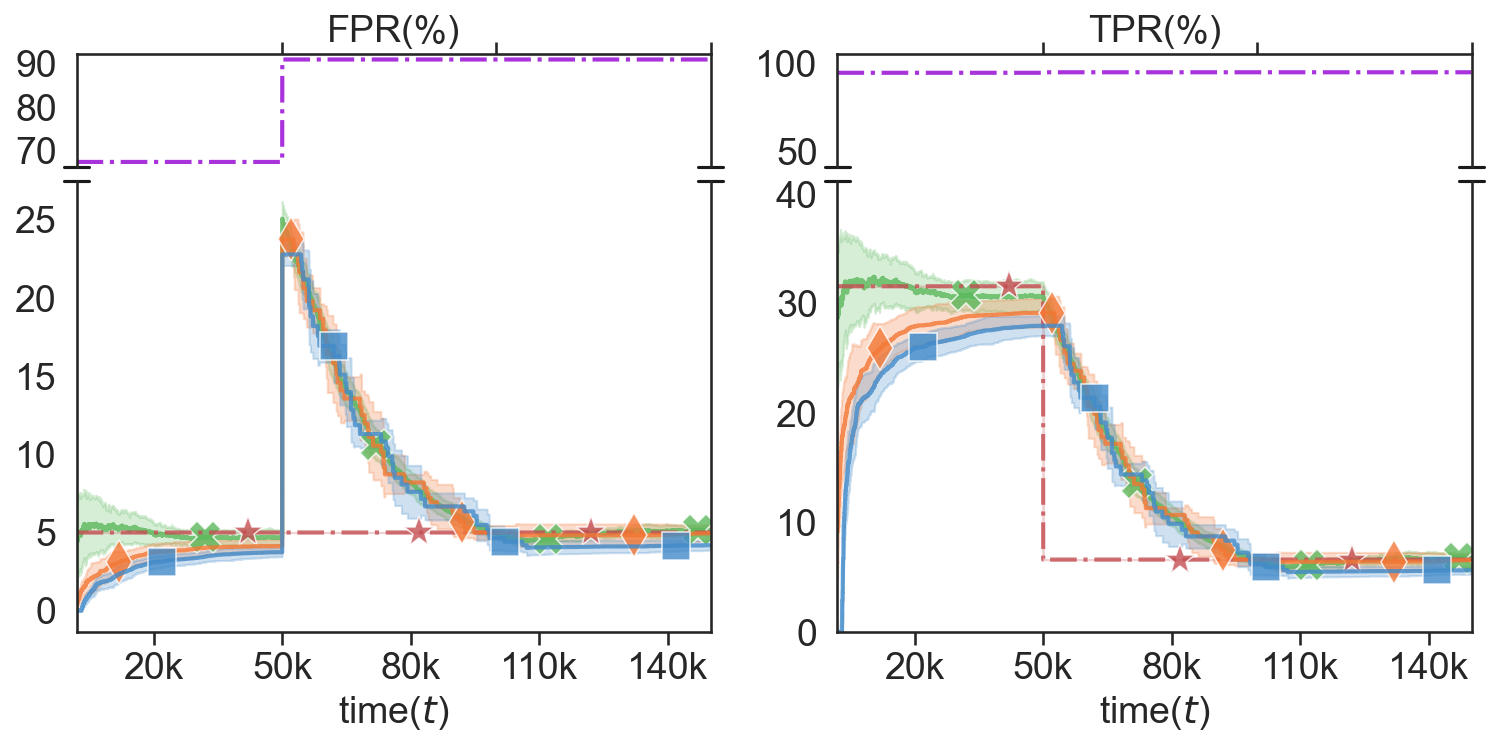}}
  }
  
  \vspace{-5pt}
  \caption{ \small {
   Results with the MDS scores on Cifar-10 as ID dataset. For (b) and (c) the distribution shifts at $t=50k$. The arrow indicates the time at which the mean FPR + std. deviation over 10 runs goes below 5\% for the LIL method.
    }}
  \label{fig:mds-cifar10}
\end{figure*}

\begin{figure*}[h]
  \centering
  \mbox{
  \hspace{-10pt}
    \subfigure[No distribution shift, no window.   \label{fig:vim_no_shift_no_win_cifar10}]{\includegraphics[scale=0.215]{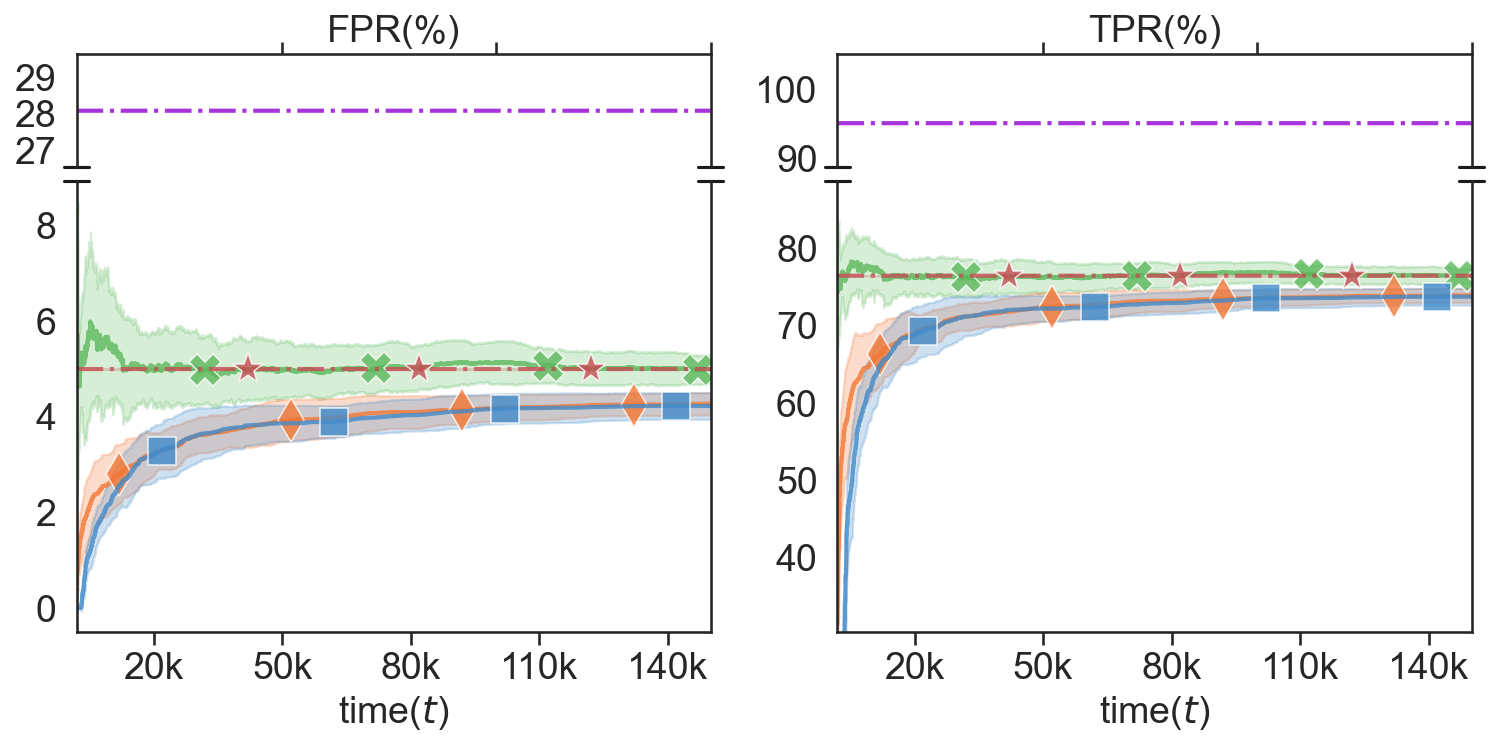}}
    \hspace{1pt}
    
    \subfigure[Distribution shift, 5k window. \label{fig:vim_shift_win_5k_cifar10}]{\includegraphics[scale=0.215]{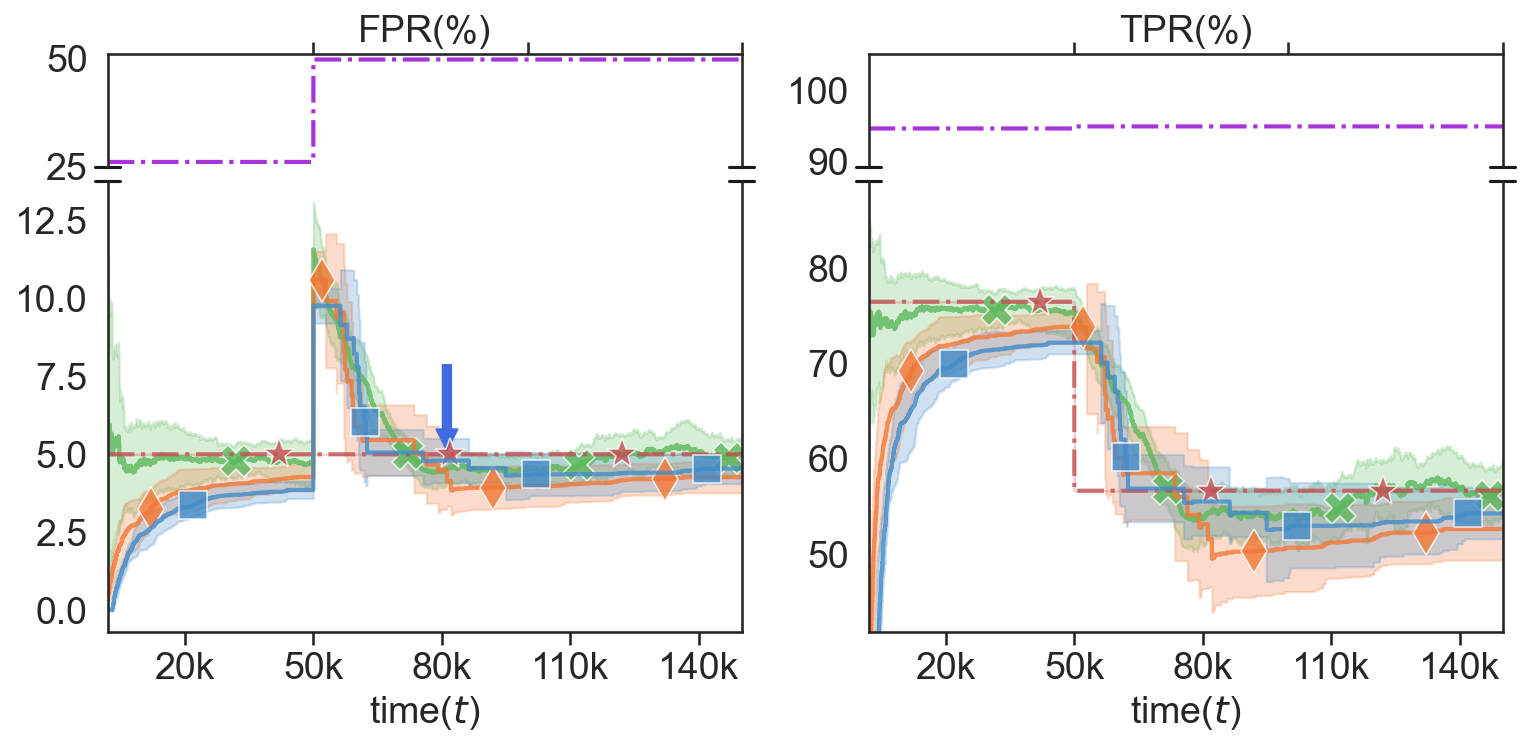}}

    \subfigure[Distribution shift, 10k window. \label{fig:vim_shift_win_10k_cifar10}]{\includegraphics[scale=0.215]{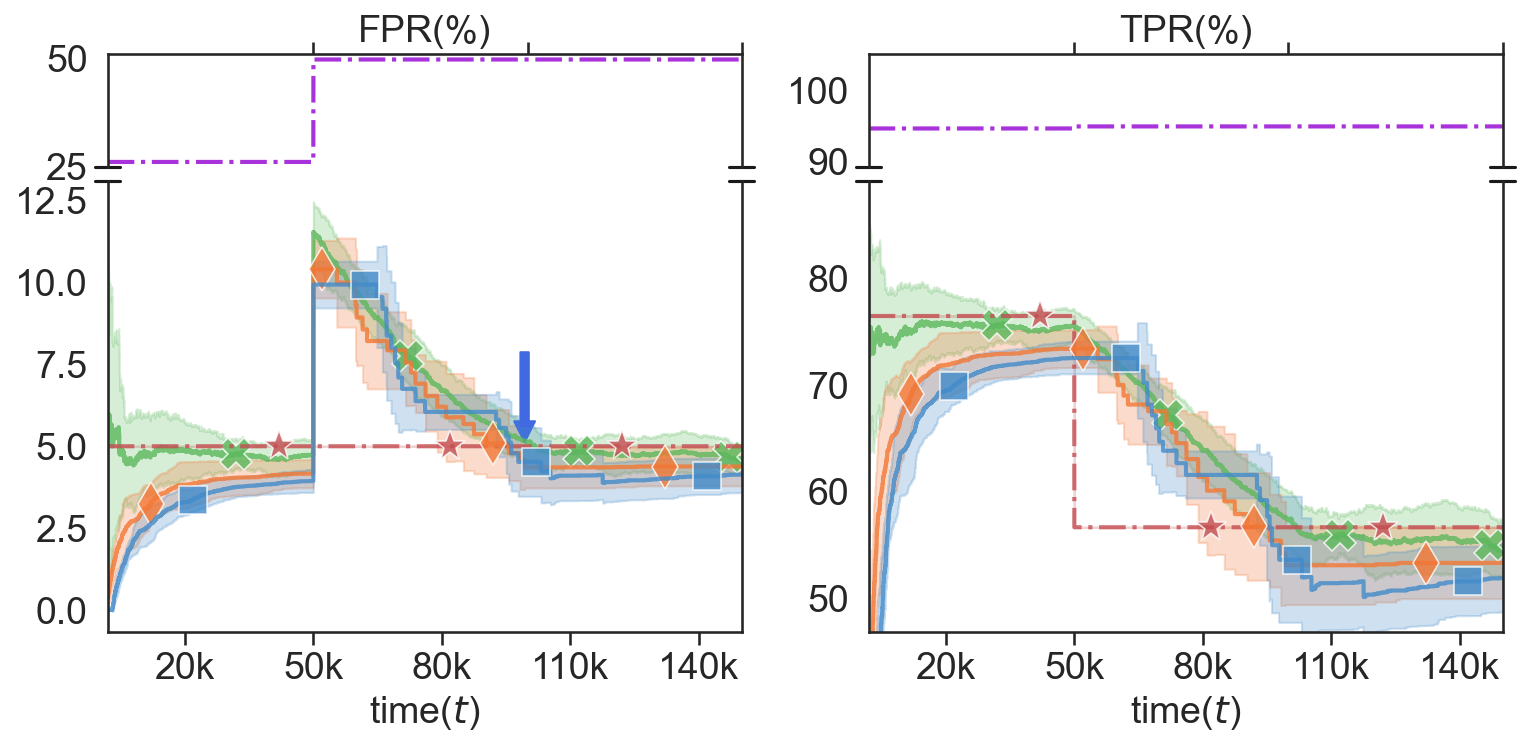}}
  }
  
  \vspace{-5pt}
  \caption{ \small {
   Results with the VIM scores on Cifar-10 as ID dataset. For (b) and (c) the distribution shifts at $t=50k$. The arrow indicates the time at which the mean FPR + std. deviation over 10 runs goes below 5\% for the LIL method.
    }}
  \label{fig:vim-cifar10}
  \vspace{-10pt}
\end{figure*}

\begin{figure*}[h]
  \centering
  \mbox{
  \hspace{-10pt}
    \subfigure[No distribution shift, no window.   \label{fig:odin_no_shift_no_win_cifar10}]{\includegraphics[scale=0.22]{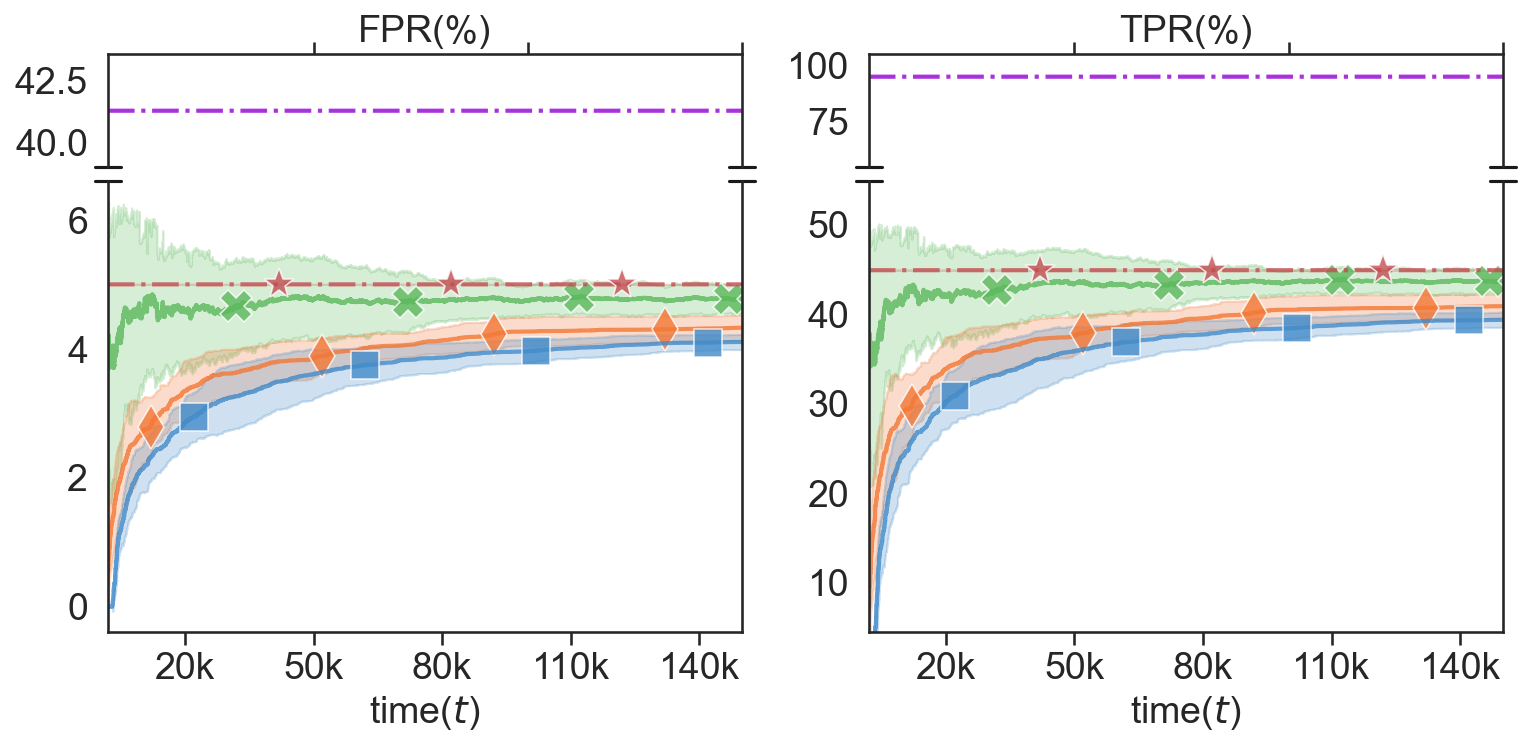}}
    \hspace{1pt}
    
    \subfigure[Distribution shift, 5k window. \label{fig:odin_shift_win_5k_cifar10}]{\includegraphics[scale=0.22]{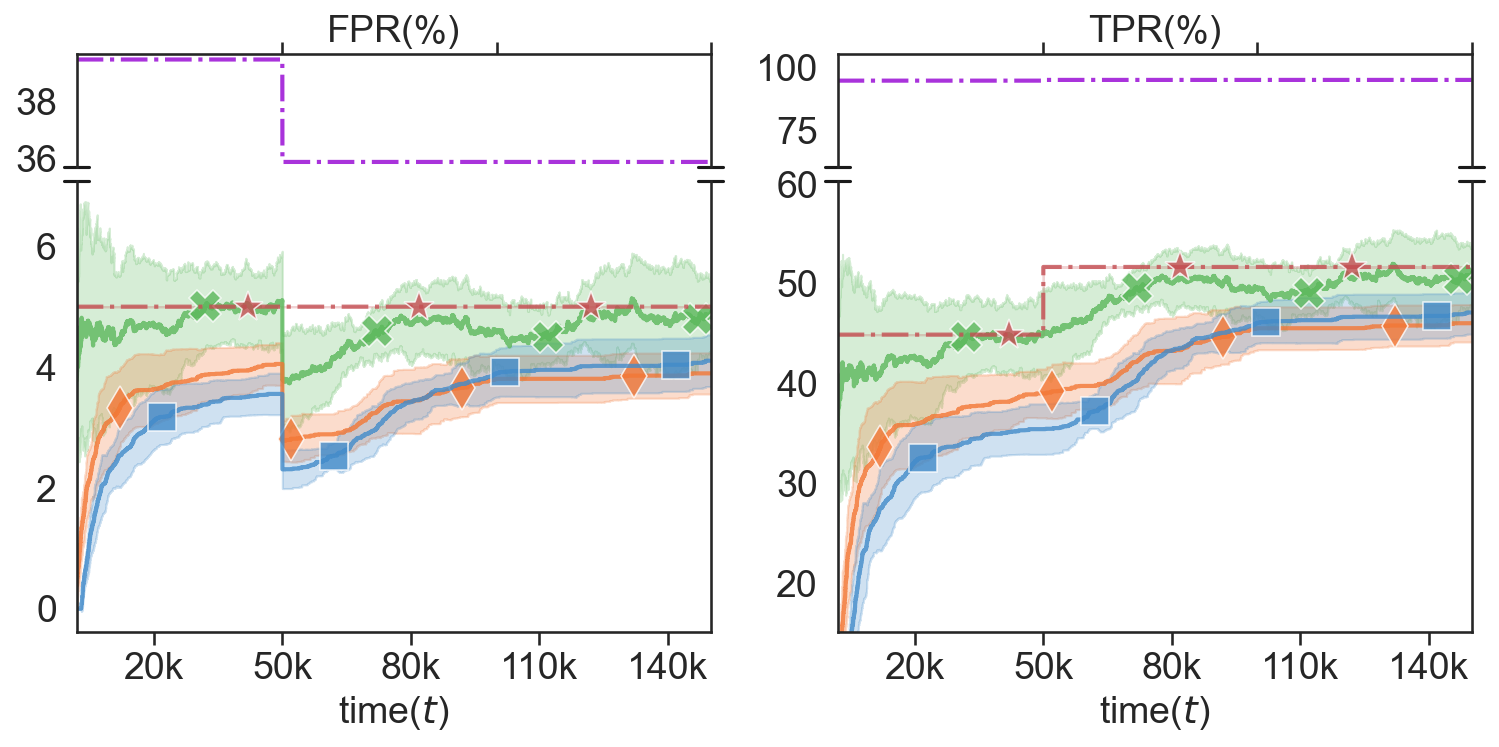}}

    \subfigure[Distribution shift, 10k window. \label{fig:odin_shift_win_10k_cifar10}]{\includegraphics[scale=0.22]{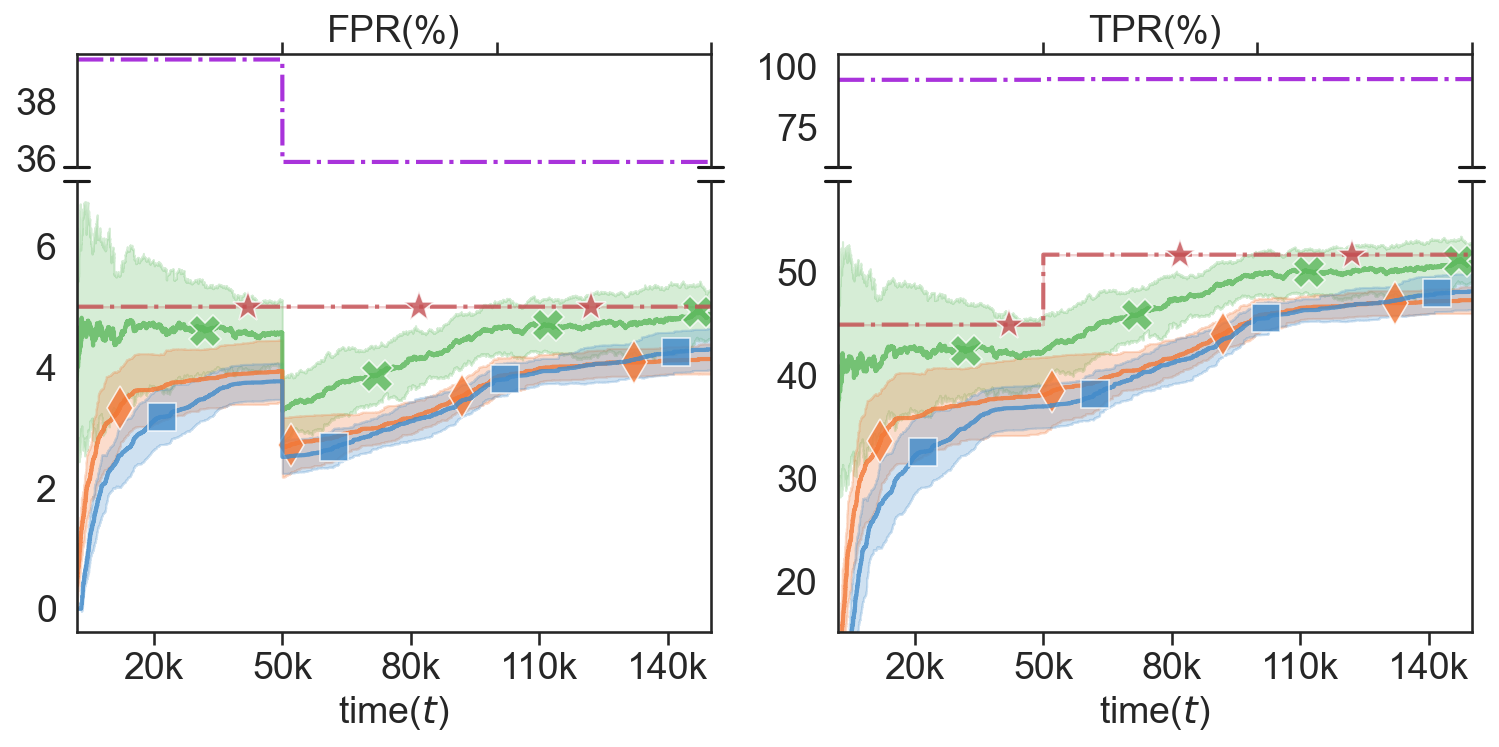}}
  }
  
  \vspace{-5pt}
  \caption{ \small {
   Results with the ODIN scores on Cifar-10 as ID dataset. For (b) and (c) the distribution shifts at $t=50k$. The arrow indicates the time at which the mean FPR + std. deviation over 10 runs goes below 5\% for the LIL method.
    }}
  \label{fig:odin-cifar10}
\end{figure*}

\begin{figure*}[t]
  \centering
  \mbox{
  \hspace{-10pt}
    \subfigure[No distribution shift, no window.   \label{fig:ssd_no_shift_no_win_cifar10}]{\includegraphics[scale=0.215]{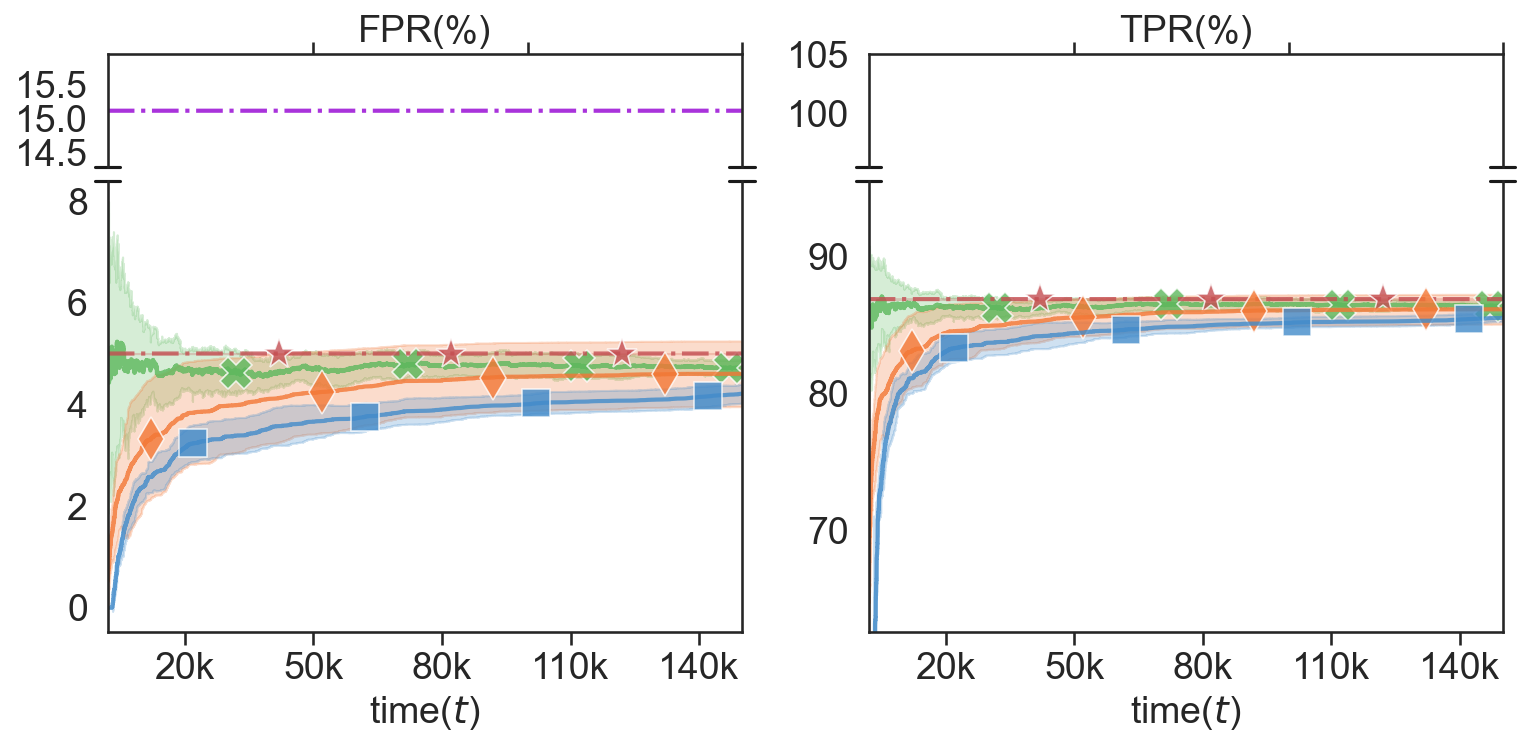}}
    \hspace{1pt}
    
    \subfigure[Distribution shift, 5k window. \label{fig:ssd_shift_win_5k_cifar10}]{\includegraphics[scale=0.215]{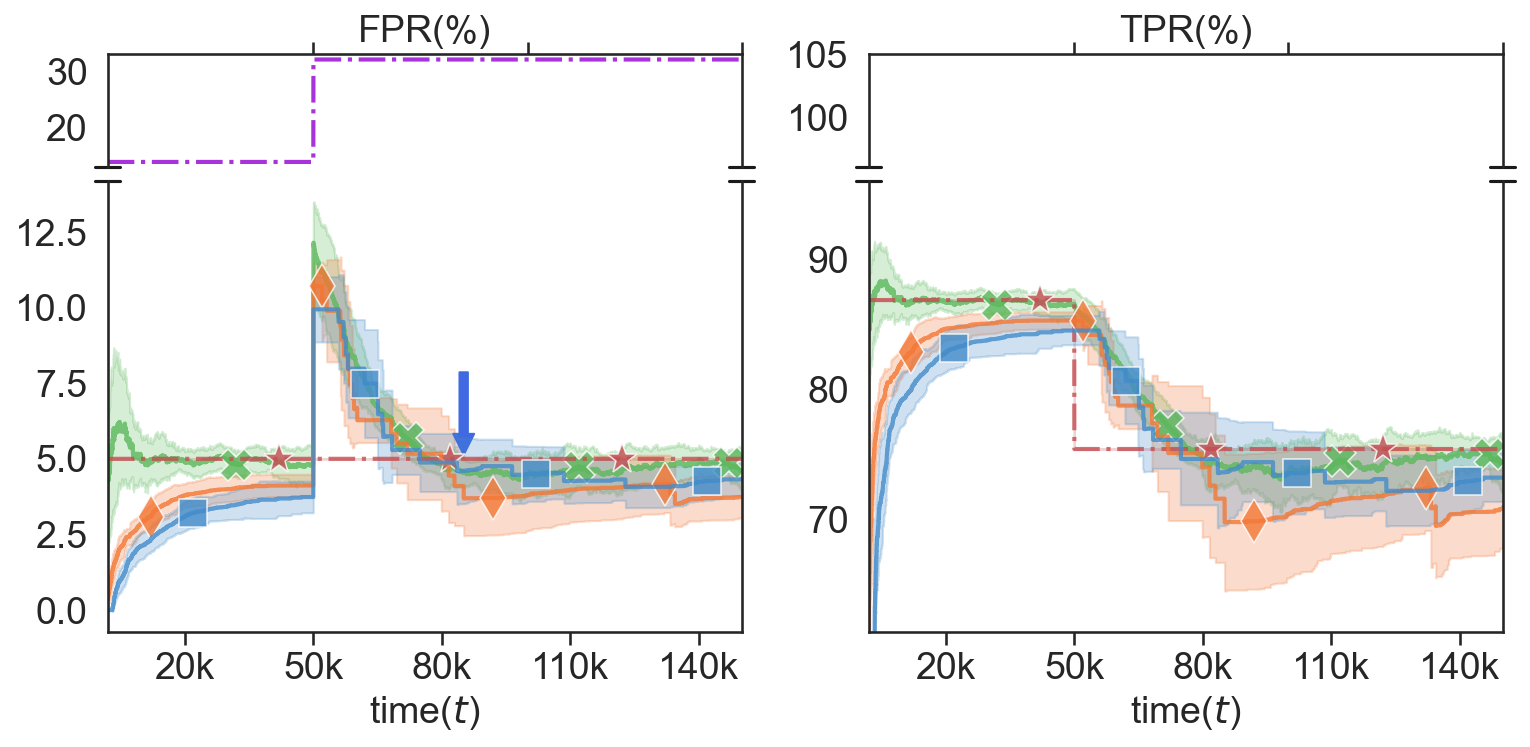}}

    \subfigure[Distribution shift, 10k window. \label{fig:ssd_shift_win_10k_cifar10}]{\includegraphics[scale=0.215]{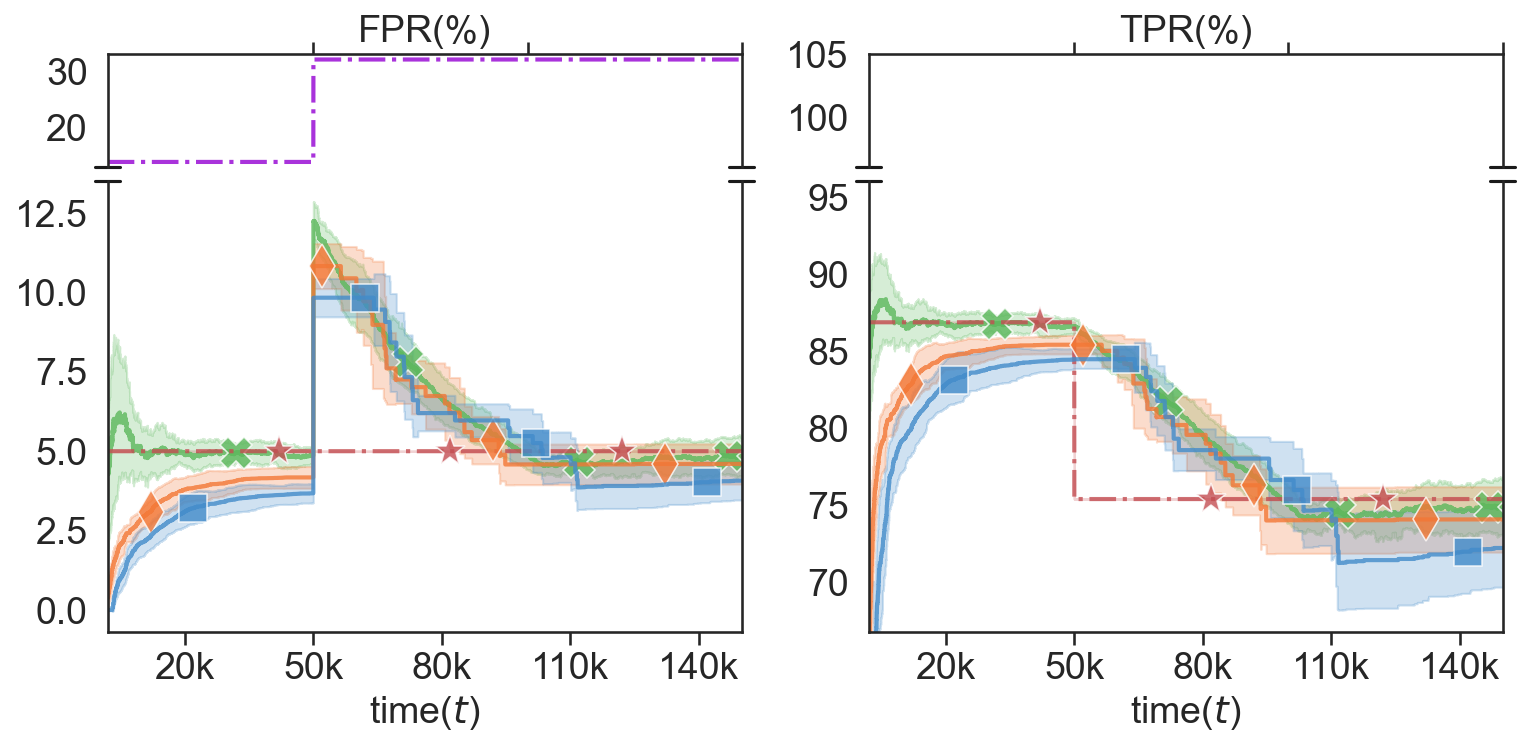}}
  }
  
  \vspace{-5pt}
  \caption{ \small {
   Results with the SSD scores on Cifar-10 as ID dataset. For (b) and (c) the distribution shifts at $t=50k$. The arrow indicates the time at which the mean FPR + std. deviation over 10 runs goes below 5\% for the LIL method.
    }}
  \label{fig:ssd-cifar10}
  \vspace{-10pt}
\end{figure*}

\begin{figure*}[h]
  \centering
 
\includegraphics[width=0.999\textwidth]{figs-v2/legend_hfdng-3.png}
  \mbox{
  \hspace{-10pt}
    \subfigure[No distribution shift, no window.   ]{\includegraphics[scale=0.22]{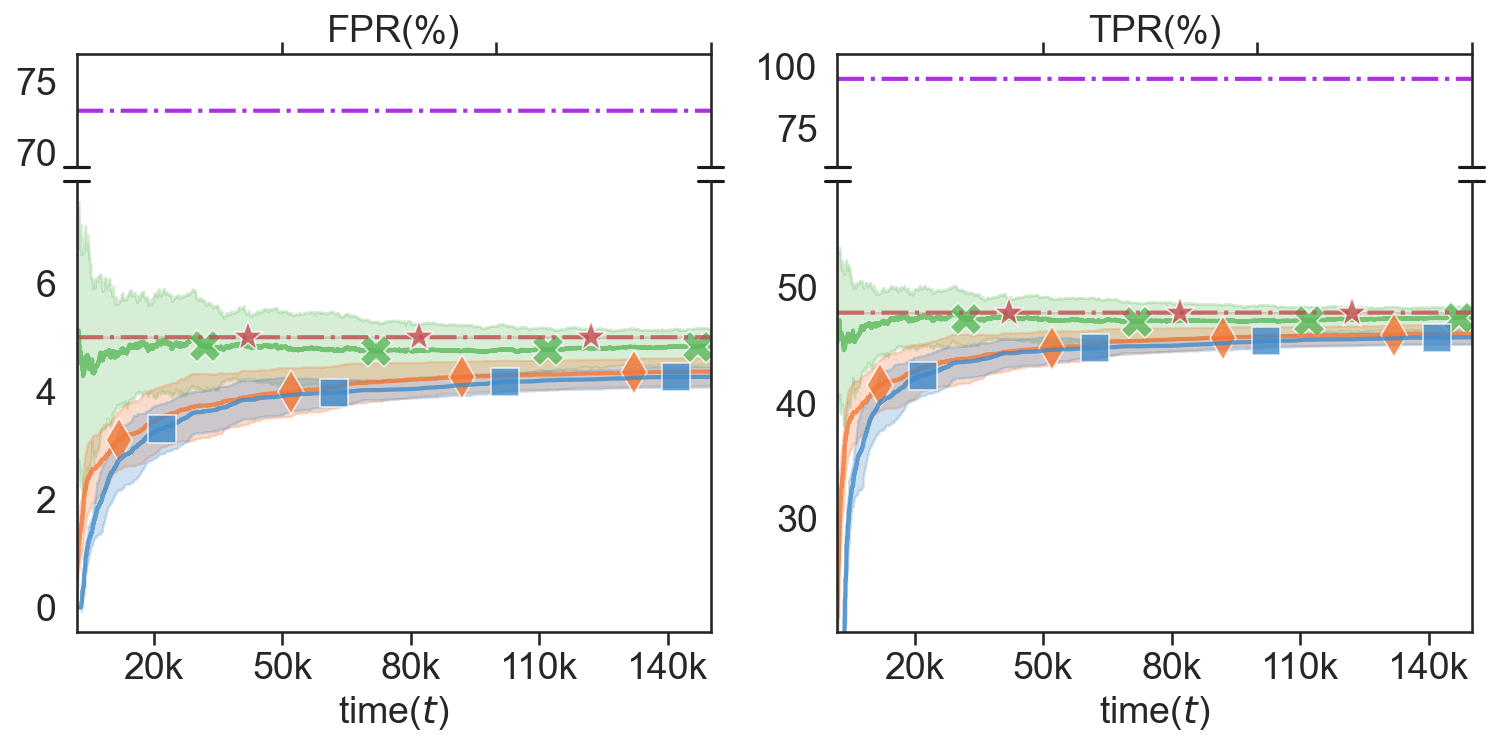}}
    \hspace{1pt}
    
    \subfigure[Distribution shift, 5k window. ]{\includegraphics[scale=0.22]{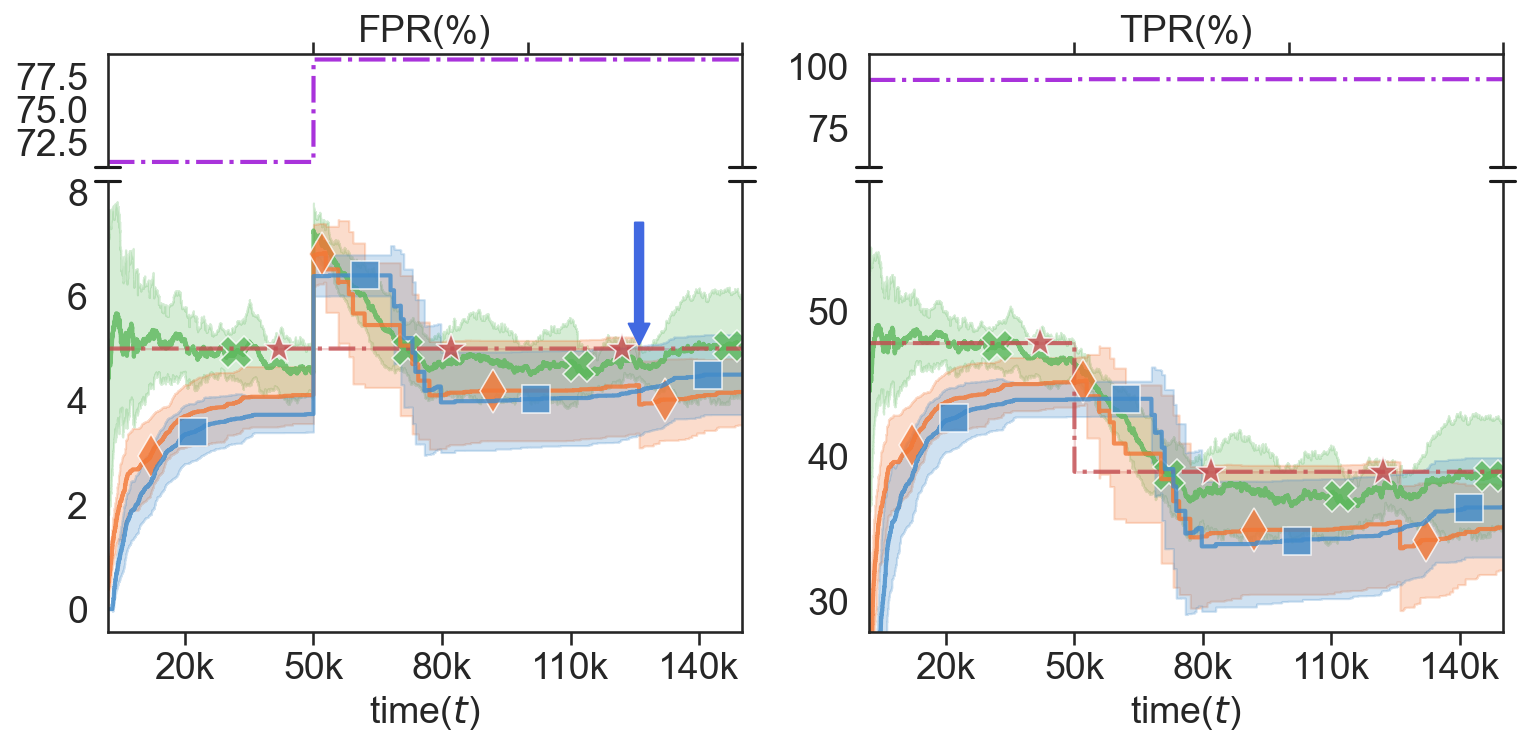}}

    \subfigure[Distribution shift, 10k window.]{\includegraphics[scale=0.22]{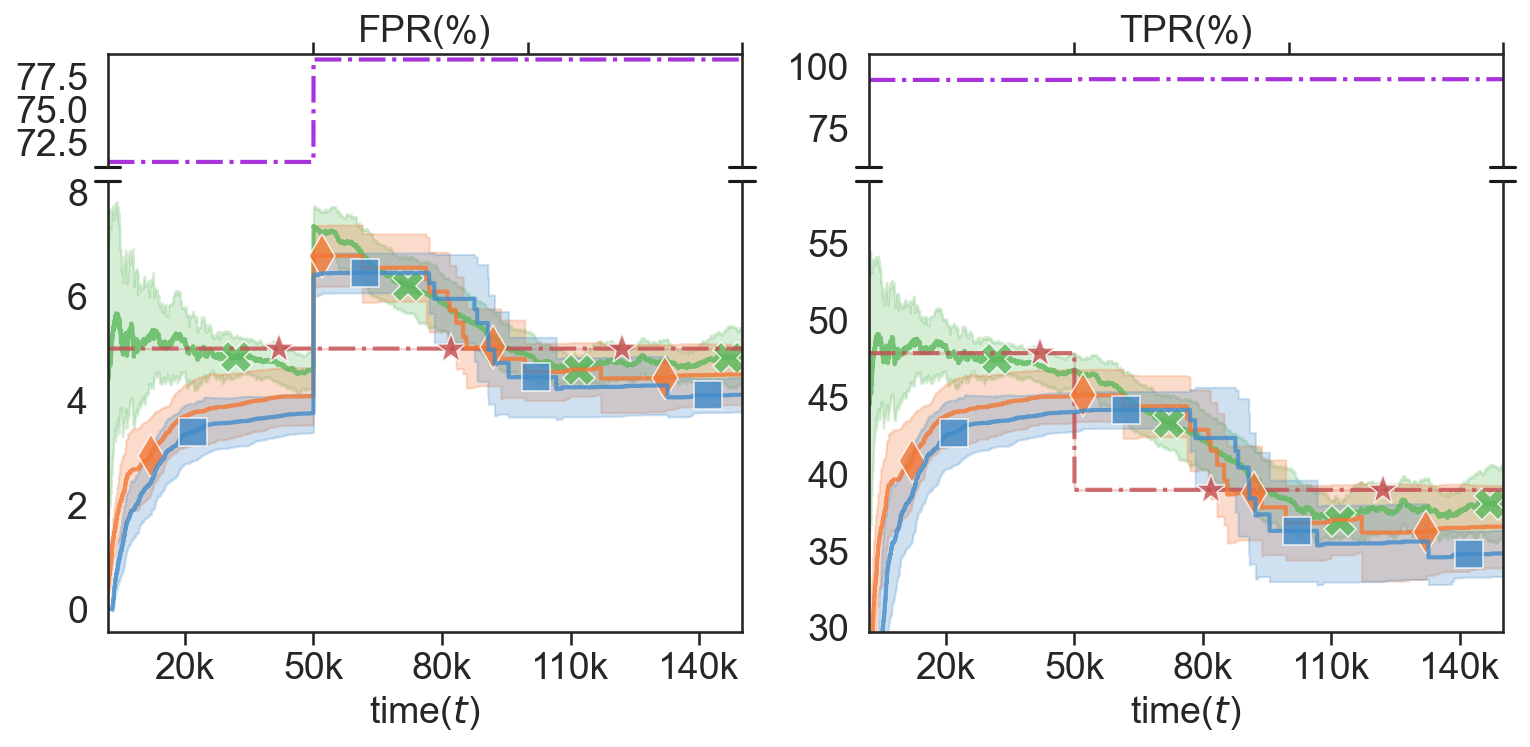}}
  }
  
  \vspace{-5pt}
  \caption{ \small {
   Results with the KNN scores on Cifar-100 as ID dataset. For (b) and (c) the distribution shifts at $t=50k$. The arrow indicates the time at which the mean FPR + std. deviation over 10 runs goes below 5\% for the LIL method.
    }}
\label{fig:knn-cifar100}
  \vspace{-10pt}
\end{figure*}

\begin{figure*}[h]
  \centering
  \mbox{
  \hspace{-10pt}
    \subfigure[No distribution shift, no window.   \label{fig:ebo_no_shift_no_win_cifar100}]{\includegraphics[scale=0.22]{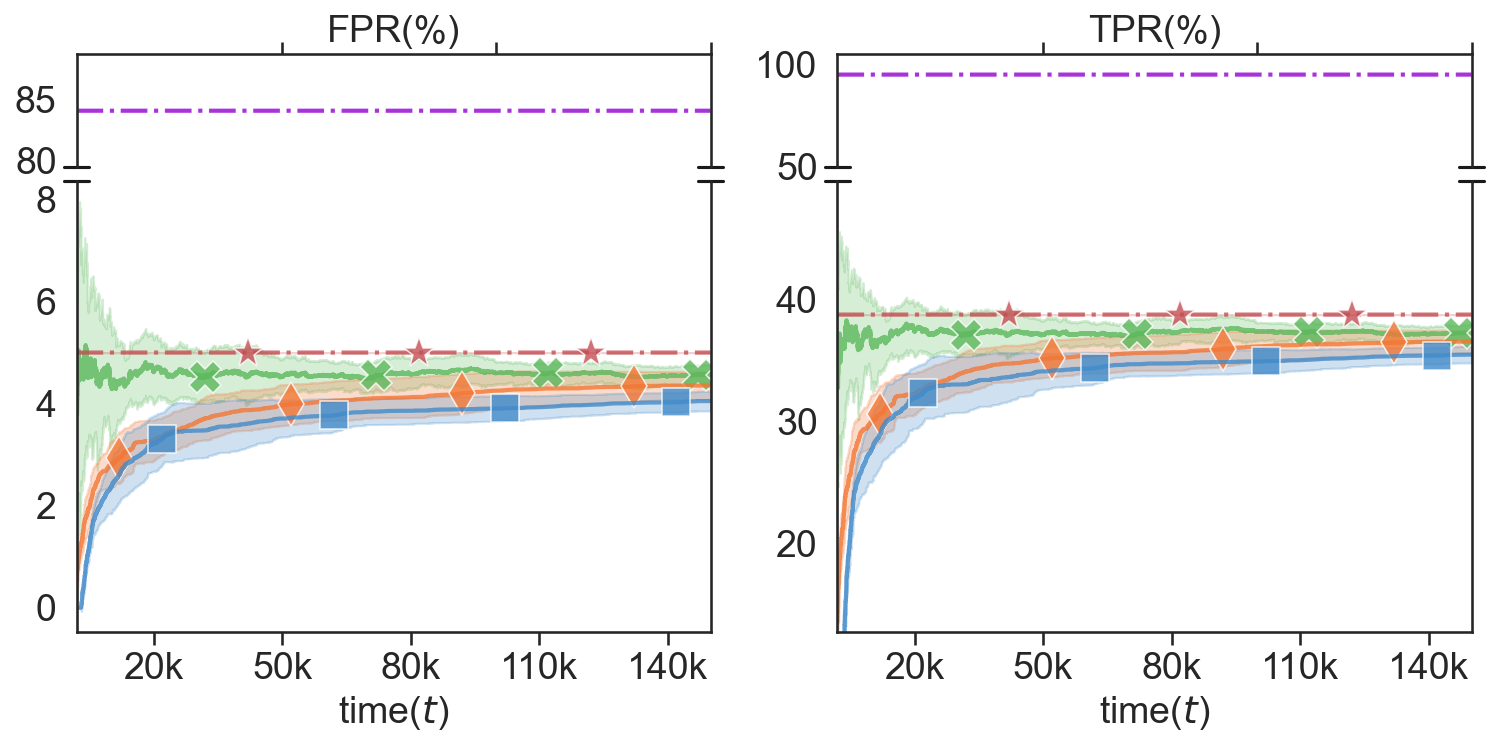}}
    \hspace{1pt}
    
    \subfigure[Distribution shift, 5k window. \label{fig:ebo_shift_win_5k_cifar100}]{\includegraphics[scale=0.22]{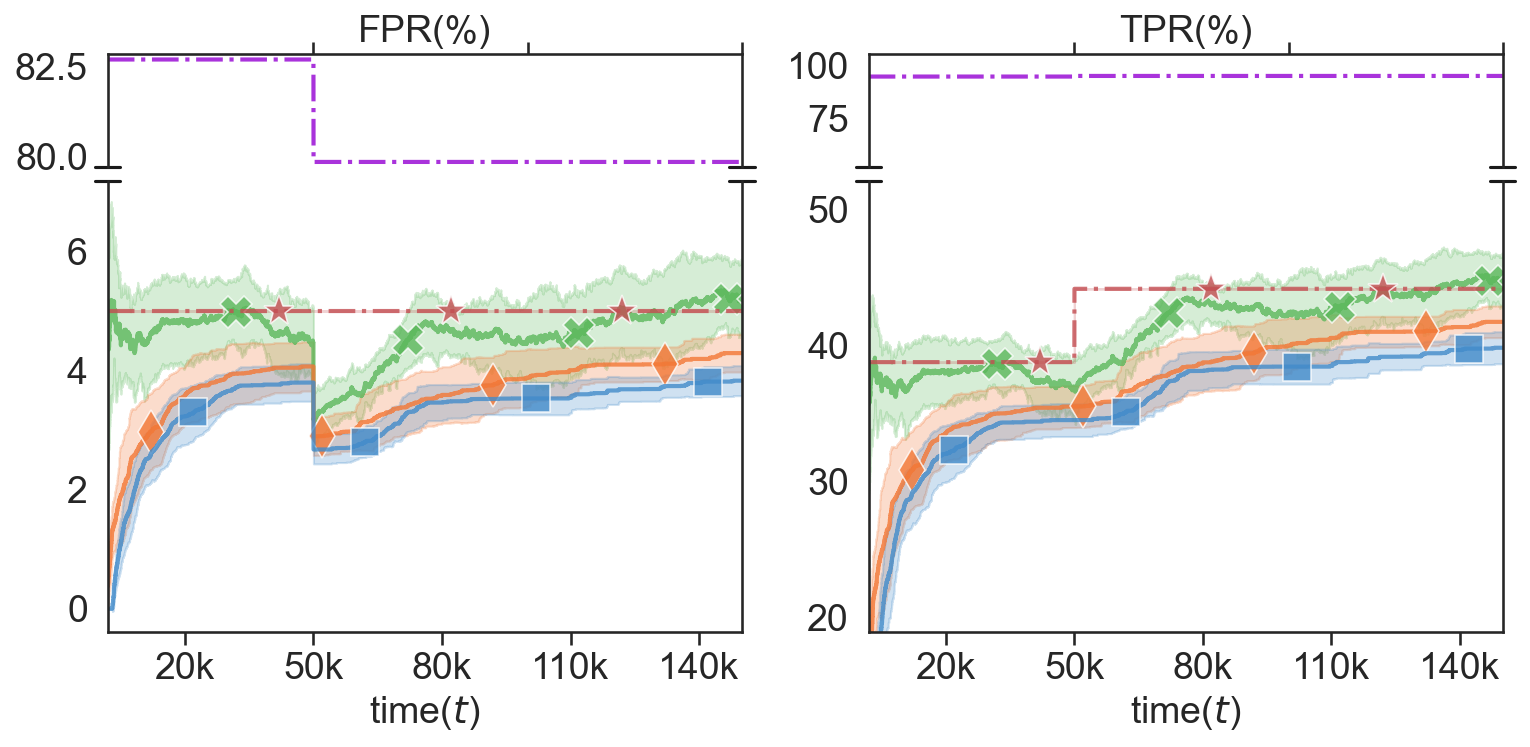}}

    \subfigure[Distribution shift, 10k window. \label{fig:ebo_shift_win_10k_cifar100}]{\includegraphics[scale=0.22]{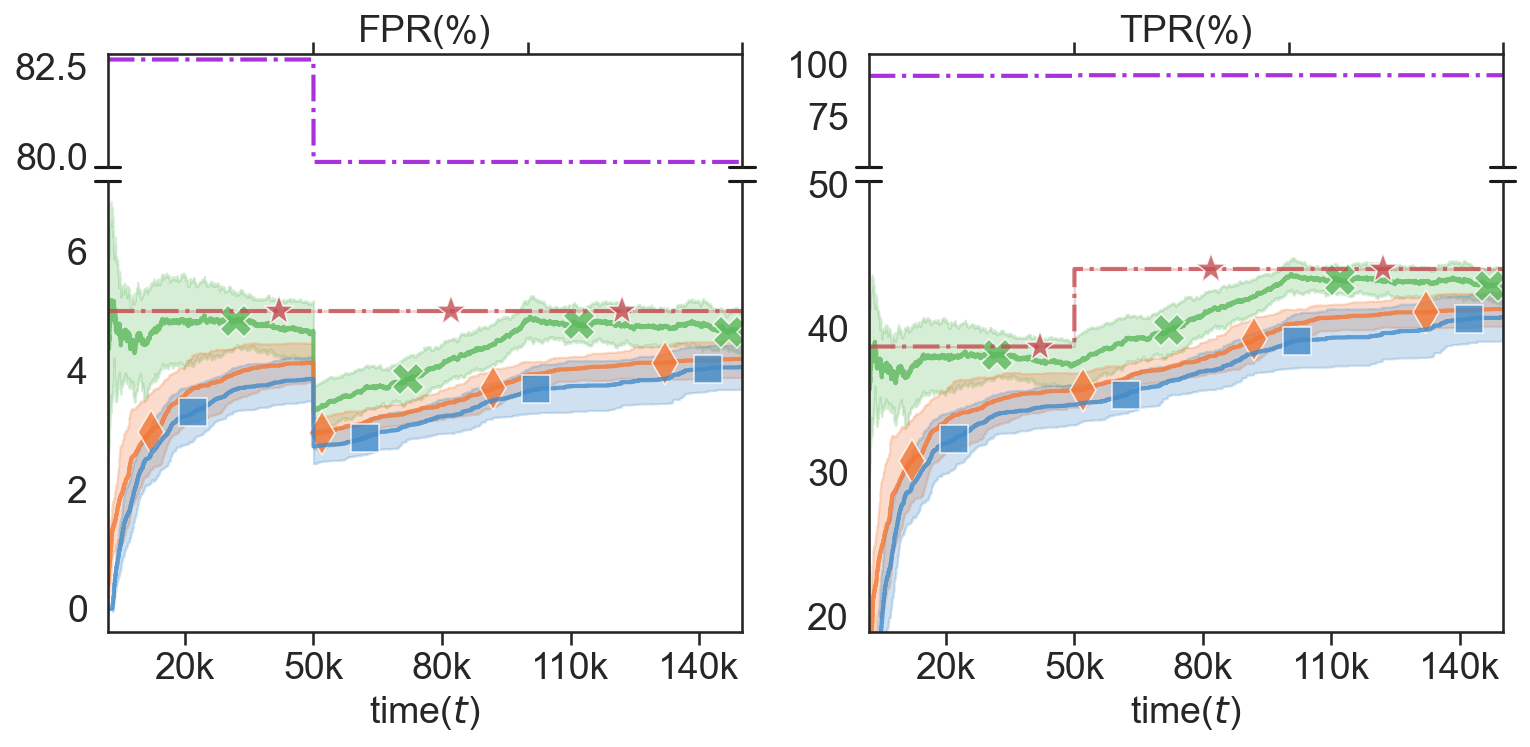}}
  }
  
  \vspace{-5pt}
  \caption{ \small {
   Results with the EBO scores on Cifar-100 as ID dataset. For (b) and (c) the distribution shifts at $t=50k$. The arrow indicates the time at which the mean FPR + std. deviation over 10 runs goes below 5\% for the LIL method.
    }}
  \label{fig:ebo-cifar100}
  \vspace{-10pt}
\end{figure*}

\begin{figure*}[h]
  \centering
  \mbox{
  \hspace{-10pt}
    \subfigure[No distribution shift, no window.   \label{fig:mds_no_shift_no_win_cifar100}]{\includegraphics[scale=0.22]{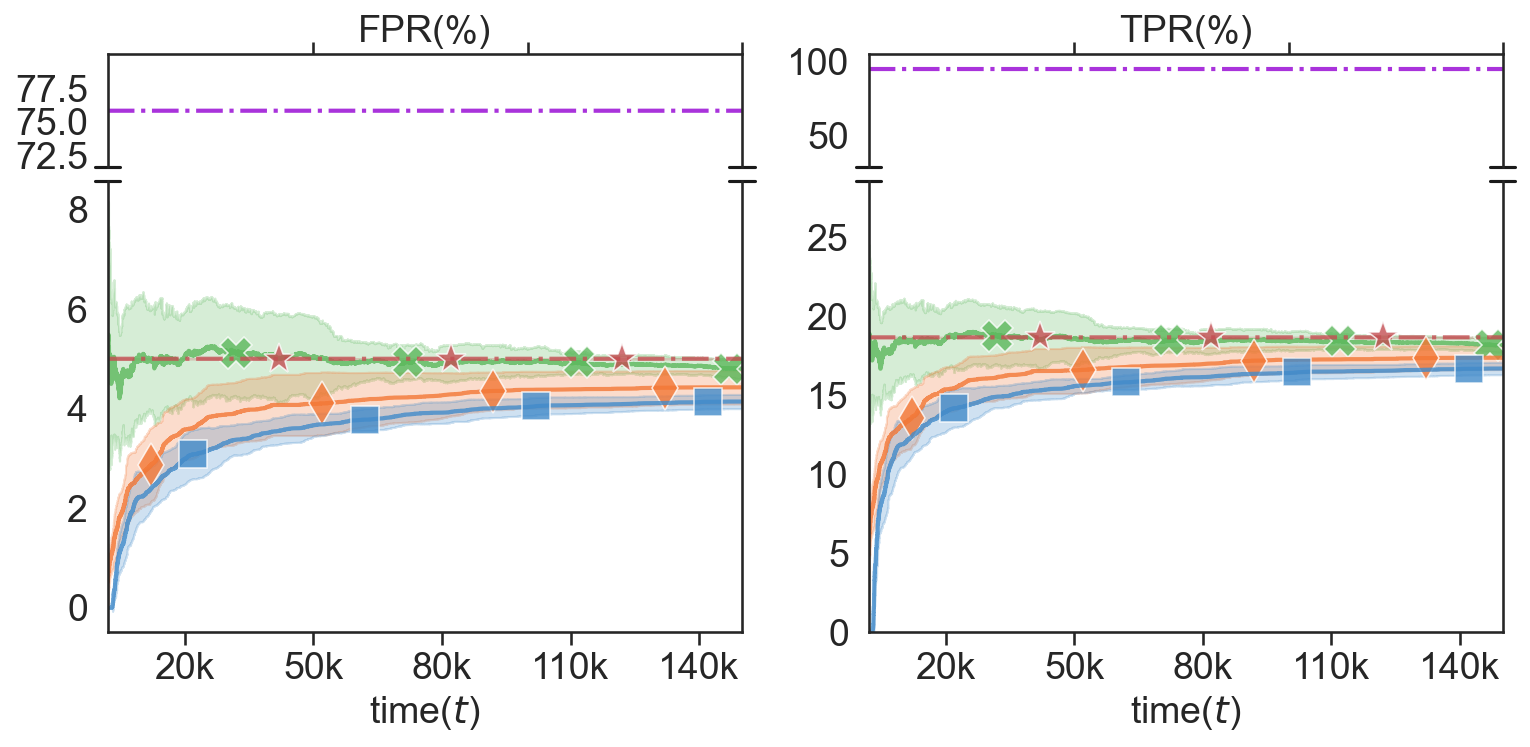}}
    \hspace{1pt}
    
    \subfigure[Distribution shift, 5k window. \label{fig:mds_shift_win_5k_cifar100}]{\includegraphics[scale=0.22]{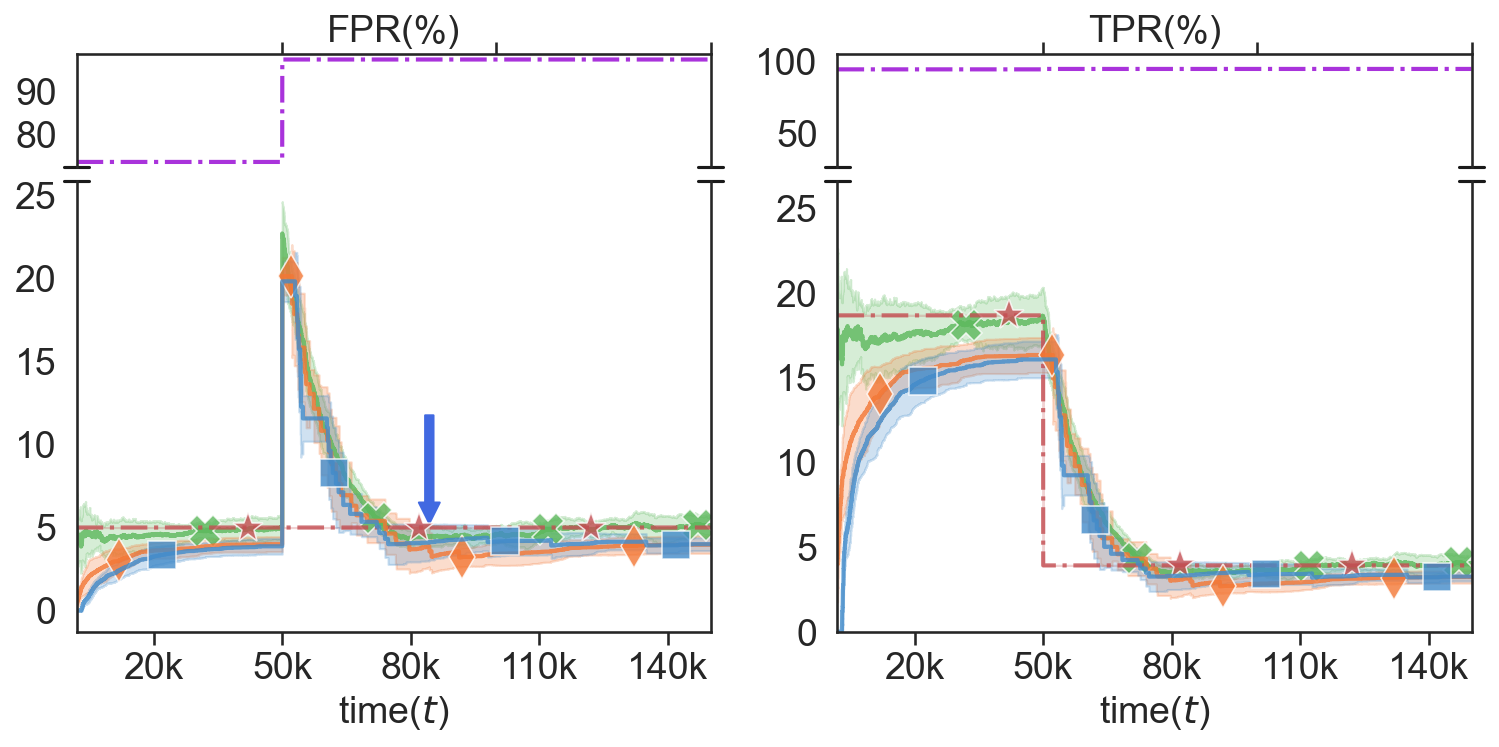}}

    \subfigure[Distribution shift, 10k window. \label{fig:mds_shift_win_10k_cifar100}]{\includegraphics[scale=0.22]{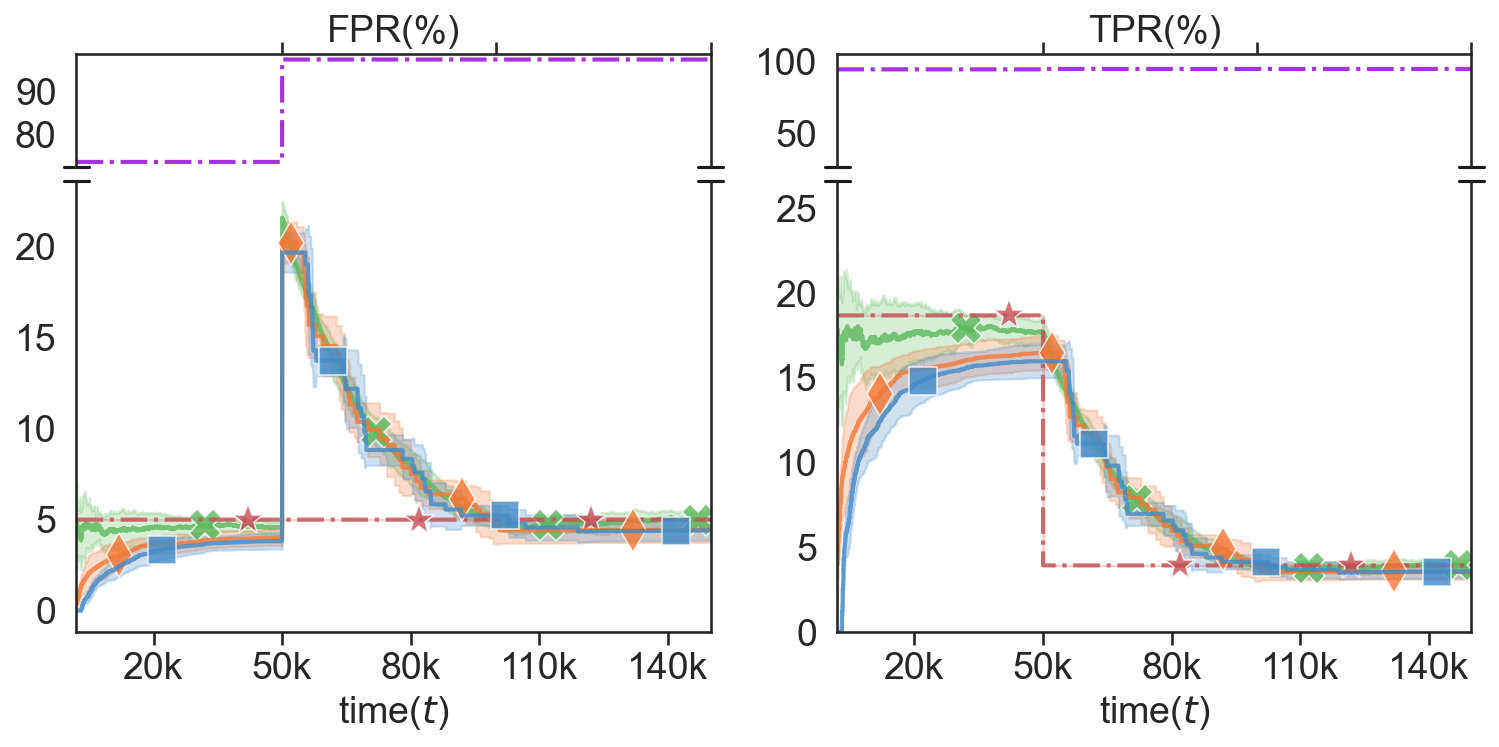}}
  }
  
  \vspace{-5pt}
  \caption{ \small {
   Results with the MDS scores on Cifar-100 as ID dataset. For (b) and (c) the distribution shifts at $t=50k$. The arrow indicates the time at which the mean FPR + std. deviation over 10 runs goes below 5\% for the LIL method.
    }}
  \label{fig:mds-cifar100}
  \vspace{-10pt}
\end{figure*}

\begin{figure*}[h]
  \centering
  \mbox{
  \hspace{-10pt}
    \subfigure[No distribution shift, no window.   \label{fig:vim_no_shift_no_win_cifar100}]{\includegraphics[scale=0.22]{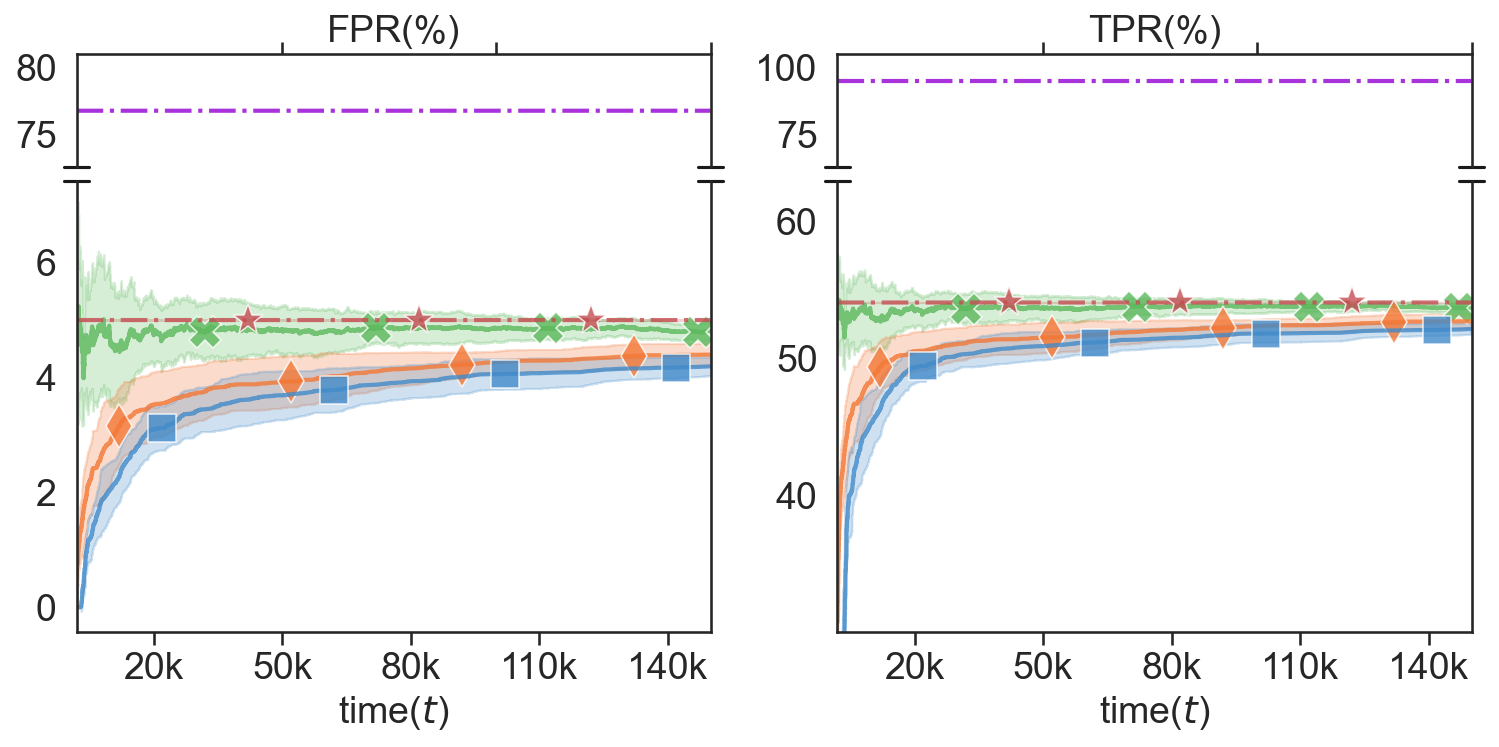}}
    \hspace{1pt}
    
    \subfigure[Distribution shift, 5k window. \label{fig:vim_shift_win_5k_cifar100}]{\includegraphics[scale=0.22]{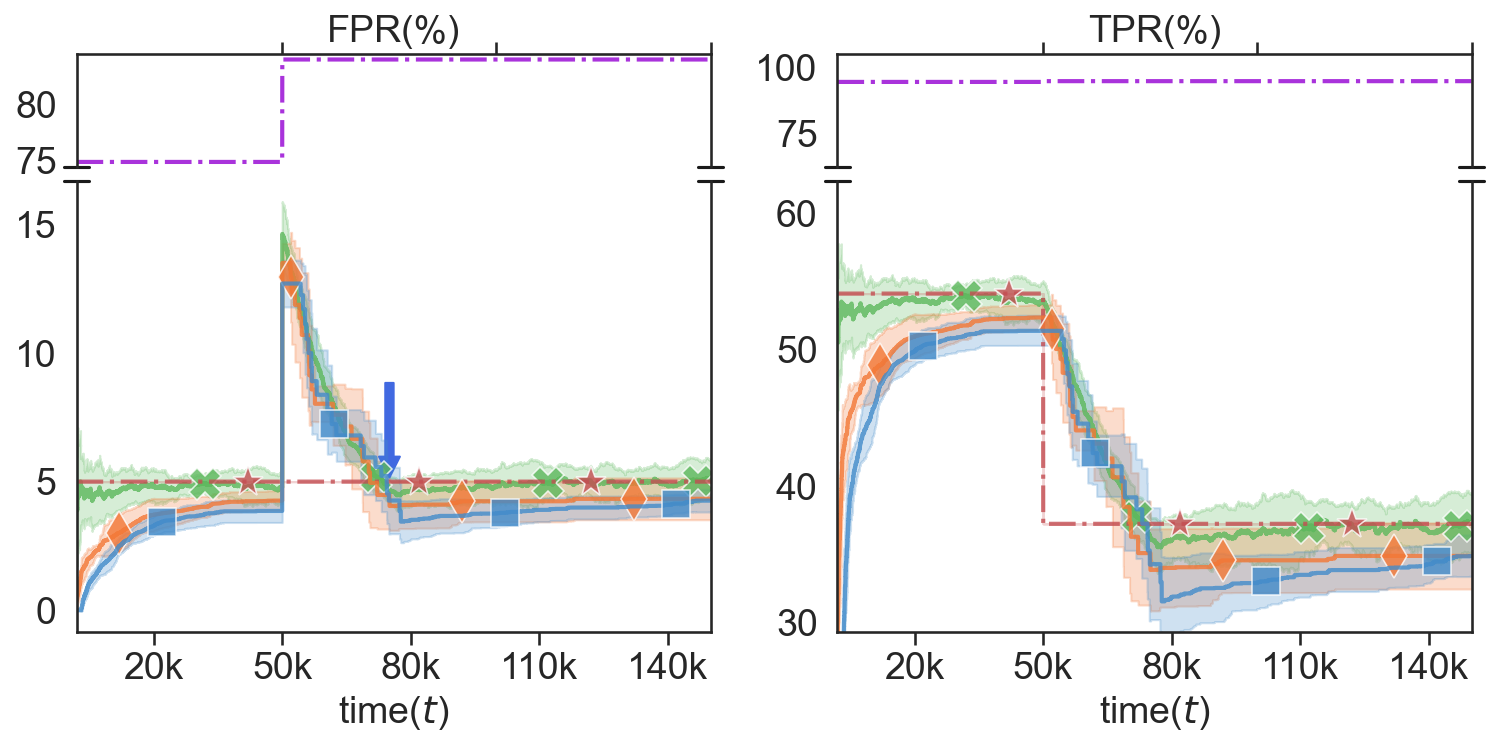}}

    \subfigure[Distribution shift, 10k window. \label{fig:vim_shift_win_10k_cifar100}]{\includegraphics[scale=0.22]{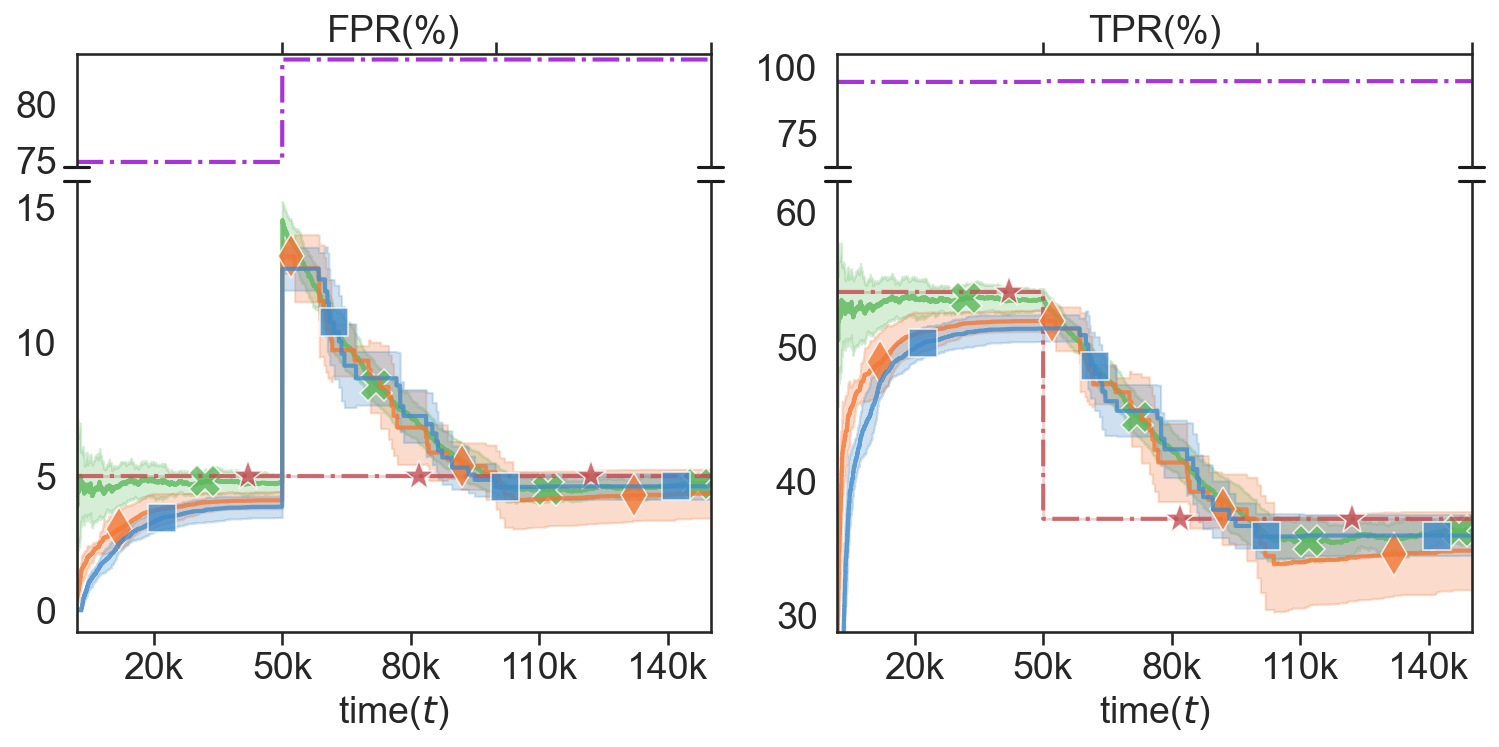}}
  }
  
  \vspace{-5pt}
  \caption{ \small {
   Results with the VIM scores on Cifar-100 as ID dataset. For (b) and (c) the distribution shifts at $t=50k$. The arrow indicates the time at which the mean FPR + std. deviation over 10 runs goes below 5\% for the LIL method.
    }}
  \label{fig:vim-cifar100}
  \vspace{-10pt}
\end{figure*}

\begin{figure*}[h]
  \centering
  \mbox{
  \hspace{-10pt}
    \subfigure[No distribution shift, no window.   \label{fig:ssd_no_shift_no_win_cifar100}]{\includegraphics[scale=0.22]{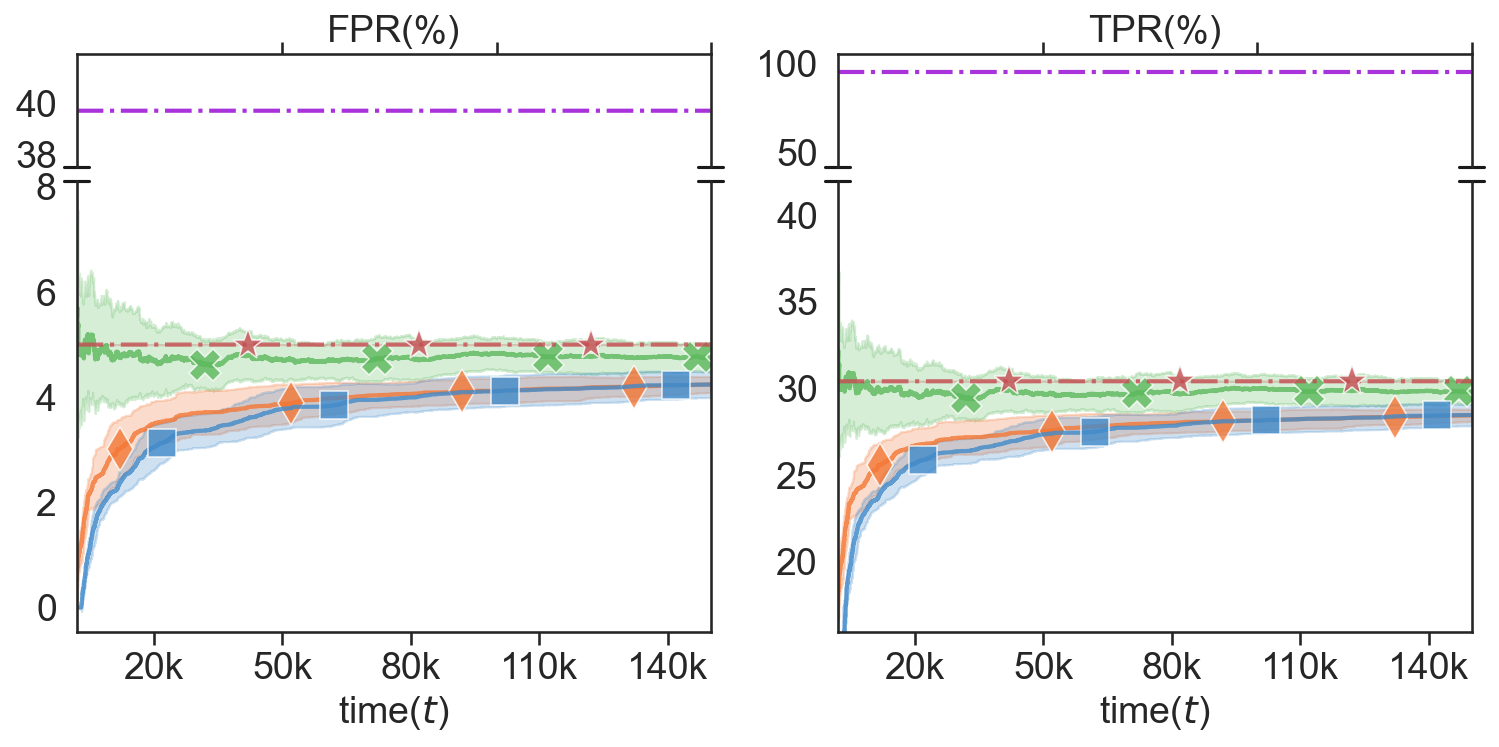}}
    \hspace{1pt}
    
    \subfigure[Distribution shift, 5k window. \label{fig:ssd_shift_win_5k_cifar100}]{\includegraphics[scale=0.22]{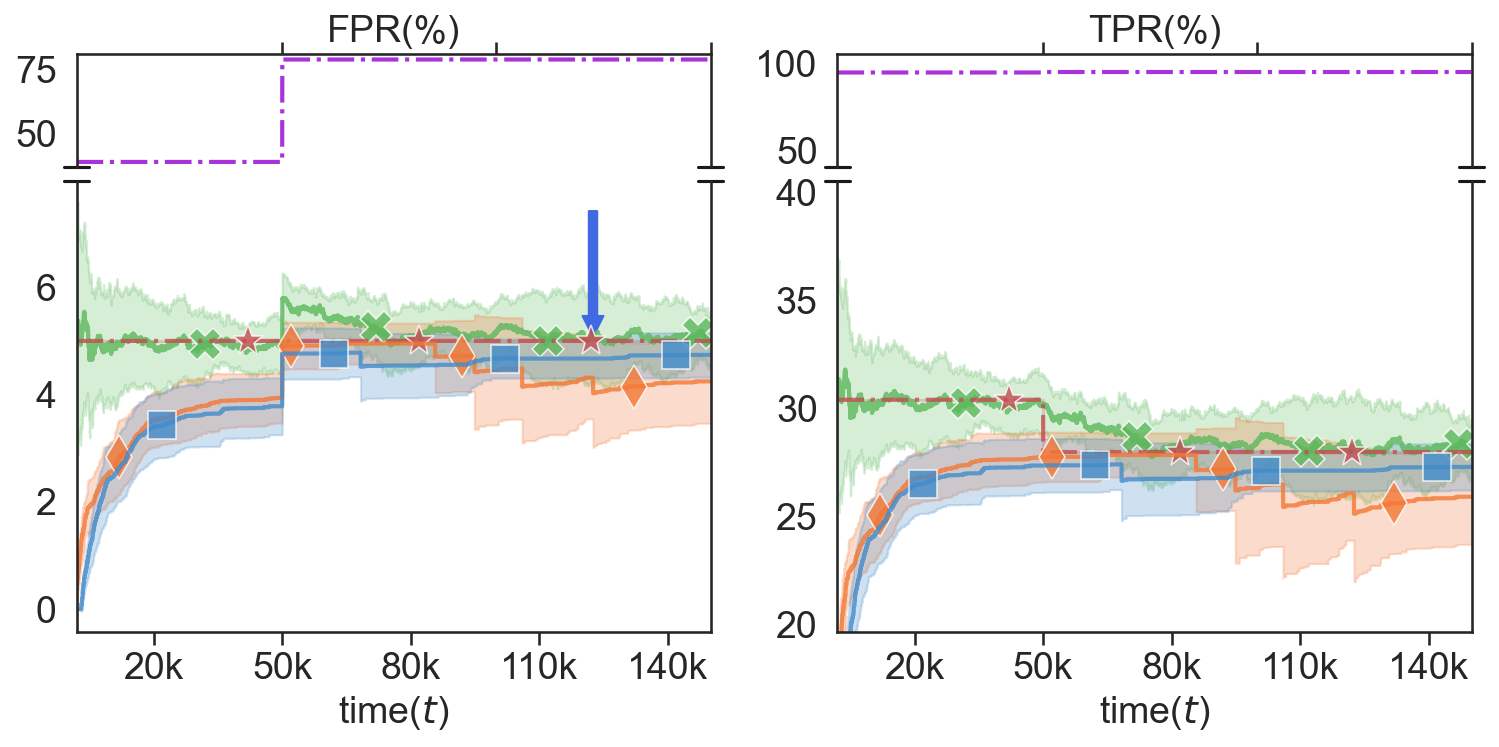}}

    \subfigure[Distribution shift, 10k window. \label{fig:ssd_shift_win_10k_cifar100}]{\includegraphics[scale=0.22]{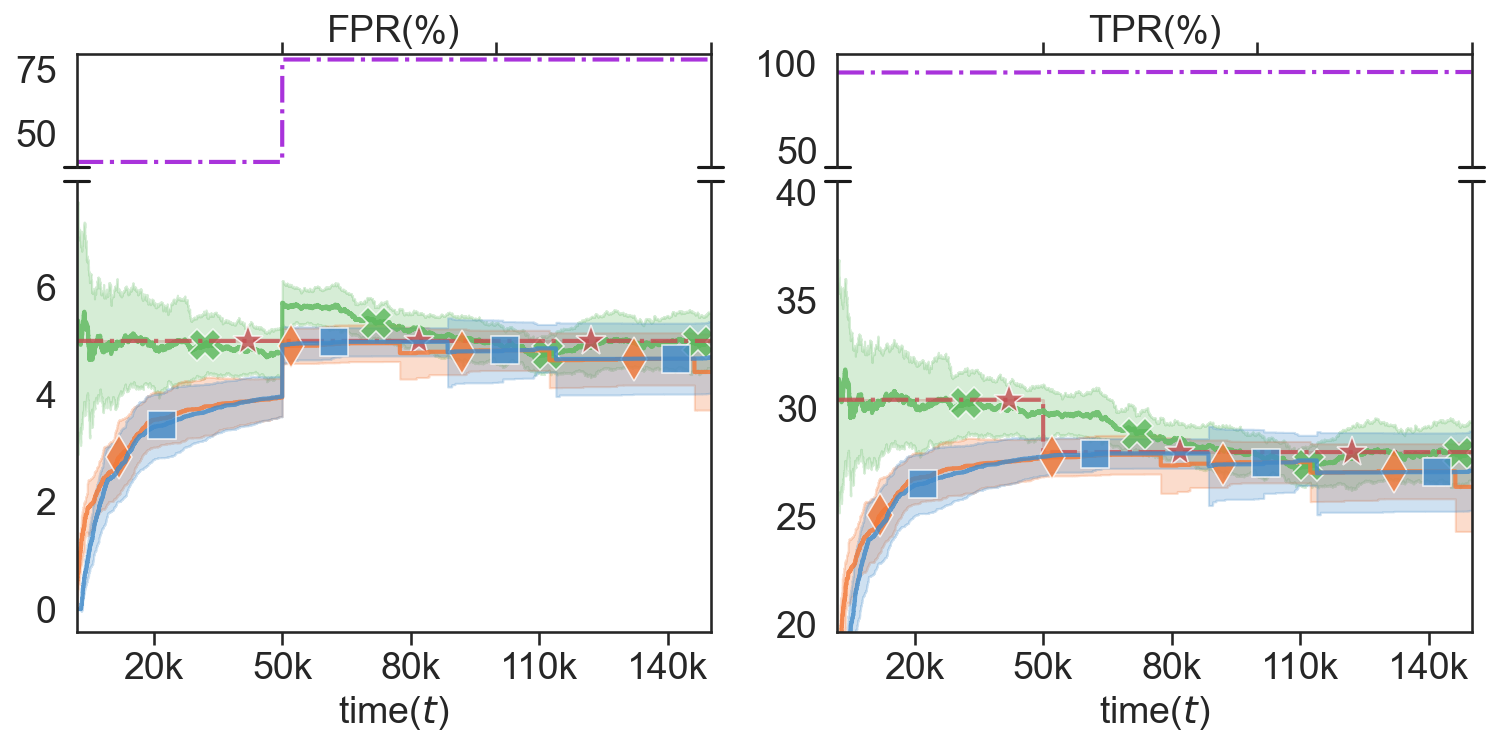}}
  }
  
  \vspace{-5pt}
  \caption{ \small {
   Results with the SSD scores on Cifar-100 as ID dataset. For (b) and (c) the distribution shifts at $t=50k$. The arrow indicates the time at which the mean FPR + std. deviation over 10 runs goes below 5\% for the LIL method.
    }}
  \label{fig:ssd-cifar100}
  \vspace{-10pt}
\end{figure*}

\begin{figure*}[h]
  \centering
\includegraphics[width=0.999\textwidth]{figs-v2/legend_hfdng-3.png}
  \mbox{
  \hspace{-10pt}
    \subfigure[No distribution shift, no window.   \label{fig:odin_no_shift_no_win_cifar100}]{\includegraphics[scale=0.215]{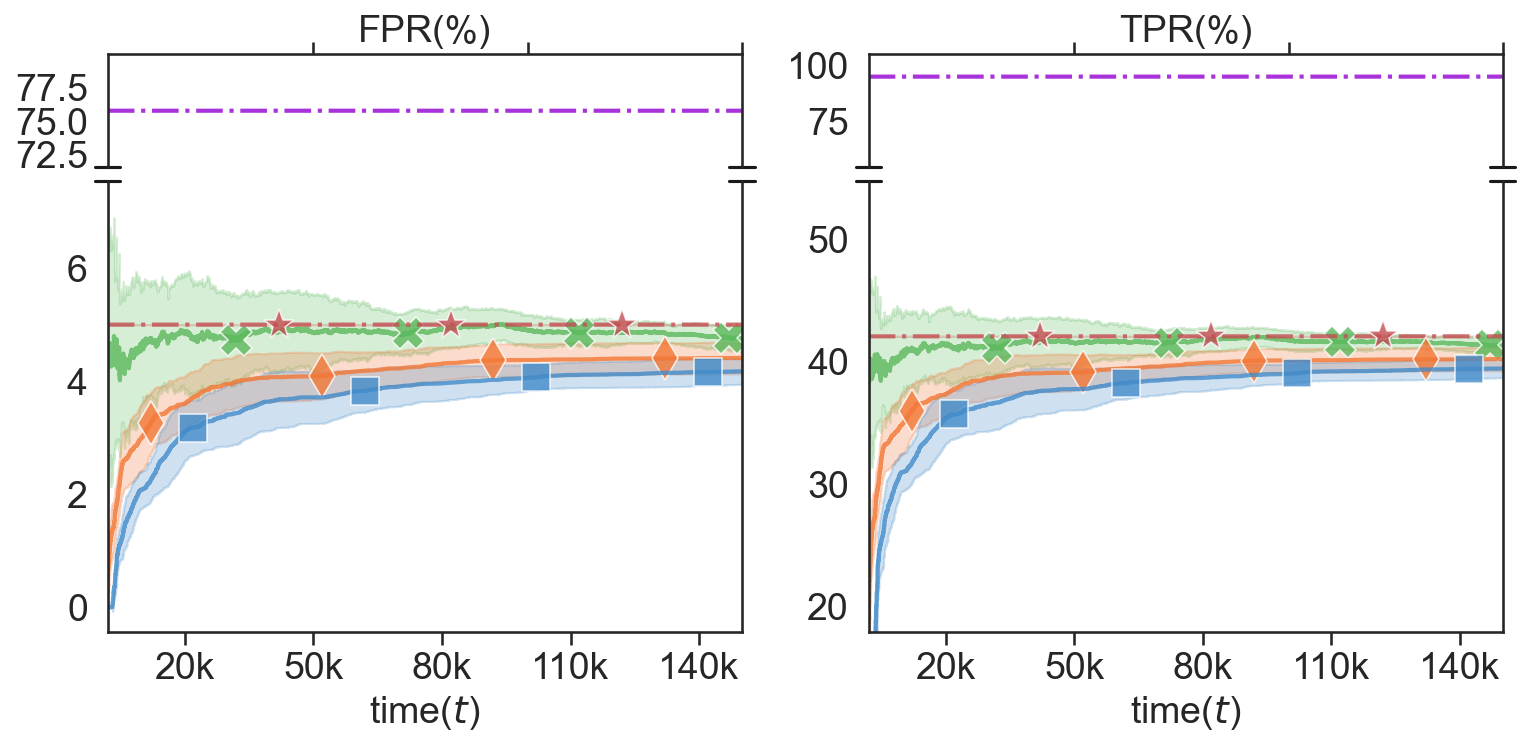}}
    \hspace{1pt}
    
    \subfigure[Distribution shift, 5k window. \label{fig:odin_shift_win_5k_cifar100}]{\includegraphics[scale=0.215]{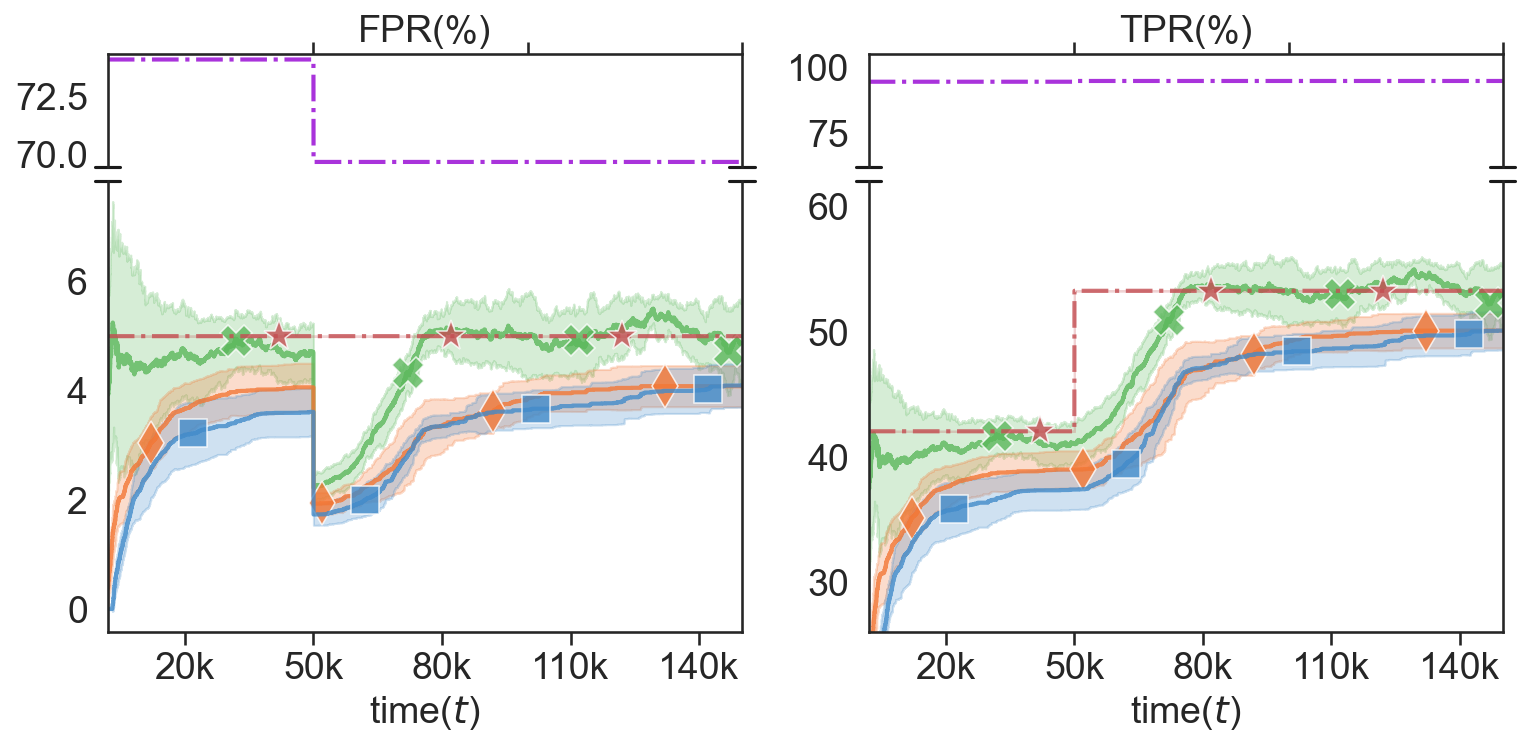}}

    \subfigure[Distribution shift, 10k window. \label{fig:odin_shift_win_10k_cifar100}]{\includegraphics[scale=0.215]{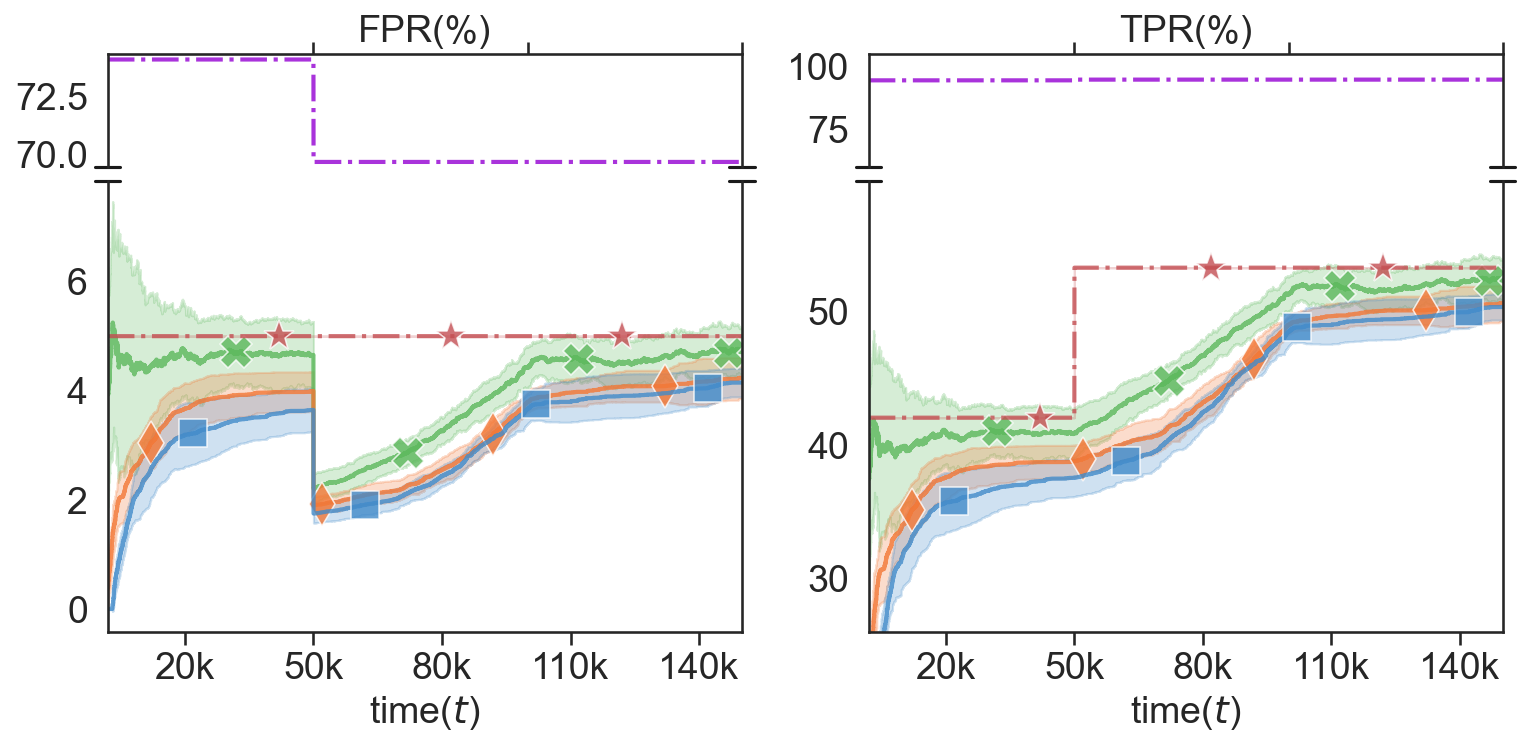}}
  }
  
  \vspace{-5pt}
  \caption{ \small {
   Results with the ODIN scores on Cifar-100 as ID dataset. For (b) and (c) the distribution shifts at $t=50k$. The arrow indicates the time at which the mean FPR + std. deviation over 10 runs goes below 5\% for the LIL method.
    }}
  \label{fig:odin-cifar100}
\end{figure*}

We also provide visualizations showing the distributions of the scores obtained using these scoring functions. Please see Figures \ref{fig:cifar-10-scores-knn} to \ref{fig:cifar-100-scores-odin}.

\newpage 

\begin{figure*}[h]
\centering
\includegraphics[trim= 200 0 200 0,clip ,height=3cm]{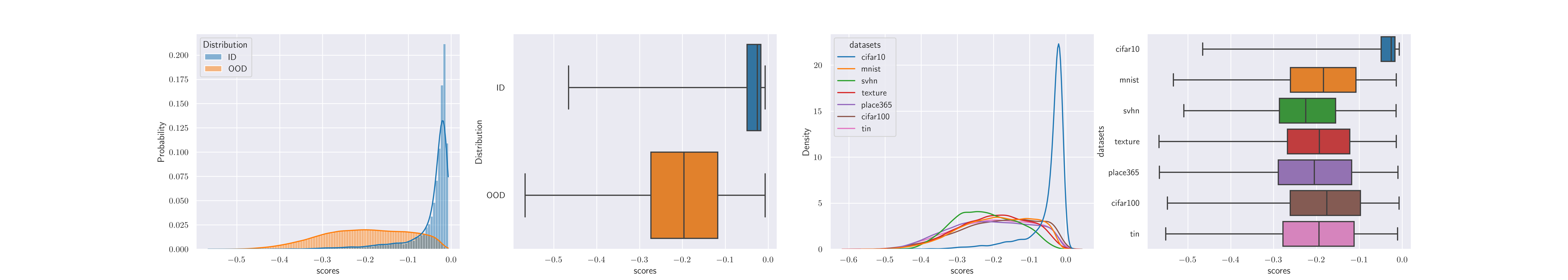}
\vspace{-5pt}
\caption{\small{Scores distribution for KNN with CIFAR-10 as In-Distribution.}}
\vspace{-1pt}
\label{fig:cifar-10-scores-knn}
\end{figure*}

\begin{figure*}[h]
\centering
\includegraphics[trim= 200 0 200 0,clip ,height=3cm]{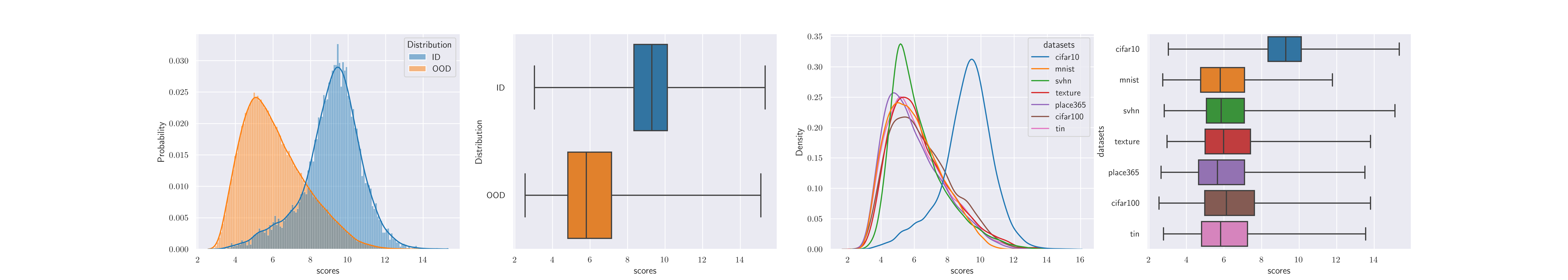}
\vspace{-5pt}
\caption{\small{Scores distribution for EBO with CIFAR-10 as In-Distribution.}}
\vspace{-1pt}
\label{fig:cifar-10-scores-ebo}
\end{figure*}

\begin{figure*}[h]
\centering
\includegraphics[trim= 200 0 200 0,clip ,height=3cm]{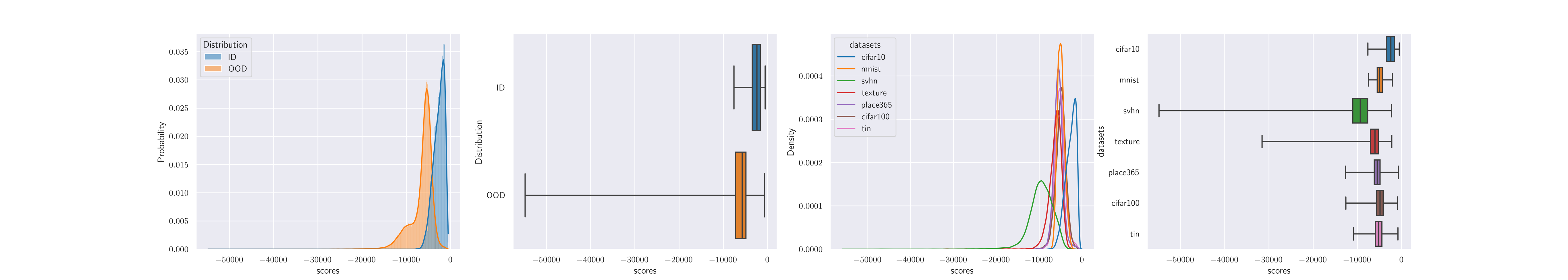}
\vspace{-5pt}
\caption{\small{Scores distribution for SSD with CIFAR-10 as In-Distribution.}}
\vspace{-1pt}
\label{fig:cifar-10-scores-ssd}
\end{figure*}

\begin{figure*}[h]
\centering
\includegraphics[trim= 200 0 200 0,clip ,height=3cm]{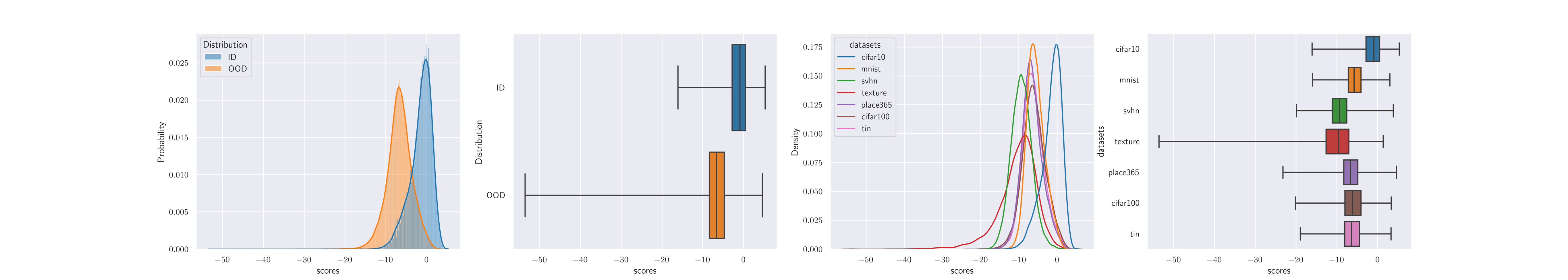}
\vspace{-5pt}
\caption{\small{Scores distribution for VIM with CIFAR-10 as In-Distribution.}}
\vspace{-1pt}
\label{fig:cifar-10-scores-vim}
\end{figure*}

\begin{figure*}[h]
\centering
\includegraphics[trim= 200 0 200 0,clip ,height=3cm]{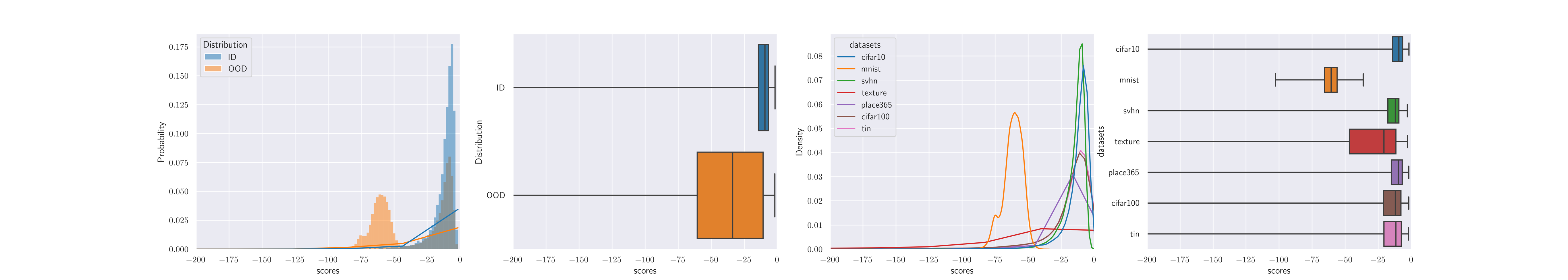}
\vspace{-1pt}
\caption{\small{Scores distribution for MDS with CIFAR-10 as In-Distribution.}}
\vspace{-1pt}
\label{fig:cifar-10-scores-mds}
\end{figure*}

\begin{figure*}[h]
\centering
\includegraphics[trim= 200 0 200 0,clip ,height=3cm]{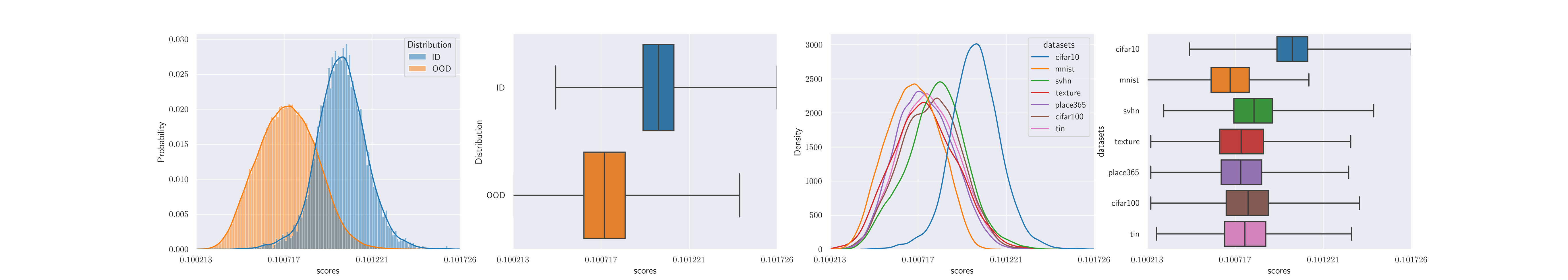}
\vspace{-5pt}
\caption{\small{Scores distribution for ODIN with CIFAR-10 as In-Distribution.}}
\vspace{-1pt}
\label{fig:cifar-10-scores-odin}
\end{figure*}

\begin{figure*}[h]
\centering
\includegraphics[trim= 200 0 200 0,clip ,height=3cm]{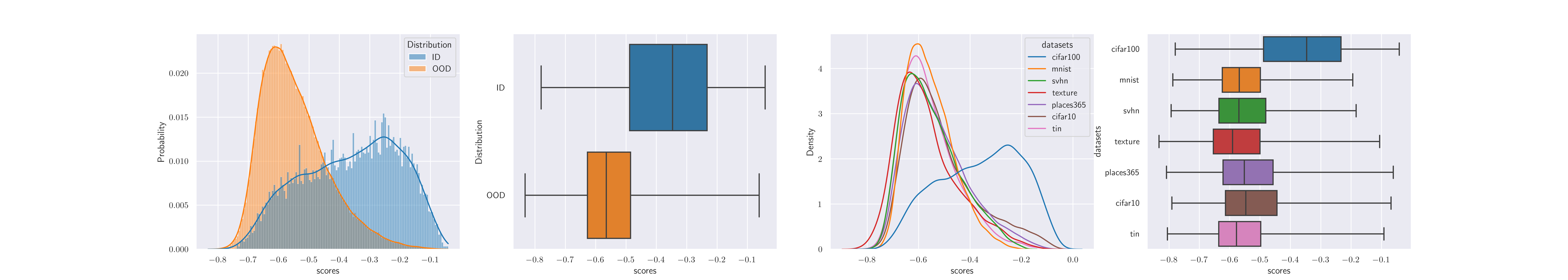}
\vspace{-5pt}
\caption{\small{Scores distribution for KNN with cifar-100 as In-Distribution.}}
\vspace{-1pt}
\label{fig:cifar-100-scores-knn}
\end{figure*}

\begin{figure*}[h]
\centering
\includegraphics[trim= 200 0 200 0,clip ,height=3cm]{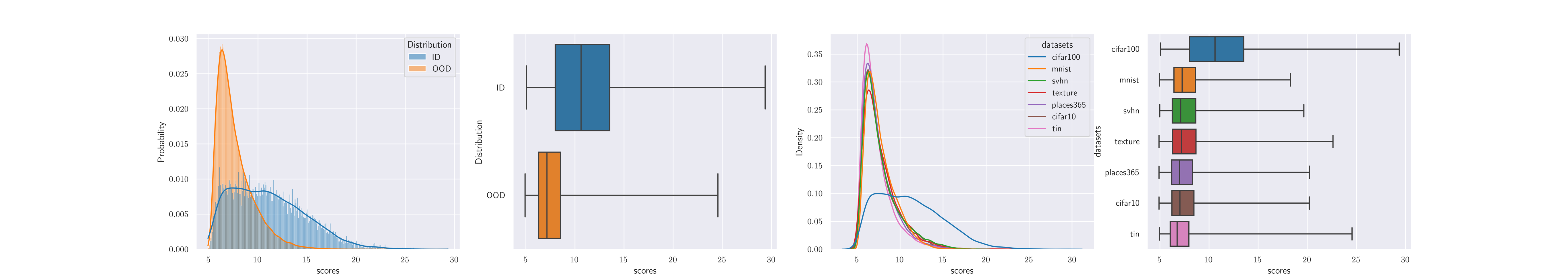}
\vspace{-5pt}
\caption{\small{Scores distribution for EBO with cifar-100 as In-Distribution.}}
\vspace{-1pt}
\label{fig:cifar-100-scores-ebo}
\end{figure*}

\begin{figure*}[h]
\centering
\includegraphics[trim= 200 0 200 0,clip ,height=3cm]{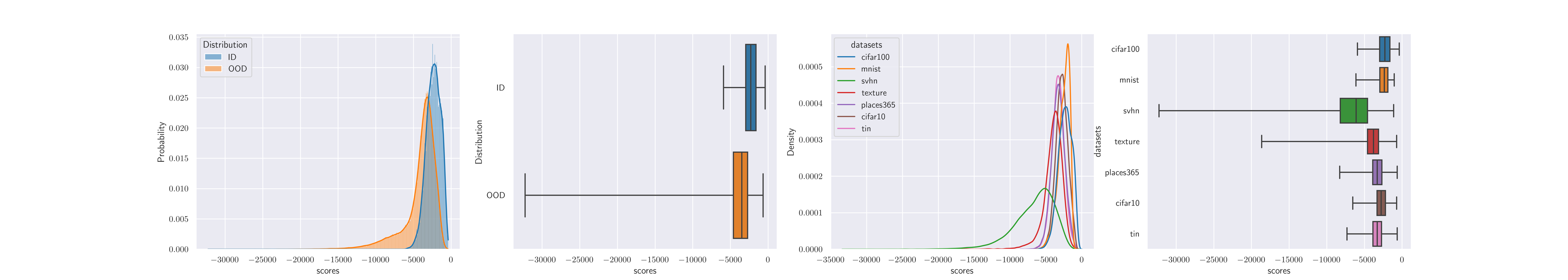}
\vspace{-5pt}
\caption{\small{Scores distribution for SSD with cifar-100 as In-Distribution.}}
\vspace{-1pt}
\label{fig:cifar-100-scores-ssd}
\end{figure*}

\begin{figure*}[h]
\centering
\includegraphics[trim= 200 0 200 0,clip ,height=3cm]{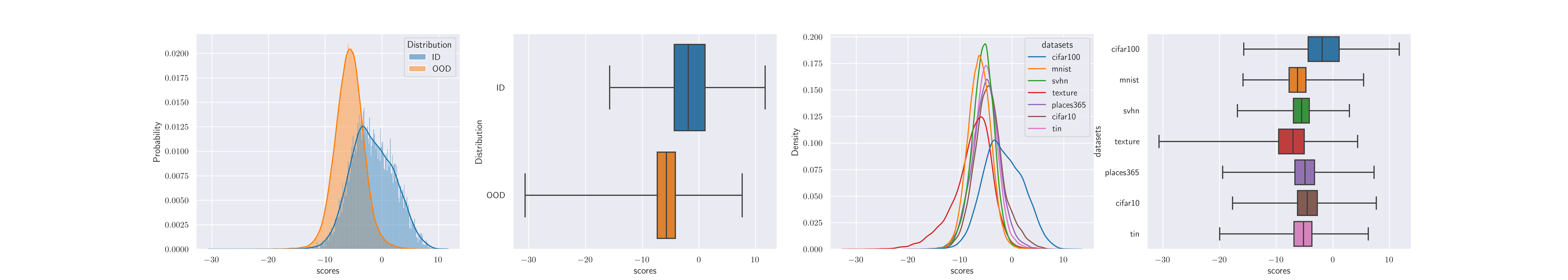}
\vspace{-5pt}
\caption{\small{Scores distribution for VIM with cifar-100 as In-Distribution.}}
\vspace{-1pt}
\label{fig:cifar-100-scores-vim}
\end{figure*}

\begin{figure*}[h]
\centering
\includegraphics[trim= 200 0 200 0,clip ,height=3cm]{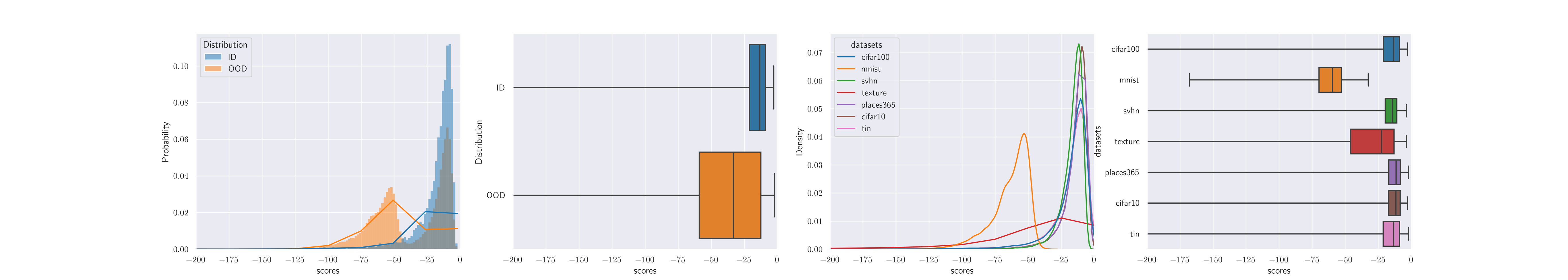}
\vspace{-1pt}
\caption{\small{Scores distribution for MDS with cifar-100 as In-Distribution.}}
\vspace{-1pt}
\label{fig:cifar-100-scores-mds}
\end{figure*}

\begin{figure*}[h]
\centering
\includegraphics[trim= 200 0 200 0,clip ,height=3cm]{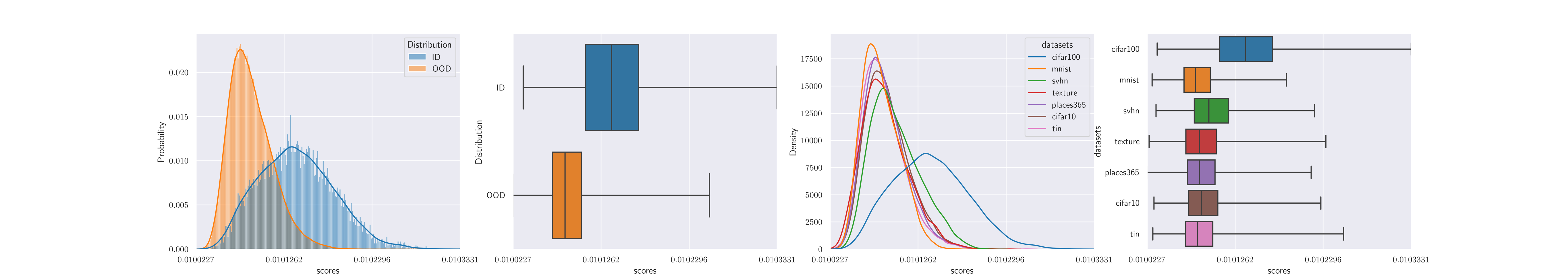}
\vspace{-5pt}
\caption{\small{Scores distribution for ODIN with cifar-100 as In-Distribution.}}
\vspace{-1pt}
\label{fig:cifar-100-scores-odin}
\end{figure*}

\end{document}